\documentclass{article}

\usepackage{arxiv}

\usepackage[utf8]{inputenc} 
\usepackage[T1]{fontenc}    
\usepackage{hyperref}       
\usepackage{url}            
\usepackage{booktabs}       
\usepackage{amsfonts}       
\usepackage{nicefrac}       
\usepackage{microtype}      
\usepackage{lipsum}

\usepackage[utf8]{inputenc} 
\usepackage[T1]{fontenc}    
\usepackage{url}            
\usepackage{booktabs,bigstrut}       
\usepackage{amsfonts}       
\usepackage{nicefrac}       
\usepackage{microtype,picins}      
\usepackage{amssymb}
\usepackage{amsmath}
\usepackage{amsthm,subfigure}
\usepackage{graphicx}
\usepackage{algpseudocode}
\usepackage[linesnumbered,ruled]{algorithm2e}
\usepackage{xfrac}
\usepackage{color}
\usepackage{wrapfig}
\usepackage{paralist}
\usepackage{mathtools}
\usepackage{comment}
\usepackage{hyperref}
\usepackage[svgnames]{xcolor}
\usepackage[customcolors,shade]{hf-tikz}
\usepackage{empheq,multicol,titlesec}

\usepackage[most]{tcolorbox}
\newtcbox{\mymath}[1][]{%
    nobeforeafter, math upper, tcbox raise base,
    enhanced, colframe=blue!30!black,
    colback=gray!30, boxrule=1pt,
    #1}
\usepackage{multirow,nicefrac}
\usepackage{tikz}
\usepackage{colortbl, array}
\usepackage{hhline}
\usepackage{algpseudocode}
\usepackage[linesnumbered,ruled]{algorithm2e}

\definecolor{rulecolor}{RGB}{0,71,171}
\definecolor{tableheadcolor}{RGB}{204,229,255}

\usepackage{caption}
\newcommand{\argmax}{\operatornamewithlimits{argmax}}

\newtheorem{definition}{Definition}
\newtheorem{proposition}{Proposition}

\makeatletter
\newsavebox{\mybox}\newsavebox{\mysim}
\newcommand{\distras}[1]{%
  \savebox{\mybox}{\hbox{\kern3pt$\scriptstyle#1$\kern3pt}}%
  \savebox{\mysim}{\hbox{$\sim$}}%
  \mathbin{\overset{#1}{\kern\z@\resizebox{\wd\mybox}{\ht\mysim}{$\sim$}}}%
}

  \usepackage{xcolor,colortbl,subfigure,wrapfig}
\definecolor{Gray}{gray}{0.85}
\definecolor{LightCyan2}{rgb}{0.98,0.82,0.17}
\definecolor{LightCyan}{rgb}{0.56,0.68,0.78}

\title{POIRot: A rotation invariant omni-directional pointnet \thanks{A tribute to Agatha Christie}}

\author{
Liu~Yang, Rudrasis~Chakraborty~and~Stella~X.~Yu \\
University of California\\ Berkeley, USA. \\
  \texttt{\{liu-yang, rudra, stellayu\}@berkeley.edu} \\
}

\begin{document}
\maketitle

Point-cloud is an efficient way to represent 3D world. Analysis of point-cloud deals with understanding the underlying 3D geometric structure. But due to the lack of smooth topology, and hence the lack of neighborhood structure, standard correlation can not be directly applied on point-cloud. One of the popular approaches to do point correlation is to partition the point-cloud into voxels and extract features using standard 3D correlation. But this approach suffers from sparsity of point-cloud and hence results in multiple empty voxels. One possible solution to deal with this problem is to learn a MLP to map a point or its local neighborhood to a high dimensional feature space. All these methods suffer from a large number of parameters requirement and are susceptible to random rotations. A popular way to make the model ``invariant'' to rotations is to use data augmentation techniques with small rotations but the potential drawback includes \begin{inparaenum}[\bfseries (a)] \item more training samples \item susceptible to large rotations. \end{inparaenum}
In this work, we develop a rotation invariant point-cloud segmentation and classification scheme based on the omni-directional camera model (dubbed as {\bf POIRot$^1$}). Our proposed model is rotationally invariant and can preserve geometric shape of a 3D point-cloud. Because of the inherent rotation invariant property, our proposed framework requires fewer number of parameters (please see \cite{Iandola2017SqueezeNetAA} and the references therein for motivation of lean models). Several experiments have been performed to show that our proposed method can beat the state-of-the-art algorithms in classification and part segmentation applications. Furthermore, we have applied our proposed framework to detect corpus callosum shape from a 3D brain scan represented as a point-cloud. We have empirically shown that our proposed method can detect corpus callosum shape from the 3D brain point-cloud given only the atlas of the corpus callosum. 

\section{Introduction}\label{sec:introduction}

Point-cloud is an efficient way to represent 3D world \cite{pointnet,pointnet++}. The recent years have witnessed the popularity of 3D computer vision tasks with the advent of 3D sensors and modeling devices. The 3D sensors such as depth cameras, LiDAR can output 3D point-cloud, which is a key component in several 3D vision tasks including but not limited to virtual/ augmented reality \cite{lin2018novel, rambach2016learning}, 3D scenes understanding \cite{tulsiani2018factoring, vasu2018occlusion, dai2018scancomplete}, and autonomous driving \cite{chen2018lidar, shen2018stereo, wu2018squeezeseg}.

Due to the enormous popularity of correlation neural networks (CNNs) in computer vision tasks \cite{deng2009imagenet, krizhevsky2012imagenet}, an obvious approach is to use CNNs to process point-cloud. But unfortunately, due to the lack of the smooth topology of a point-cloud, standard correlation can not be applied as it is. This is mainly due to the fact that at a given point in a point-cloud, it is hard to define a grid structure analogous to an image, hence applying standard correlation turns out to be a non-trivial and challenging problem.  To alleviate this problem, several researchers \cite{wu20153d, riegler2017octnet, wang2017cnn,su2015multi, qi2016volumetric,kalogerakis20173d, zhou2018learning}  use CNNs to process 3D point-cloud is by first mapping the point-cloud on a smooth topological space where doing correlation makes sense. Several popular solutions to overcome the main bottleneck of lack of a smooth neighborhood topology of a 3D point include \begin{inparaenum} \item converting the 3D point-cloud into regular voxel representation \cite{wu20153d, riegler2017octnet, wang2017cnn} or \item using view projection \cite{su2015multi, qi2016volumetric,kalogerakis20173d, zhou2018learning}\end{inparaenum}. Most of these voxel based methods suffer from possible sparsity of point-clouds which results in multiple empty voxels. 

One possible solution is to use MLP to extract features from each point \cite{pointnet} or a local neighborhood of each point \cite{pointnet++}. These models \cite{pointnet, pointnet++, rethage2018fully, MR3D}  directly work on 3D point-clouds. Similar to the CNNs, given a set of points, the ``point correlation layer'' finds ``local patch'' around each point by using point affinity matrix. Here the affinity matrix is defined as the adjacency matrix for the fully connected graph on point-cloud. These local patches are used in standard correlation operator to extract local patches and this operator is defined as ``point correlation''. This basic ``point correlation'' operator layers are stacked together to extract features. But unlike images, defining local patches on point-cloud needs to deal with geometric structure of point-cloud. In most methods \cite{pointnet, pointnet++,DGCNN, li2018pointcnn}, researchers use nearest neighbors to define local neighborhood, a.k.a., ``local patch''. 

The rationale behind using MLP to extract features from each point or a local neighborhood can be thought of as mapping the local neighborhood in a high dimensional space ($\subset \mathbf{R}^n$ for some $n$) from where extracting features using standard correlation makes sense. This is analogous to kernel methods \cite{scholkopf1997kernel} where the features are mapped in Hilbert space, but nonetheless this essentially implies that on the feature space we use the topology induced from the Hilbert space. In the context of point-cloud processing, this analogy translates to the use of the induced topology from the standard smooth topology of $\mathbf{R}^n$.  But due to the presence of geometric structures in a 3D point-cloud, in order to induce the globally flat topology from $\mathbf{R}^n$, $n$ needs to be very large. Naturally this increases the complexity of the model in terms of number of parameters and computational time. To overcome this limitation, one can embed the local structures and geometry of the 3D point-cloud in a ``curved'' space with known non-Euclidean geometry. One of the well-known non-Euclidean spaces is hypersphere, whose topology we use to induce a topology on point-cloud.

 Using the induced topology from sphere, we define a correlation operator on the point-cloud. In order to define correlation operator, we first put sphere on each point of the point-cloud and collect response from it. This essentially represents the local geometry in the point-cloud captured as spherical response. We use spherical correlation to extract rotation equivariant features. Unlike previous methods, we {\it implicitly} look at the interaction between points in the point-cloud by looking at the collective spherical responses. After extracting rotation equivariant local features, we look at {\it explicit} interaction between points in the point-cloud. In section \ref{sec:classmodel}, we give the detailed description of the scheme to extract local and global features for classification and segmentation tasks.  Our proposed method has several advantages over the previous methods including: \begin{inparaenum}[\bfseries (a)]  \item the induced spherical topology makes our proposed scheme invariant to rotations \item due to the presence of intrinsic geometry, our proposed scheme has much leaner model \item the interaction between local features makes the proposed method invariant to permutations. \item  our proposed segmentation scheme outperforms the state-of-the-art methods on benchmark datasets. \item we can achieve similar classification accuracy on rotated data  without any explicit data augmentation. \end{inparaenum} Before describing our proposed method in detail, we discuss some of the previous work.

\subsection{Related Works:} Previous works on 3D point-clouds either deal with representing 3D shapes based on 3D grid, and use standard correlation networks \cite{wu20153d, maturana2015voxnet, tatarchenko2017octree}. These methods mostly suffer from inefficient usage of 3D voxels due to empty voxels and fail to capture 3D geometric shapes. Furthermore, due to the computational complexity of 3D correlation operations, this is not a desirable choice. In some recent work \cite{klokov2017escape, riegler2017octnet, tatarchenko2017octree, wang2017cnn} researchers proposed techniques to somewhat overcome these limitations but still the partition of point-cloud into voxels makes these algorithm not suitable to capture 3D geometric shapes. 

Another major body of work is mostly deal with developing ``correlation'' like techniques on 3D point-clouds. As a first work in this genre, PointNet \cite{pointnet} embeds each point coordinate in a high dimensional space by learning a mapping and then aggregating information by pooling the features. Although achieving reasonable accuracy, PointNet did not learn any local geometric information of the 3D shape. PointNet++ \cite{pointnet++} handled this by proposing a hierarchical application of isolated 3D point feature learning to multiple subsets of point-cloud data. The authors ideally used the single point processing unit hierarchically on multiple subsets of the point-cloud. Several other researchers proposed techniques to combine local neighborhood information either by defining correlation operator like $\chi$-conv \cite{li2018pointcnn} or by using a dynamic graph based technique \cite{DGCNN}. In order to capture geometric shapes, \cite{su2018splatnet} extracted local structure by grouping points based on permutohedral lattices \cite{adams2010fast}, and then applied bilateral correlation \cite{bilaterCNN} for feature learning. Super-point graphs \cite{landrieu2018large} proposed to partition point-cloud into super-points to learn 3D geometric shapes. 

Though there is a large body of work defining ``correlation'' like technique on point-clouds, none of them define a equivariant correlation operator on point-cloud. As stated before, the challenge is mostly due to lack of smooth topology in 3D point-clouds which are naturally equipped with discrete topology. This motivates us to define a correlation operator by defining an induced smooth topology on point-clouds. In the rest of the paper, we first describe our proposed classification and segmentation technique for point-clouds in Section \ref{sec:classmodel} with experimental validations in Section \ref{sec:results}.

\section{A rotation invariant CNN for point-clouds} \label{sec:classmodel}
In this section, we first give the motivation of our proposed geometric framework for processing 3D point-cloud. More specifically, we will point out that the existing methods has several shortcomings including:  \begin{inparaenum} \item earlier methods either mapped each point or its neighborhood in a point-cloud into higher dimension in order to extract features and hence require significantly large number of parameters \item none of these previous methods define a correlation operator on point-clouds preserving geometric invariance properties.\end{inparaenum} To avoid these shortcomings, we propose our framework to process 3D point-clouds which is \begin{inparaenum} \item inherently invariant to rigid transformations, \item use definition of correlation on point-clouds by induced topology from sphere \item leaner compared to the previous methods \end{inparaenum}

We first propose a rotation invariant correlation neural network (CNN) using an induced spherical topology on the point-cloud. Though the formulation described below can be applied on $\mathbf{R}^n$ for any $n\in \mathbf{Z}^+$, in this work we have restricted ourselves to $n=3$ as our proposed framework is specifically designed for 3D point-cloud. Our proposed framework consists of three basic building blocks which we describe next.

\pichskip{5pt}
\parpic[r][b]{%
  \begin{minipage}{42mm}
  \centering
        \includegraphics[scale=0.17]{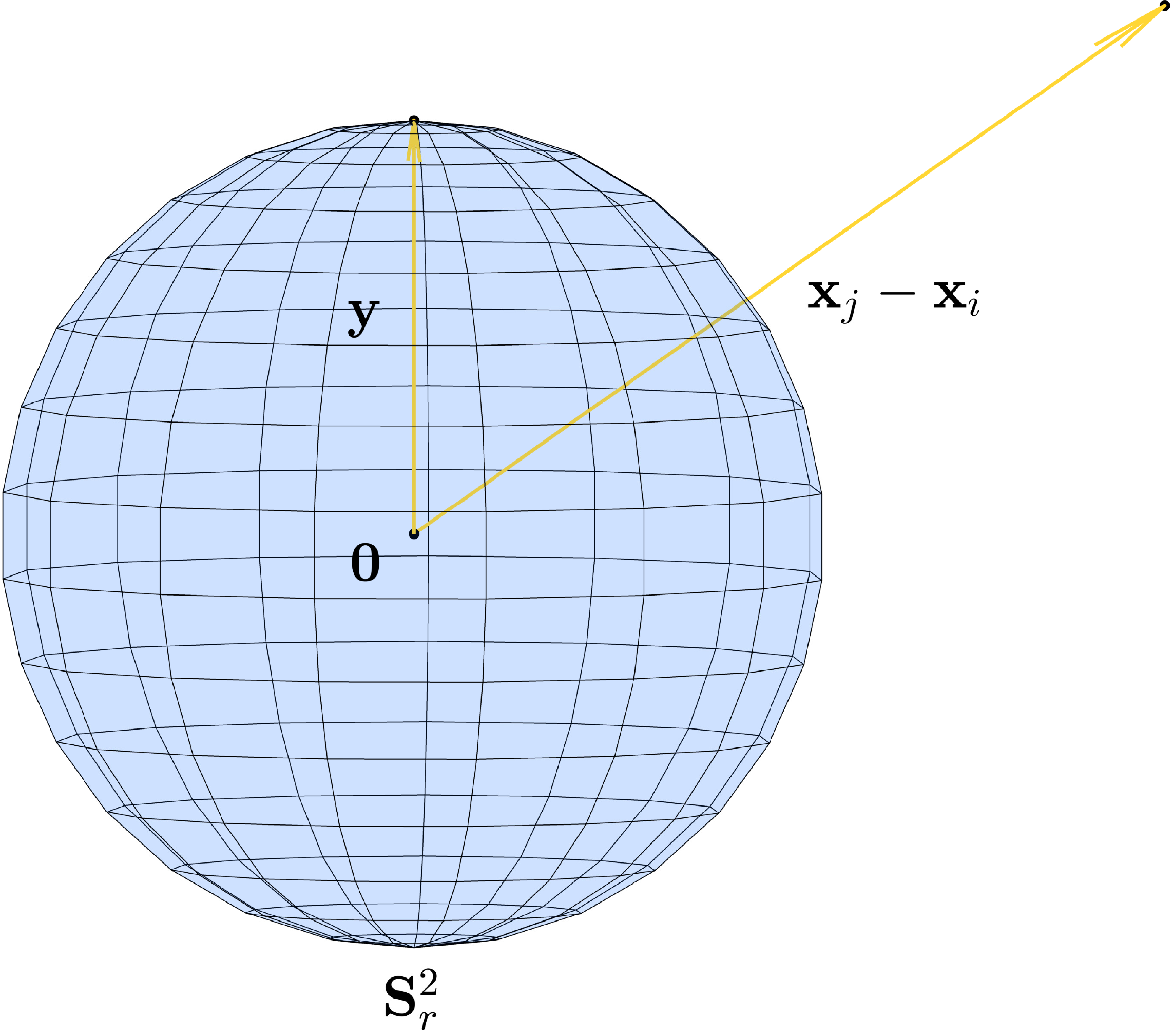}
        \captionof{figure}{Response from point-cloud collected on sphere.}
        \label{classmodel:figm1}
    \vspace{-5em}
  \end{minipage}
}
    
\subsection{Collecting responses on the point-cloud:} At each point in the point-cloud we put a sphere and collect the response from the entire point-cloud. This gives at each point $\mathbf{x}_i \in \mathbf{R}^3$, a function $f_i: \mathbf{S}^2_r(\mathbf{x}_i) \rightarrow \mathbf{R}$, where, $\mathbf{S}^2_r(\mathbf{x}_i)$ is the sphere of radius $r>0$ centered at $\mathbf{x}_i$. Thus given the point-cloud $X = \left\{\mathbf{x}_i\right\}_{i=1}^N$, we represent by $\left\{f_i\right\}$. Now, we collect the combined responses from entire point-cloud. Before doing that, for each $\mathbf{x}_i$, we subtract $\mathbf{x}_i$ from $X$ so that $\mathbf{S}^2_r(\mathbf{x}_i)$ at $\mathbf{x}_i$ is centered at the origin of $\mathbf{R}^3$. Without any loss of generality, we will denote the sphere centered at the origin by $\mathbf{S}^2_r$. Now, given $\mathbf{y} \in \mathbf{S}^2_r$, we compute the response $f_i(\mathbf{y})$ (an example is shown in Fig. \ref{classmodel:figm1}) as 
\begin{align}
\label{classmodel:eq1}
f_i(\mathbf{y})= \sum_{\mathbf{x}_j \not\in \mathbf{B}^2_r(\mathbf{x}_i)} \max\left\{0, \mathbf{y}^t(\mathbf{x}_j - \mathbf{x}_i)\right\}, 
\end{align}
where, $\mathbf{B}^2_r(\mathbf{x}_i)$ is the unit ball with radius $r$ centered at $\mathbf{x}_i$. The reason for ignoring the negative responses is twofold: \begin{inparaenum} \item Given $\mathbf{y}, \widetilde{\mathbf{y}} \in \mathbf{S}^2_r$ and $\mathbf{x} \in \mathbf{R}^3$, if $\mathbf{x}^t\mathbf{y}$ and $\mathbf{x}^t \widetilde{\mathbf{y}}$ differ in sign (assume $\mathbf{x}^t\mathbf{y} \geq 0$), then the two points, $\mathbf{y}, \widetilde{\mathbf{y}}$ on  $\mathbf{S}^2_r$ must lie on two hemispheres  separated by the equator perpendicular to $\mathbf{x}$. Thus, we can eliminate $\widetilde{\mathbf{y}}$ as there exists a $-\widetilde{\mathbf{y}} \in \mathbf{S}^2_r$ such that $\mathbf{x}^t \left(-\widetilde{\mathbf{y}}\right)\geq 0$. Thus eliminating negative responses will reduce information bottleneck. \item The underlying hypothesis is that response from every point in the point-cloud should be captured by exactly one antipodal point on  $\mathbf{S}^2_r$, thus eliminating negative responses will reduce the amount of conflicting information gathered on $\mathbf{S}^2_r$. \end{inparaenum} A schematic diagram depicting a point-cloud and the corresponding collected response is given in Fig.~\ref{classmodel:fig1}.

\begin{figure}[!ht]
 \centering
      \includegraphics[scale=0.22]{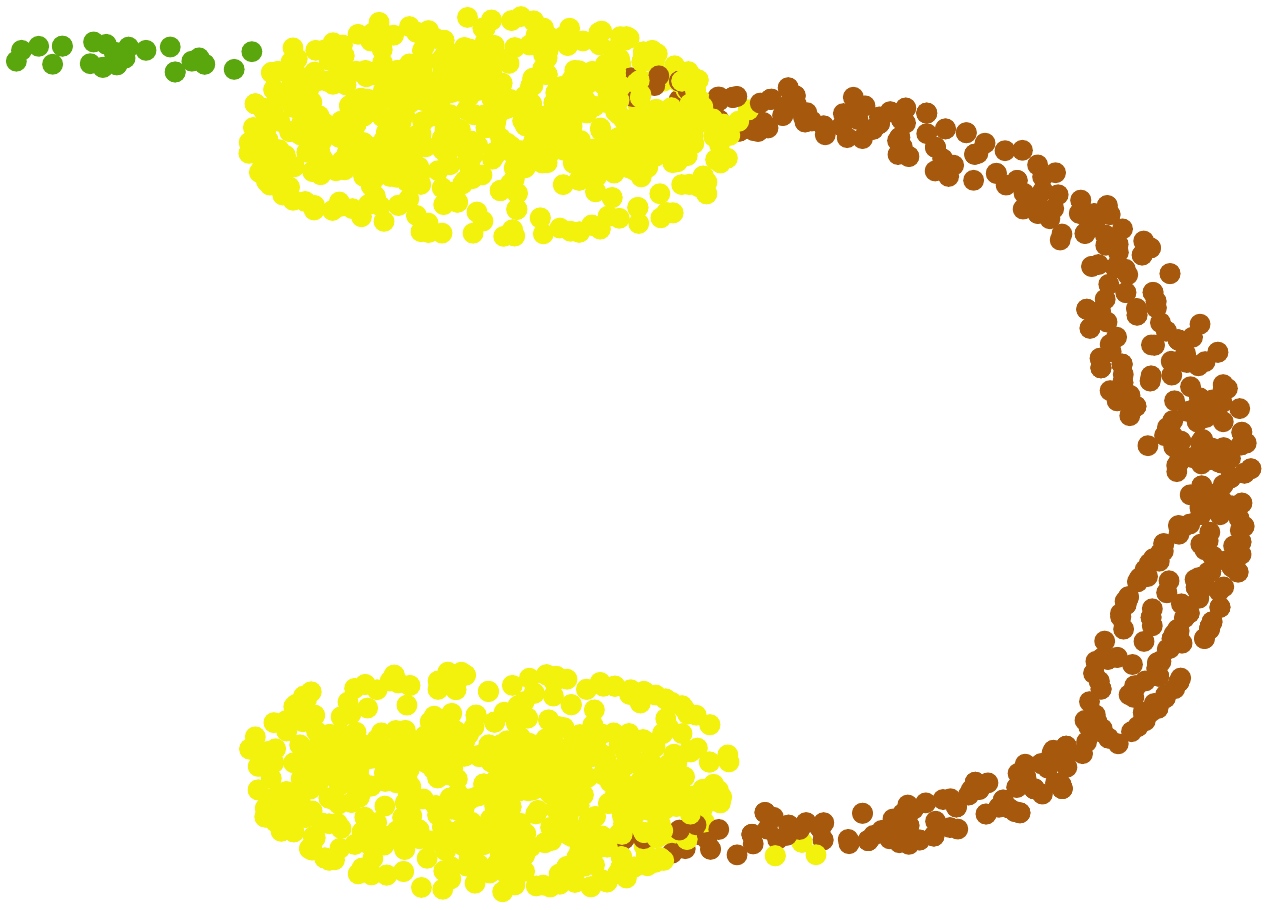}
       \includegraphics[scale=0.21]{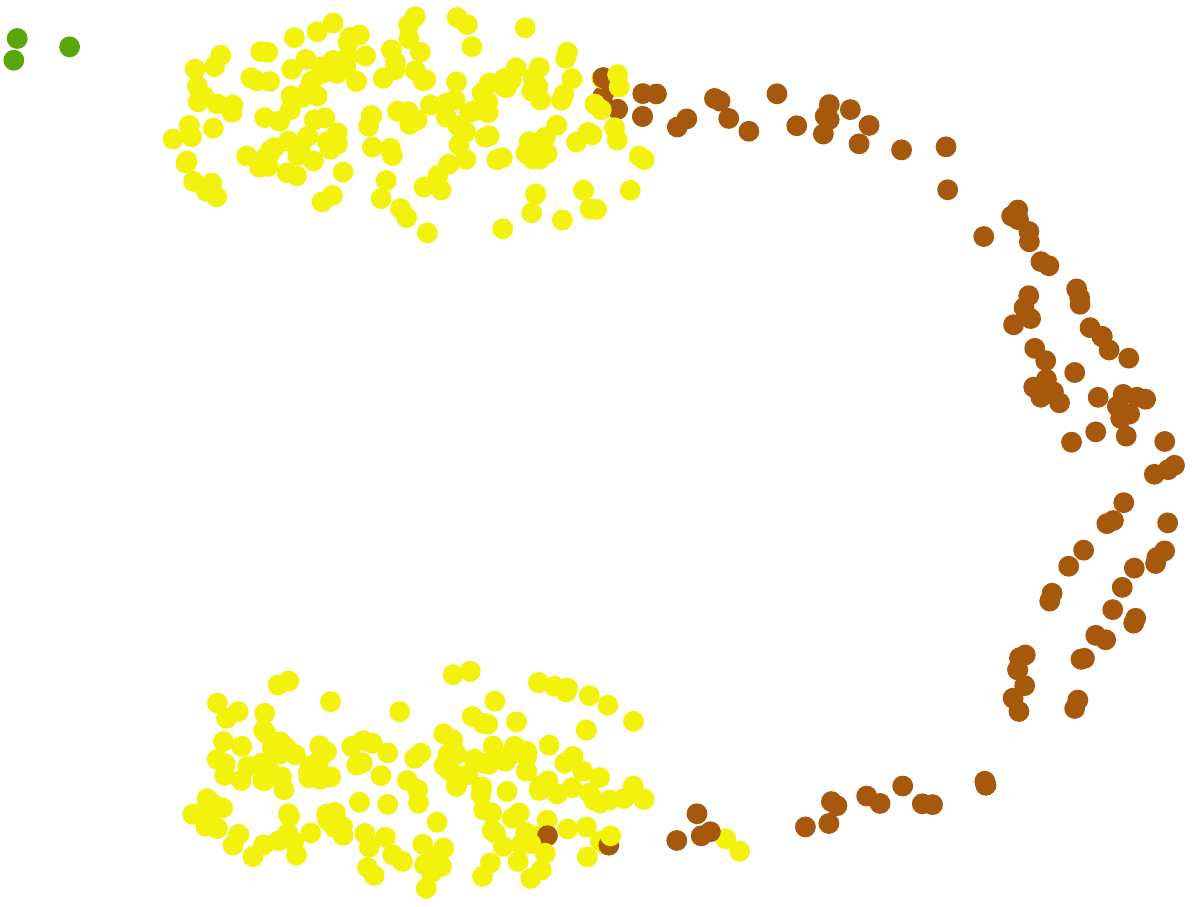}
       \includegraphics[scale=0.14]{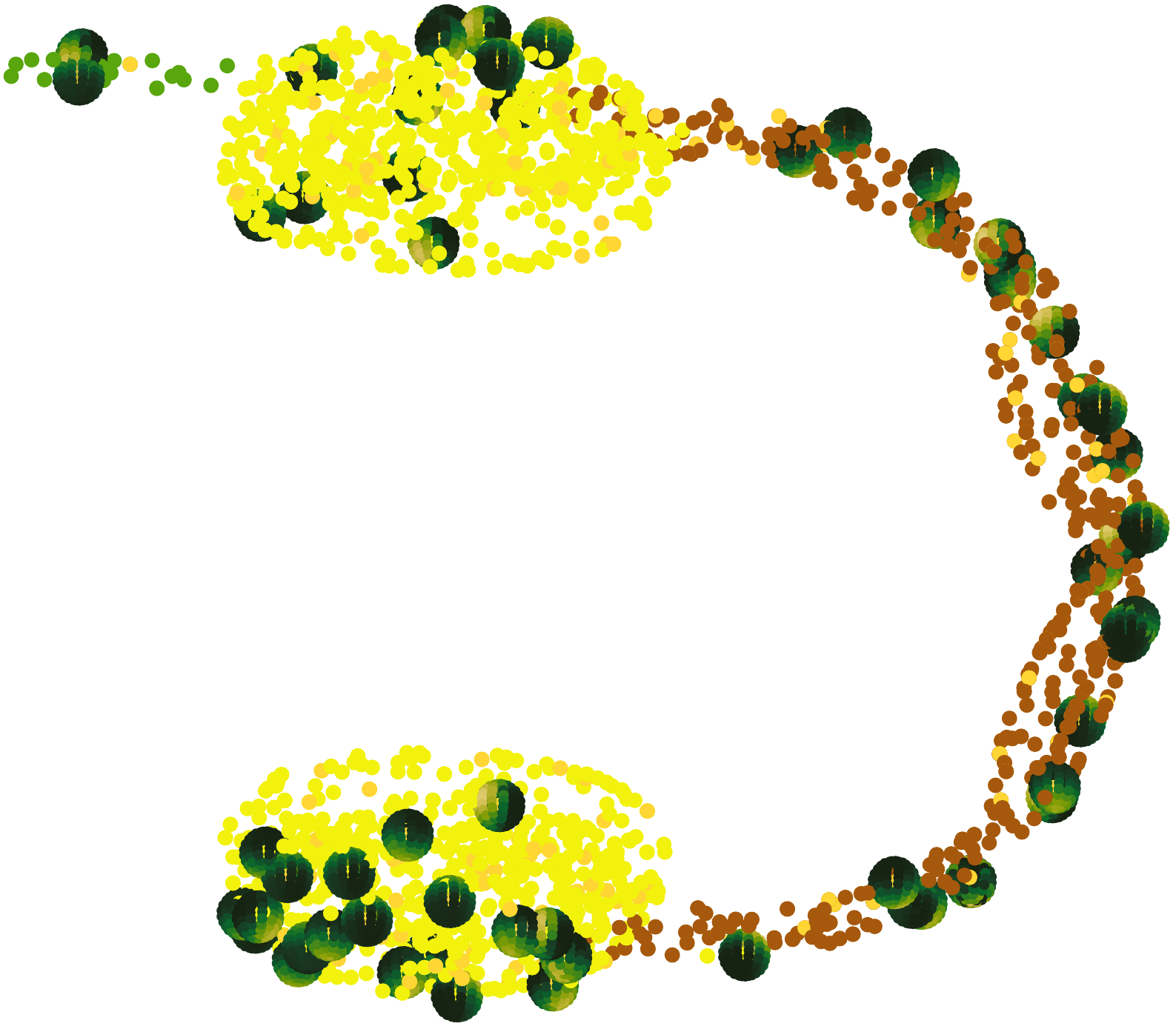}
       \caption{({\it Left:}) The original point-cloud, ({\it Middle:}) the downsampled point-cloud, ({\it Right}) convex hull points.}
      \label{classmodel:fig1}
\end{figure}

This representation can be viewed as {\it putting omni-directional camera at each point and collecting the responses in each viewing direction.} This analogy makes one wonder: {\it Is there a necessity for $N$ cameras where $N$ is the number of points in the point-cloud?} Obviously, for a dense point-cloud the answer is no and hence we propose a {\it multinomial downsampling} strategy as follows.

\pichskip{5pt}
\parpic[r][b]{%
  \begin{minipage}{42mm}
  \centering
        \includegraphics[scale=0.28]{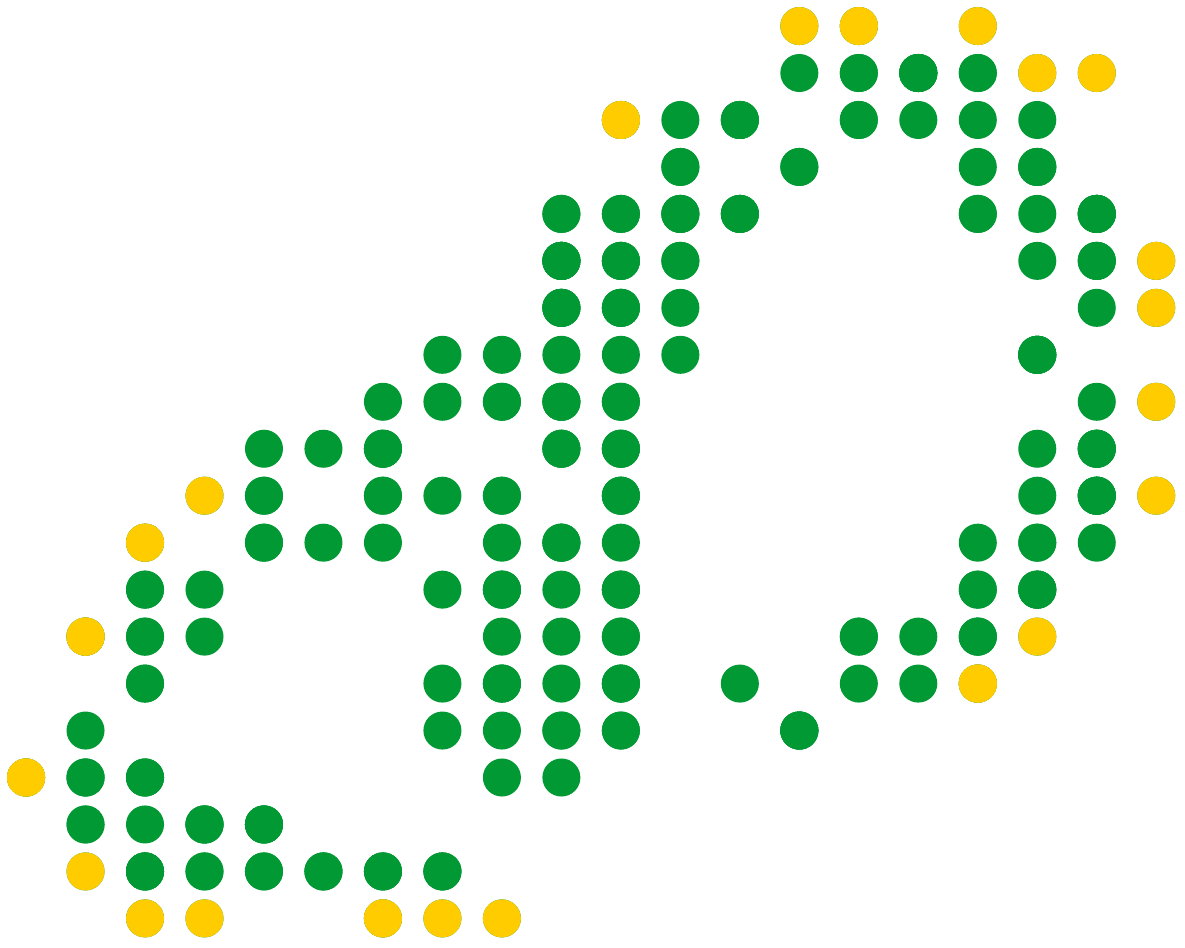}
        \captionof{figure}{Point-cloud with convex set of points (shown in orange).}
        \label{classmodel:fig0}
    \vspace{-0.1em}
  \end{minipage}
}

 {\bf Multinomial downsampling:}  
\begin{inparaenum}[\bfseries{(a)}]  \item on each $\mathbf{S}^2_r$ centered at $\mathbf{x}$, we collect the omni-directional response from the entire point-cloud. \item for each $i^{th}$ point in the point-cloud, we assign a value $v_i$ to be the largest response collected on $\mathbf{S}^2_r$. \item we use normalized $v_i$ as the sampling probability for the multinomial distribution. We draw $n<N$ number of sub-sampled points from this multinomial distribution. \end{inparaenum} In the following proposition, we claim that the set of points selected according to the largest responses lie on the convex hull of the point-cloud (an example is shown in Fig. \ref{classmodel:fig0}). 
    
\begin{proposition}
\label{classmodel:prop1}
Let $J$ be the set of indices corresponding to the points with the largest responses. Then $\left\{\mathbf{x}_j\right\}_{j\in J}$ lie on the convex hull of the point-cloud. 
\end{proposition}

\begin{proof}
We will prove it by contra-positive. Let $\mathbf{y}\in \mathbf{S}^2_r$. Let $\mathbf{x}_k$ be the point that does not lie on the convex hull of the point-cloud, i.e., there exists a $\mathbf{x}_i$ such that $\mathbf{y}^t(\mathbf{x}_i - \mathbf{x}_k)<0$. Let $\widetilde{\mathbf{x}_l}$ be a point on the convex hull of the point-cloud, i.e.,  for $\mathbf{x}_i$ and $\mathbf{y}\in \mathbf{S}^2_r$, $\mathbf{y}^t(\mathbf{x}_i - \widetilde{\mathbf{x}_l})>0$. Thus, $f_l(\mathbf{y}) > f_k(\mathbf{y})$, which concludes that 
$k \not\in J$. Thus by contra-positive, we conclude the proof.
\end{proof}   

Let $S = \left\{\mathbf{x}_i\right\}_{i\in I}$ be a given set of points. As mentioned before, we will collect response at each $\mathbf{x}_i\in S$ as in Algorithm \ref{classmodel:alg1}.
 \begin{algorithm}
 \KwData{Input $X=\left\{\mathbf{x}_i\right\}$, $S=\left\{\mathbf{x}_i\right\}_{i\in I}\subset X$, $r>0$}
 \KwResult{Responses $\left\{f_i: \mathbf{S}^2 \rightarrow \mathbf{R}\right\}_{i\in I}$}
 Generate a grid on $\left\{\mathbf{y}_j\right\}_{j=1}^K$ on $\mathbf{S}^2_r$\;
 For each $\mathbf{x}_i \in S$ and for each $\mathbf{y}_j$, assign $f_i(\mathbf{y}_j) = \sum_{\mathbf{x}_k \not\in \mathbf{B}^2_r(\mathbf{x}_i)} \max\left\{0, \mathbf{y}_j^t(\mathbf{x}_k - \mathbf{x}_i)\right\}$ \;

 \caption{Compute responses at each grid point on a given set of points $S$.}
 \label{classmodel:alg1}
\end{algorithm}
Observe that \begin{inparaenum} \item the proposed  multinomial downsampling strategy can be viewed as a data augmentation technique. \item as a consequence of this downsampling, our proposed model is robust to outliers. \end{inparaenum} Now that we have a set of spherical signals $\left\{f_i: \mathbf{S}^2 \rightarrow \mathbf{R}\right\}_{i\in I}$, the next step is to extract invariant features from the spherical signals.

\pichskip{5pt}
\parpic[r][b]{%
  \begin{minipage}{15mm}
  \centering
        \includegraphics[scale=0.30]{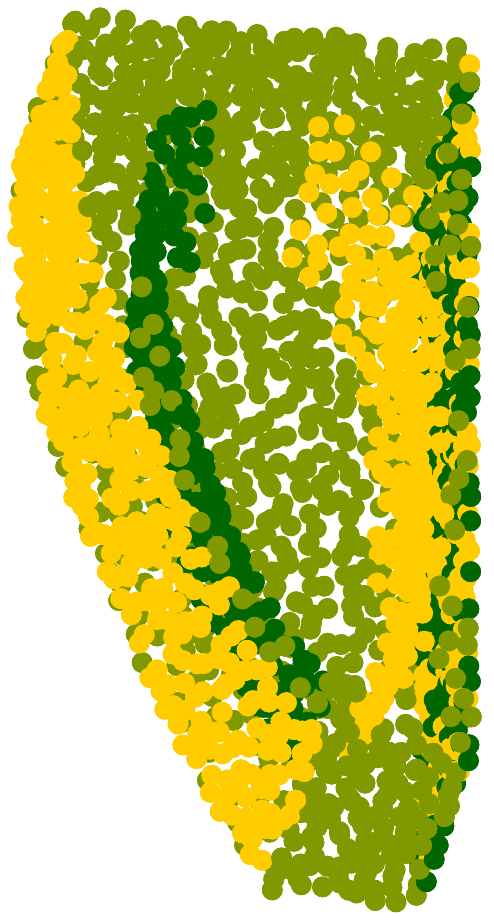}
    \vspace{-1em}
  \end{minipage}
}

{\bf Partitioning the response based on normal directions:} If we have normal information present, we can use this information to separate the collected responses. In Algorithm \ref{classmodel:alg1}, we change the construction of $f_i(\mathbf{y}_j)$ as follows:

\begin{align}
f^l_i(\mathbf{y}_j) = \sum_{{\substack{\mathbf{x}_k \not\in \mathbf{B}^2_r(\mathbf{x}_i) \\ \mathbf{n}_k^t\mathbf{n}_i \in [\frac{(l-1)\pi}{2n}, \frac{l\pi}{2n})}}} \max\left\{0, \mathbf{y}_j^t(\mathbf{x}_k - \mathbf{x}_i)\right\}, 
\end{align}
where, $n$ is the number of partitions we choose based on normal directions and $\mathbf{n}_k$ is the normal direction for $\mathbf{x}_k$. Thus we essentially partition the responses from $\left\{\mathbf{x}_k\right\}$ collected at $\mathbf{S}_r^2$ centered at $\mathbf{x}_i$ based on the similarity of normal at $\mathbf{x}_i$, denoted by $\mathbf{n}_i$ with the normal at $\mathbf{x}_k$, denoted by $\mathbf{n}_k$. An example for partitioning into 3 channels is shown in the adjacent figure.

\subsection{Generating spherical invariant response from the point-cloud:}
Given $\left\{f_i: \mathbf{S}^2 \rightarrow \mathbf{R}\right\}_{i\in I}$, we will use the spherical correlation network as defined in \cite{cohen2017convolutional}. For completeness, we will first briefly introduce the correlation operator before discussing the usage. 

{\bf Correlation operators:} 
\begin{definition}[$\mathbf{S}^2$ correlation]
\label{classmodel:def1}
Given $f:\mathbf{S}^2 \rightarrow \mathbf{R}$ (the signal) and $w:\mathbf{S}^2 \rightarrow \mathbf{R}$ (the learnable kernel), we define the correlation operator $f\star w: \textsf{SO}(3) \rightarrow \mathbf{R}$ as
\begin{align}
\label{classmodel:eq2}
\left(f\star w\right)(g) = \int_{\mathbf{S}^2} f(\mathbf{x}) w(g^{-1}\cdot \mathbf{x}) \omega(\mathbf{x}).
\end{align}
Here, $\omega$ is the chosen volume density on $\mathbf{S}^2$. As showed in \cite{cohen2017convolutional,chakraborty2018cnn}, the above definition of spherical correlation is equivariant to the action of $\textsf{SO}(3)$. 
\end{definition}
 As the output of the spherical correlation operator is a function on $\textsf{SO}(3)$ (the $3\times 3$ special orthogonal group), we will use the correlation operator on $\textsf{SO}(3)$ as defined in \cite{helgason}.
 
\begin{definition}[$\textsf{SO}(3)$ correlation]
\label{classmodel:def2}
Given $f:\textsf{SO}(3) \rightarrow \mathbf{R}$ (the signal) and $w:\textsf{SO}(3) \rightarrow \mathbf{R}$ (the learnable kernel), we define the correlation operator $f\star w: \textsf{SO}(3) \rightarrow \mathbf{R}$ as
\begin{align}
\label{classmodel:eq3}
\left(f\star w\right)(g) = \int_{\textsf{SO}(3)} f(\mathbf{x}) w(g^{-1} \mathbf{x}) \widehat{\omega}(\mathbf{x}).
\end{align}
Here, $\widehat{\omega}$ is the Haar measure on  $\textsf{SO}(3)$. This definition of $\textsf{SO}(3)$ correlation is equivariant to the action of $\textsf{SO}(3)$. 
\end{definition}

We will use the rotational equivariance property of the spherical correlation to extract rotation invariant features from the 3D point set. But in order to do that, we need to address the following questions \begin{inparaenum} \item what is the effect of rotation on the spherical signal? \item how to get invariant features? \end{inparaenum}

{\it What is effect of rotation on the spherical signal:} If we rotate the 3D point-cloud by a rotation matrix $R$, the responses as computed using Algorithm \ref{classmodel:alg1} will be rotated by the same matrix, this is stated in the following proposition. 

\begin{proposition}
\label{classmodel:prop2}
If the point-cloud $X$ is rotated by the matrix $R$, then the corresponding responses $\left\{f_i\right\}$ are rotated by the matrix $R$.
\end{proposition}
\begin{proof}
Given $r>0$ and $\mathbf{x}_i \in X$, let $\mathbf{x}_j \not\in \mathbf{B}^2_r(\mathbf{x}_i)$. Let $\widetilde{X} = RX$ be the rotated point set. Let $\mathbf{y} \in \mathbf{S}^2_r$, then $\mathbf{y}^t\left(\mathbf{x}_j - \mathbf{x}_i\right)$ becomes $\mathbf{y}^t R\left(\mathbf{x}_j - \mathbf{x}_i\right)$ after rotation. And as $f_i$ in Eq. \ref{classmodel:eq1} is sum of he responses $\mathbf{y}^t\left(\mathbf{x}_j - \mathbf{x}_i\right)$ for all $\mathbf{x}_j \not\in \mathbf{B}^2_r(\mathbf{x}_i)$, this concludes the proof.
\end{proof}

Thus by the above proposition, with rotation of the point-cloud, the spherical signal also gets rotated. And observe that as $\mathbf{S}^2$ and $\textsf{SO}(3)$ correlation operators are equivariant to rotations, the output features after cascaded $\mathbf{S}^2$ and $\textsf{SO}(3)$ correlation operators are also rotational equivariant. But, in order to extract rotational invariant features, we will use an invariant layer as defined next. But before that, we will define equivariance and invariance for completeness.

{\bf Equivariance and Invariance:} Given $f:\mathbf{S}^2\rightarrow \mathbf{R}$ and $R \in \textsf{SO}(3)$, an operator on $f$, denoted by $\mathfrak{F}(f):\textsf{SO}(3)\rightarrow \mathbf{R}$ is \begin{inparaenum}[\bfseries (a)] \item {\it equivariant} to the action of $\textsf{SO}(3)$ if
\begin{align}
R\cdot \mathfrak{F}(f) = \mathfrak{F}(R\cdot f), 
\end{align}
\item {\it invariant} to the action of $\textsf{SO}(3)$ if
\begin{align}
\mathfrak{F}(R\cdot f) = \mathfrak{F}(f), 
\end{align}
\end{inparaenum} where, $R\cdot f(\mathbf{x}) := f(R^{-1} \cdot \mathbf{x})$ for all $\mathbf{x} \in \mathbf{S}^2$. 

{\it How to get invariant features:} As the output of the spherical correlation, $f:\textsf{SO}(3)\rightarrow \mathbf{R}$ is equivariant to the action of $\textsf{SO}(3)$, we will integrate $f$ over $\textsf{SO}(3)$ with respect to the Haar measure to get $\textsf{SO}(3)$ invariant feature. This entails ``quotienting'' out the group $\textsf{SO}(3)$, which results the invariant features with respect to $\textsf{SO}(3)$. 

Notice that we will use $\mathbf{S}^2$ correlation followed by cascaded $\textsf{SO}(3)$ correlation with intermediate ReLU and normalization operators, followed by the invariant layer to generate rotation-invariant features.

{\bf Non-linear operator:} In order to design a deep architecture, it is essential to use some amount of non-linearities in-between correlation operators. As the aforementioned correlation operator is $\mathbf{R}$ valued, we can use ReLU operator as our choice of non-linearity.

{\bf Normalization:} It is well-known that normalization is crucial for stability of optimization and even achieving better optimum. We will resort to two types of normalization schemes, namely {\it Batch-normalization} \cite{ioffe2015batch} and {\it Activation-Normalization} \cite{kingma2018glow}.  We use PyTorch implementation of Batchnorm. For Actnorm, given an input tensor $T \in \mathbf{R}^{B \times c \times h \times w}$ ($B, c, h, w$ denotes the batch size, number of channels and spatial resolution respectively.), we learn the scaling tensor $H \in \mathbf{R}^{1\times c\times 1\times 1}$ and bias tensor $B\in \mathbf{R}^{1\times c\times 1\times 1}$ to get the normalized output as $T \mapsto H\odot T + B$.

A schematic depicting the pipeline to extract rotation invariant features from spherical signal is shown in Fig.~\ref{classmodel:fig2}.

 \begin{figure}[!ht]
 \centering
\begin{tikzpicture}
\node[inner sep=0pt] (fig1) at (0,0)
    {\includegraphics[width=0.12\textwidth, height=2.3cm]{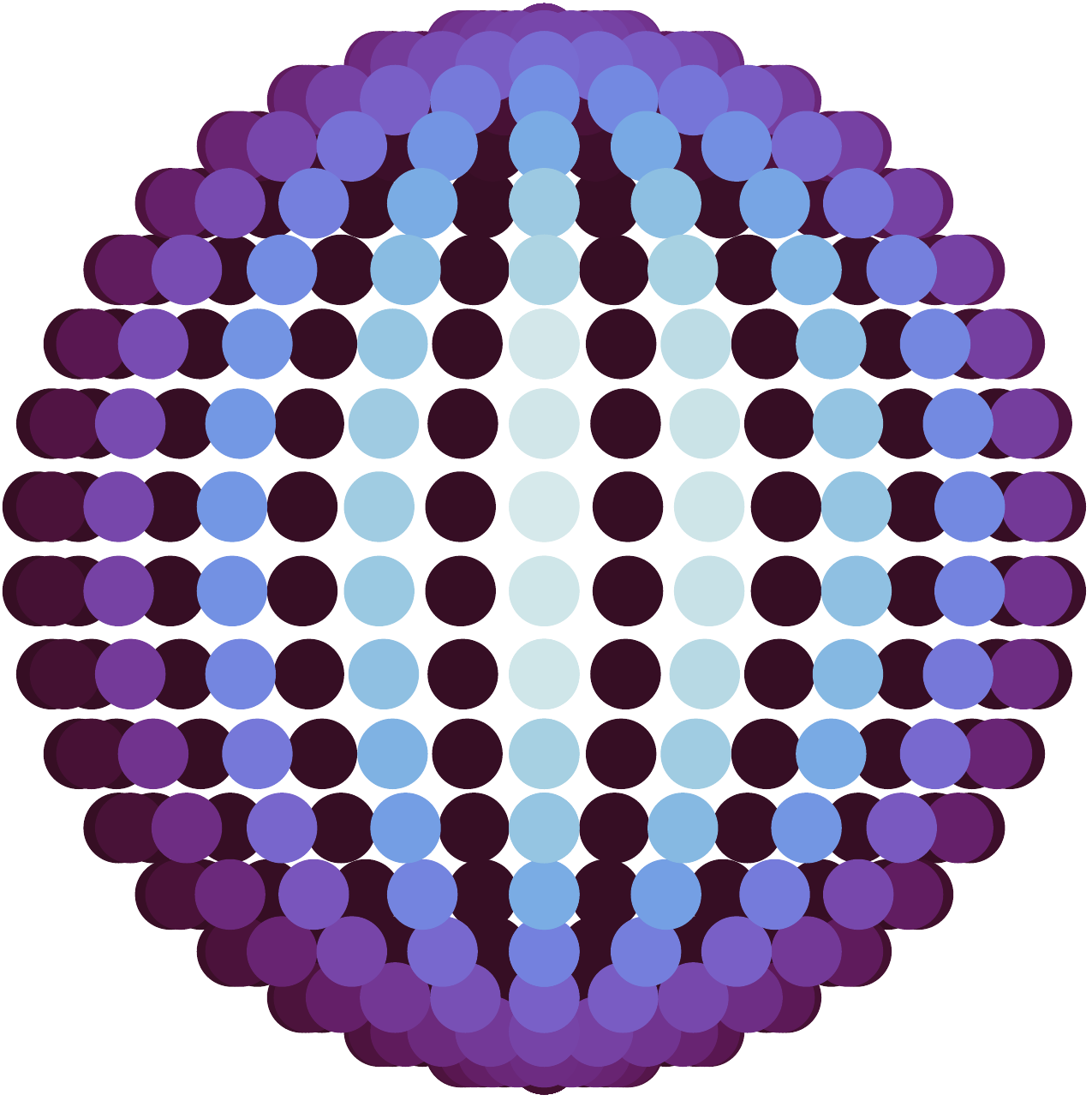}};
\node[inner sep=0pt] (fig2) at (3,0)
    {\includegraphics[width=0.12\textwidth, height=2.1cm]{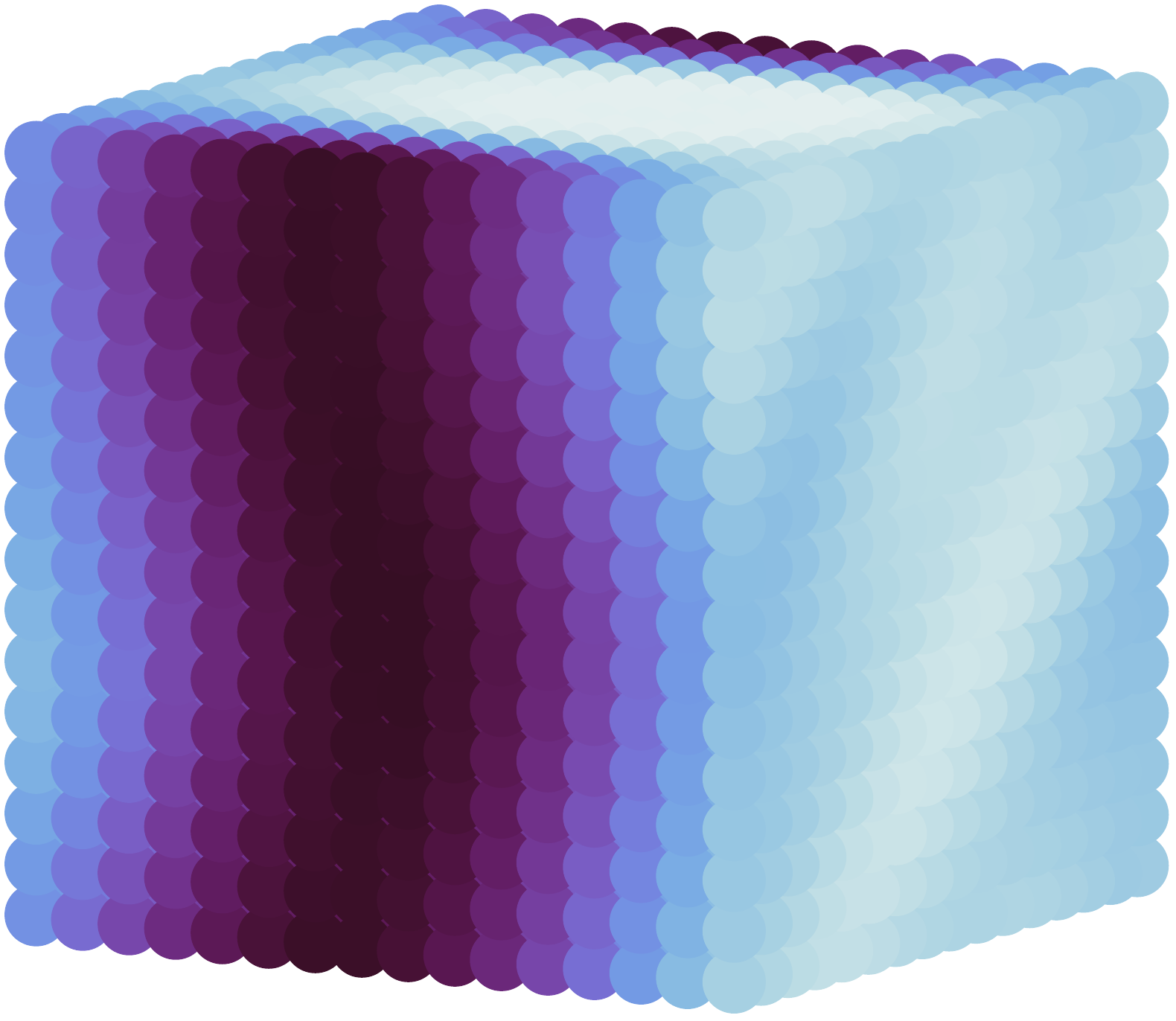}};
    \node[inner sep=0pt] (fig3) at (6,0)
    {\includegraphics[width=0.12\textwidth, height=2.1cm]{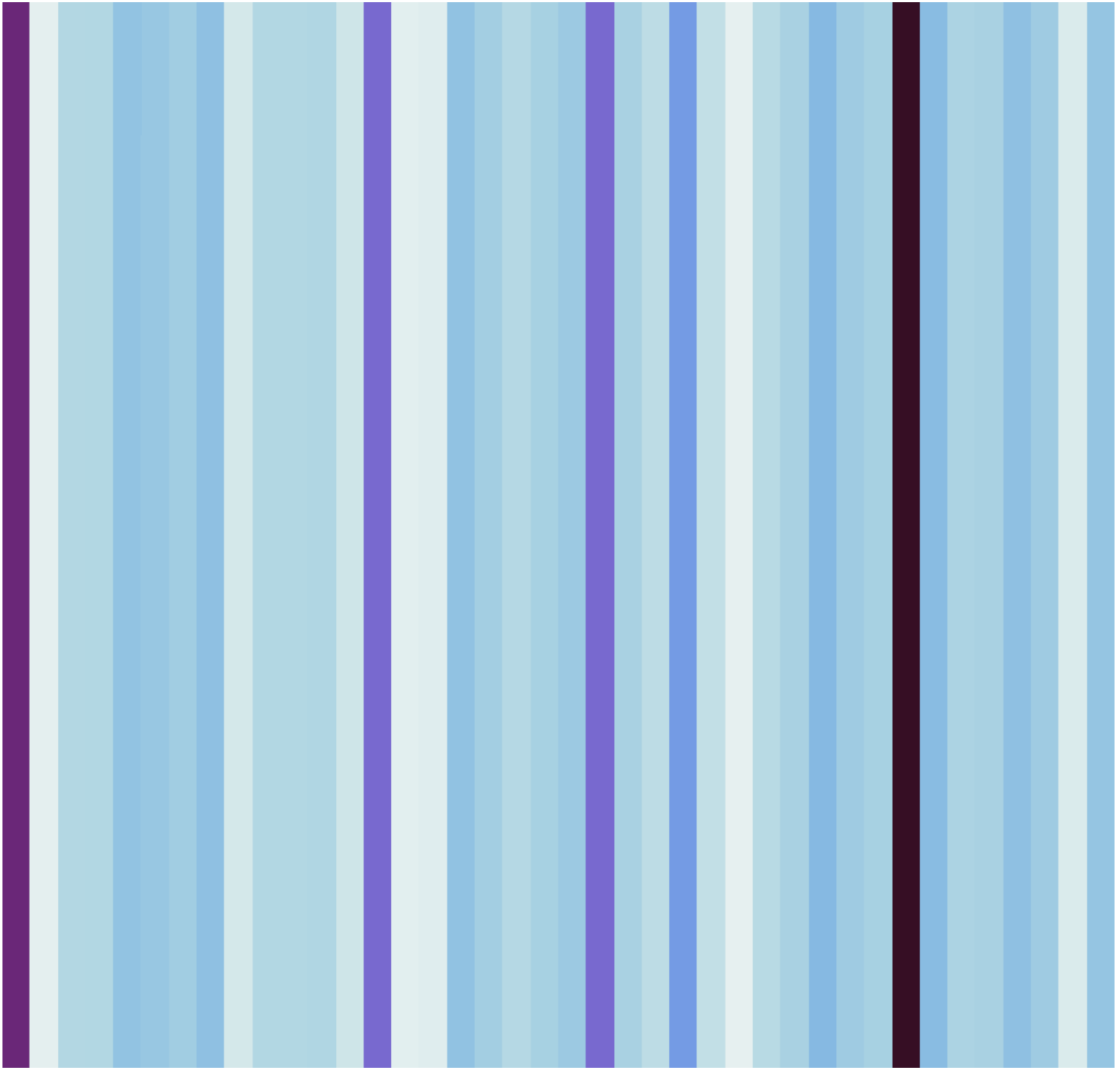}};
    \draw[->, line width=0.8mm] (1.1,0) -- (1.9, 0);
    \draw[->, line width=0.8mm] (4.1,0) -- (4.9, 0);
\end{tikzpicture}
\caption{A schematic from response on a sphere to the final invariant layer.}
      \label{classmodel:fig2}
\end{figure}

\subsubsection{Reduction of parameters}

Before moving forward with the rest of our proposed model architecture, we will introduce a leaner variant of spherical correlation operator, {\it Tensor Ring}. Observe that the correlation operators defined in Eqs. \ref{classmodel:eq2}, \ref{classmodel:eq3} are going to be computed in Fourier basis of the respective domain. For example, on $\mathbf{S}^2$, we use Spherical Harmonics basis \cite{hobson1931theory} to compute spherical correlation operator. Thus, for spherical correlation, we denote the learnable kernel $w:\mathbf{S}^2 \rightarrow \mathbf{R}$ with the corresponding coefficient (with respect to Spherical Harmonics basis) matrix $W\in \mathbf{R}^{C_{\text{in}}\times N\times N\times C_{\text{out}}}$, where, $C_{\text{in}}$, $C_{\text{out}}$ are the input and output number of channels and $N \times N$ is the resolution of the parametrization of $\mathbf{S}^2$. 

  \begin{figure}[!ht]
 \centering
      \includegraphics[scale=0.45]{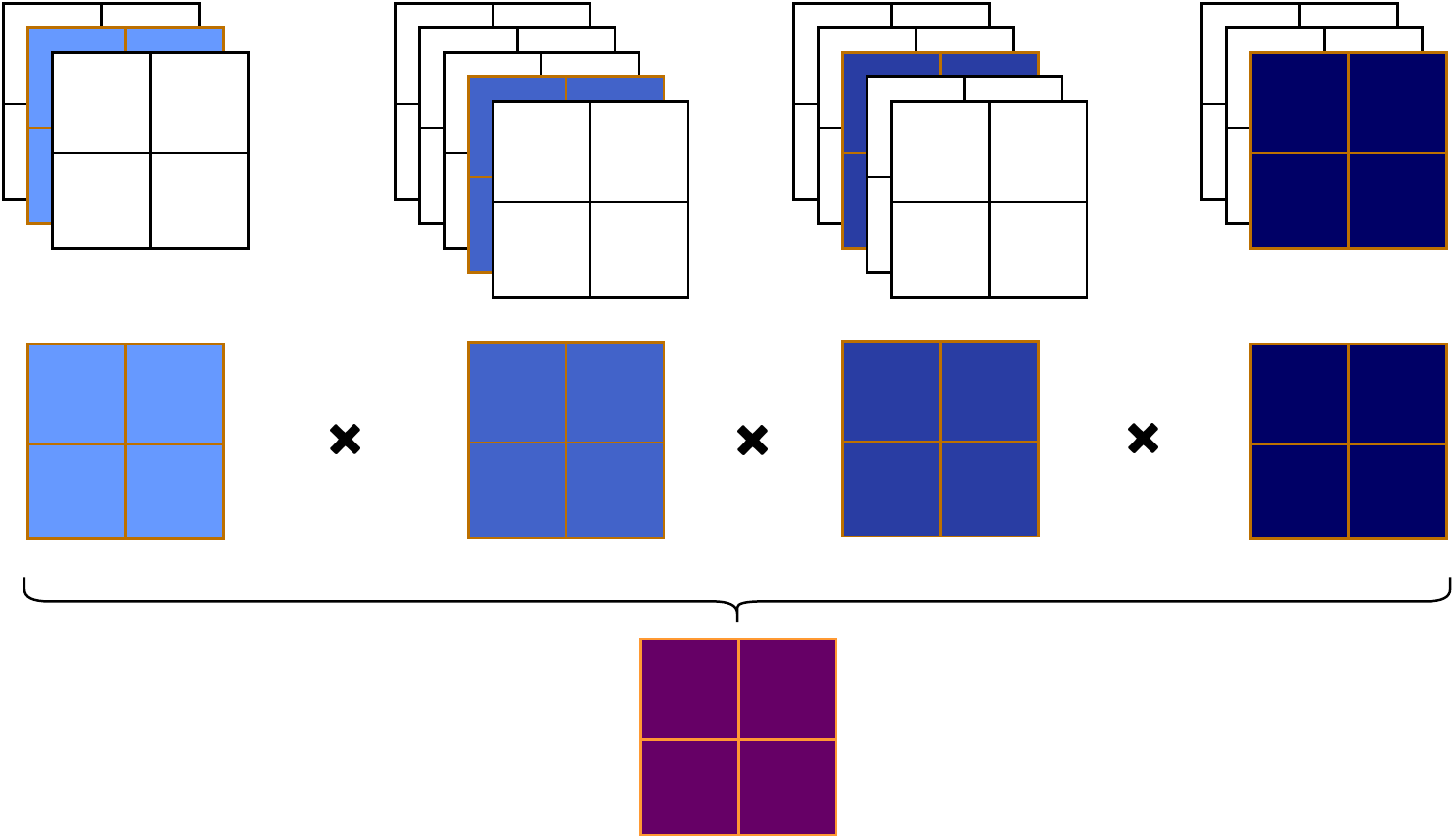}
       \caption{An example demonstrating Tensor Ring as given in Eq. \ref{eq_tr} with $k=4, p=2, n_1=3, n_2 = 5, n_3 = 5, n_4 = 3$. Here after the multiplication of tensor cores we use trace operator on the result.}
      \label{classmodel:figm3}
\end{figure}

Any tensor of the size $W\in \mathbf{R}^{n_1\times \cdots n_k}$ can be decomposed using Tensor Ring \cite{zhao2016tensor} as follows:
\begin{align}
\label{eq_tr}
W(i_1, \cdots i_k) = \text{trace}\left(T_1(i_1)\cdots T_k(i_k)\right),
\end{align}
where, $T_j(i_j) \in \mathbf{R}^{p\times p}$ and $i_j \in [1, \cdots n_j]$ for all $j$. Here $p$ is the size of the core of the tensor ring. It has been shown that under some assumptions any tensor can be decomposed in tensor ring form with arbitrary approximation error \cite{zhao2016tensor}. An example demonstrating tensor ring in Fig. \ref{classmodel:figm3}. Notice that this form of decomposition amounts of learning $p^2(\sum_{j=1}^k n_j)$ number of parameters compared to $\prod_{j=1}^k n_j$ number of parameters. To give a concrete example, for a tensor $W \in \mathbf{R}^{10\times 20\times 20 \times 40}$ with a core size $p=5$, using tensor ring we use $2250$ number of parameters instead of $160000$, which amounts to {it $98.59\%$ parameters reduction.}

In our implementation of spherical correlation, we use the tensor ring form in order to implement the correlation operator. We will use $\Phi^l$ to denote the spherical correlation block to extract rotation invariant local features. After extracting this features from each of the selected points, $\left\{\mathbf{x}_i\right\}_{i\in I}$, our algorithm to combine features is different for classification and segmentation. Below, we will first discuss our proposed scheme for segmentation followed by classification. 

\subsection{Combining local and global features:} Before describing the scheme to combine the extracted features,  we will first discuss the necessity of extracting global features from a point-cloud. {\it Why global features are needed?} It is obvious that global features are helpful for classification, though correlation is a powerful operator to extract local features, even for standard (Euclidean) correlation operators, aggregating local features, e.g., pooling is a necessity. As for classification we need to find a single consensus for an entire point-cloud, aggregating local features to get a global response is desired. Moreover in order to do semantic/part segmentation, just local feature at each point in the point-cloud is not enough as we need to find spatial context of each point with respect to the global structure, e.g., in order to do part segmentation of a earphone in terms of head band and ears, we need information in addition to local features. An example showing the usefulness of global features is shown in Fig. \ref{classmodel:figm4}. Here the two selected points have similar local features  (after quotient out rotations) but they belong to two different classes. Hence the need for global information is justified.

\begin{table}
\centering
\renewcommand{\arraystretch}{3}
\begin{tabular}{ccc}
    \centering
  \multirow{4}[2]{*}[3mm]{\includegraphics[width=4cm]{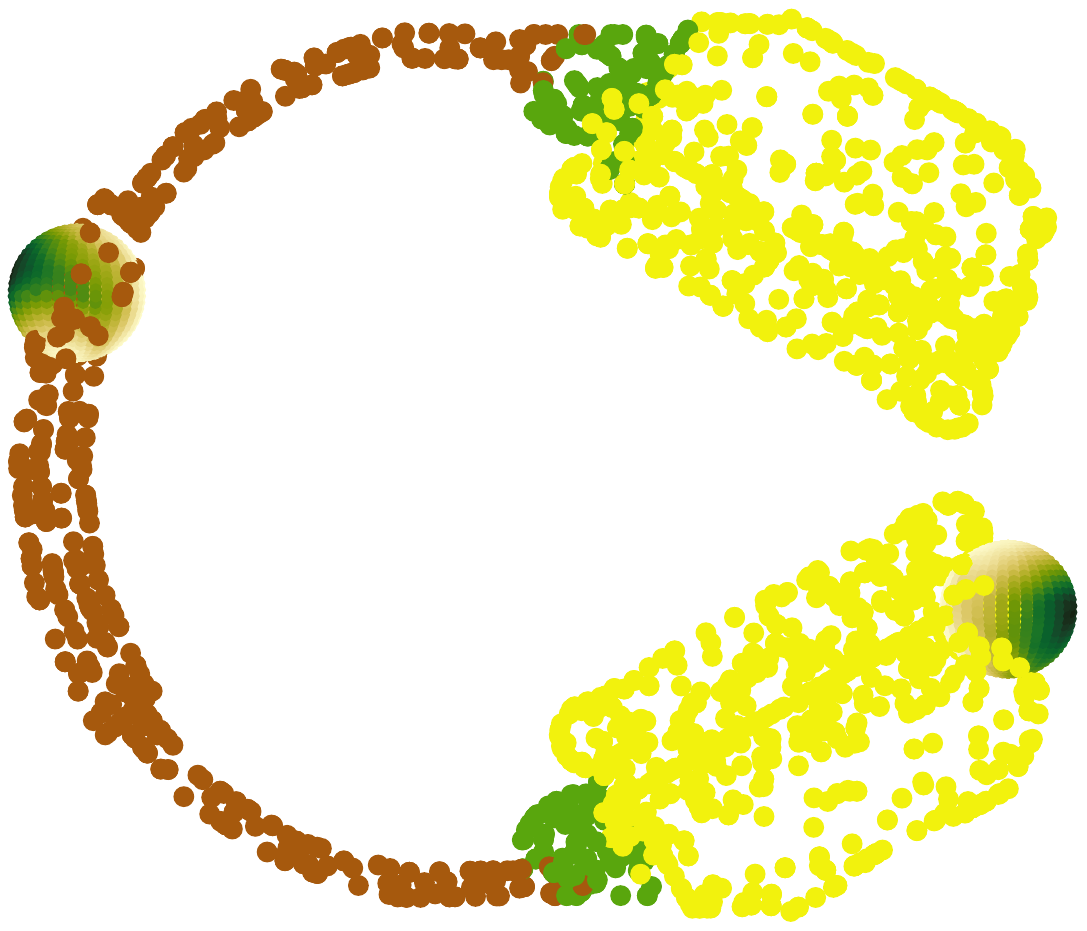}}    & \multirow{2}[2]{*}[2mm]{\includegraphics[width=1.5cm]{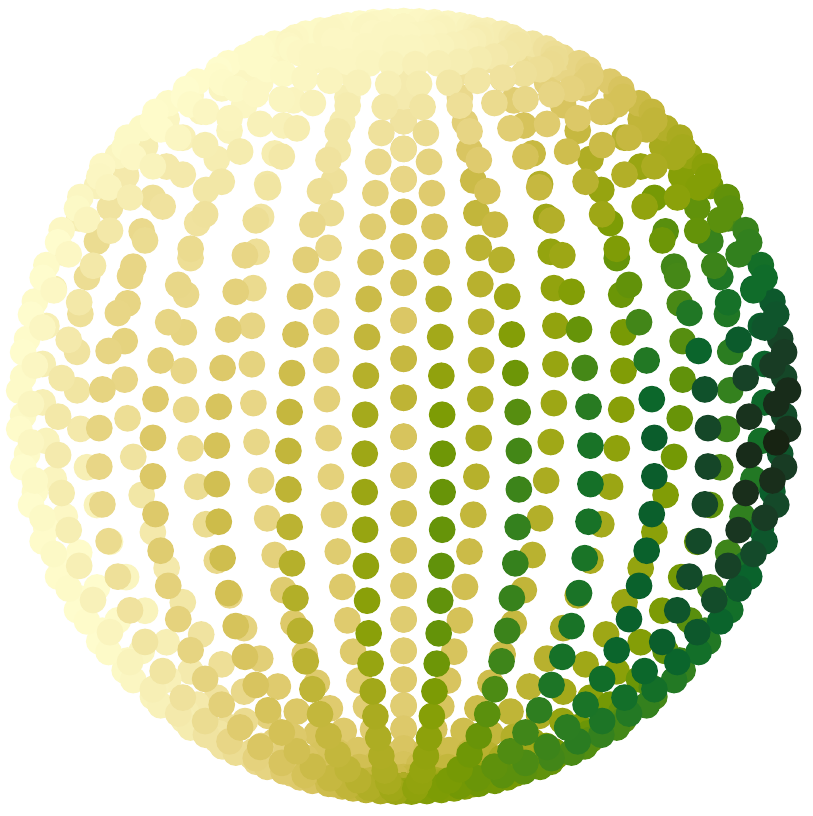}} \bigstrut \\
   & & \\
 
                                                                             & \multirow{1}[1]{*}[2mm]{\includegraphics[width=1.5cm]{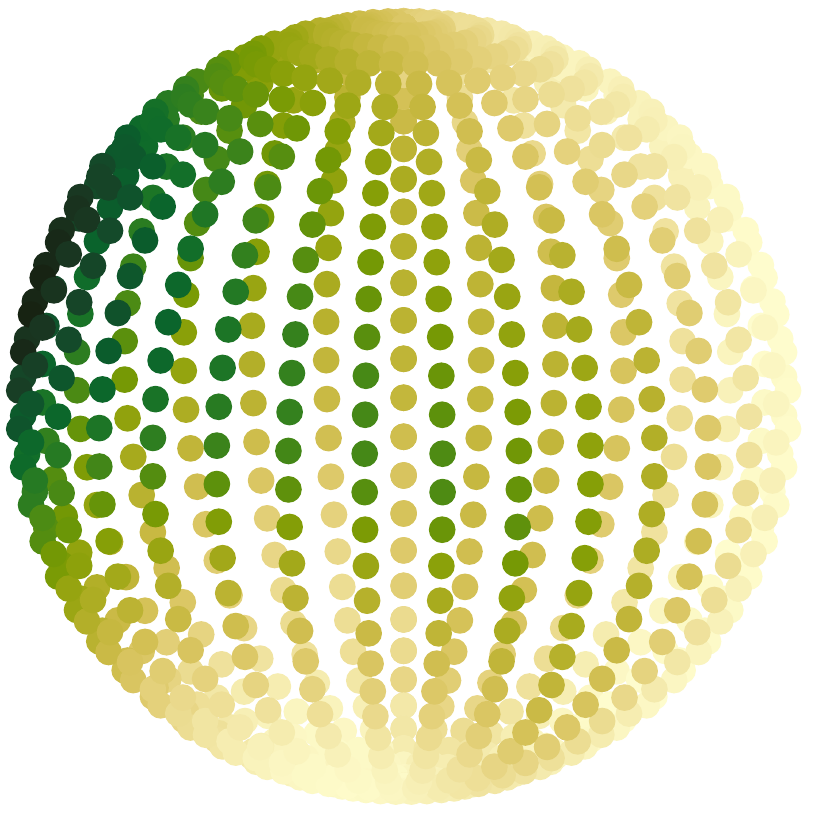}}  \\
\end{tabular}
\vspace{1em}
\captionof{figure}{Local responses across different classes look similar.}
\label{classmodel:figm4}
\vspace{-1em}
\end{table}

Now that we justify the necessity of the global features, we propose our scheme to extract global features. We divide our scheme into two subsections, namely \begin{inparaenum}[\bfseries (1)] \item extraction and combining features for segmentation and \item extraction and combining features for classification. \end{inparaenum}

\subsubsection{Extraction of global features: Segmentation}
Given a non-empty $S = \left\{\mathbf{x}_i\right\}_{i\in I} \subset X$ as the given convex hull for the point-cloud $X$, we will define global {\it affine coordinate} of each $\mathbf{x}_i$ as follows: If 
\begin{align}
\label{classmodel:eq4}
\mathbf{x}_i = \sum_{j|\mathbf{x}_j \in S} c_{ij} \mathbf{x}_j,
\end{align}
then, we will say $(c_{i1}, \cdots, c_{i|S|})^t$ to be the affine coordinates of $\mathbf{x}_i$ where $\sum_j c_{ij}=1$. Let $\Phi^g_S: X \rightarrow \mathbf{A}^{|S|} \subset \mathbf{R}^{|S|}$ be the function which returns affine coordinates for each $\mathbf{x}_i$ given the convex hull $S$. Notice that here $\mathbf{A}^n$ is the affine subspace of $\mathbf{R}^n$ and $|S|$ is the cardinality of set $S$. It is easy to see that the affine coordinate is rotation invariant as stated formally in the next proposition.
\begin{proposition}
\label{classmodel:prop3}
Given $S$, $\Phi^g_S(R\cdot \mathbf{x}_i) = \Phi^g_S(\mathbf{x}_i)$ for all $R\in \textsf{SO}(3)$ and for all $\mathbf{x}_i \in X$. 
\end{proposition}
\begin{proof}
The proof follows from the linearity of the affine coordinates as defined in Eq. \ref{classmodel:eq4}.
\end{proof}
Now, that we have local and global rotation invariant features, we combine them to do segmentation as illustrated in Algorithm \ref{classmodel:alg2}.
 \begin{algorithm}
 \KwData{Input $X=\left\{\mathbf{x}_i\right\}$, $S=\left\{\mathbf{x}_i\right\}_{i\in I}\subset X$, $\Phi^l, \Phi^g_S$}
 \KwResult{$\Phi^{lg}_S$}
For each $\mathbf{x}_j \in S$, compute $\Phi^l(\mathbf{x}_j)$ \;
For each $\mathbf{x} \in X$, interpolate $\Phi^l(\mathbf{x})$ from $\left\{\Phi^l(\mathbf{x}_j)\right\}_{\mathbf{x}_j\in I}$ according to Algorithm \ref{classmodel:alg3} \;
For each $\mathbf{x} \in X$, compute $\Phi^g_S(\mathbf{x}$ according to Eq. \ref{classmodel:eq4}\;
For each $\mathbf{x} \in X$, use a self-attention block to compute the probabilities for local and global features (as defined in Algorithm \ref{classmodel:alg4}). Let the probabilities be denoted by $p_l$ and $p_g$ respectively \;
Define $\Phi^{lg}_S$ to be $\mathbf{x} \mapsto \left(p_l(\mathbf{x}) \Phi^l(\mathbf{x}) | p_g(\mathbf{x}) \Phi^g_S(\mathbf{x})\right)^t$, where $|$ denotes the concatenation of local and global features;
 \caption{Extract features from $X$.}
 \label{classmodel:alg2}
\end{algorithm}

In the following algorithm, we will give the interpolation technique to interpolate the local features from the responses on $S$ (a schematic is shown in Fig.~\ref{classmodel:fig3}). Together with this and Algorithm \ref{classmodel:alg2}, we have a rotation invariant combined local and global features which we will use for segmentation using fully connected (FC) layer(s).
 
  \begin{figure}[!ht]
 \centering
      \includegraphics[scale=0.4]{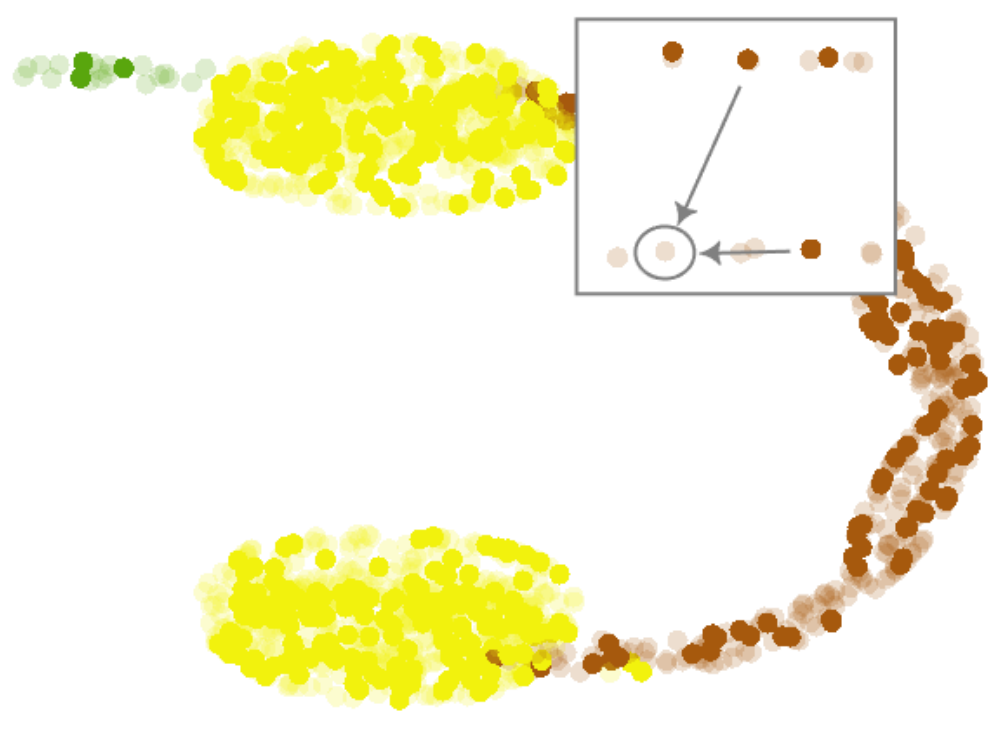}
       \caption{The dark brown points are the nearest points of the light brown point whose features need to be interpolated.}
      \label{classmodel:fig3}
\end{figure}
 
 \begin{algorithm}
 \KwData{Input $X=\left\{\mathbf{x}_i\right\}$, $S=\left\{\mathbf{x}_i\right\}_{i\in I}\subset X$, $\Phi^l(S)$, $k\geq 2$}
 \KwResult{$\Phi^{l}(X)$}
For each $\mathbf{x} \in X$, we find the $k$-nearest neighbors in $S$, denoted by $\mathcal{N}_{\mathbf{x}} \subset S$ \;
For each $\mathbf{y}\in \mathcal{N}_x$, we assign the weight, $w_{\mathbf{xy}} = \sfrac{1}{\|\mathbf{x}-\mathbf{y}\|^2}$ \;
For each $\mathbf{x} \in X$, interpolate $\Phi^l(\mathbf{x}) = \sum_{\mathbf{y}\in S}  w_{\mathbf{xy}} \Phi^l(\mathbf{y})$ \;
\caption{Interpolation of rotation invariant local features.}
 \label{classmodel:alg3}
\end{algorithm}

 \begin{algorithm}
 \KwData{Input $X=\left\{\mathbf{x}_i\right\}$, $\left\{\Phi^l(X)\right\}$}
 \KwResult{$(p_l, p_g)^t$}
For each $\mathbf{x} \in X$, use $\Phi^l$ on the $k$-nearest neighbors $\mathcal{N}_{\mathbf{x}}$ \;
We use 1D correlation with kernel size $1$ to look at the interaction between the neighbors \;
We use multiple FC layers to learn feature followed by a softmax layer to extract $(p_l(\mathbf{x}), p_g(\mathbf{x}))^t$ for all $\mathbf{x} \in X$ \;
\caption{Self-attention block to combine local and global features.}
 \label{classmodel:alg4}
\end{algorithm}

A schematic of our proposed segmentation pipeline is shown in Fig.~\ref{classmodel:fig4}.

\begin{figure*}[!ht]
 \centering
      \includegraphics[scale=0.22]{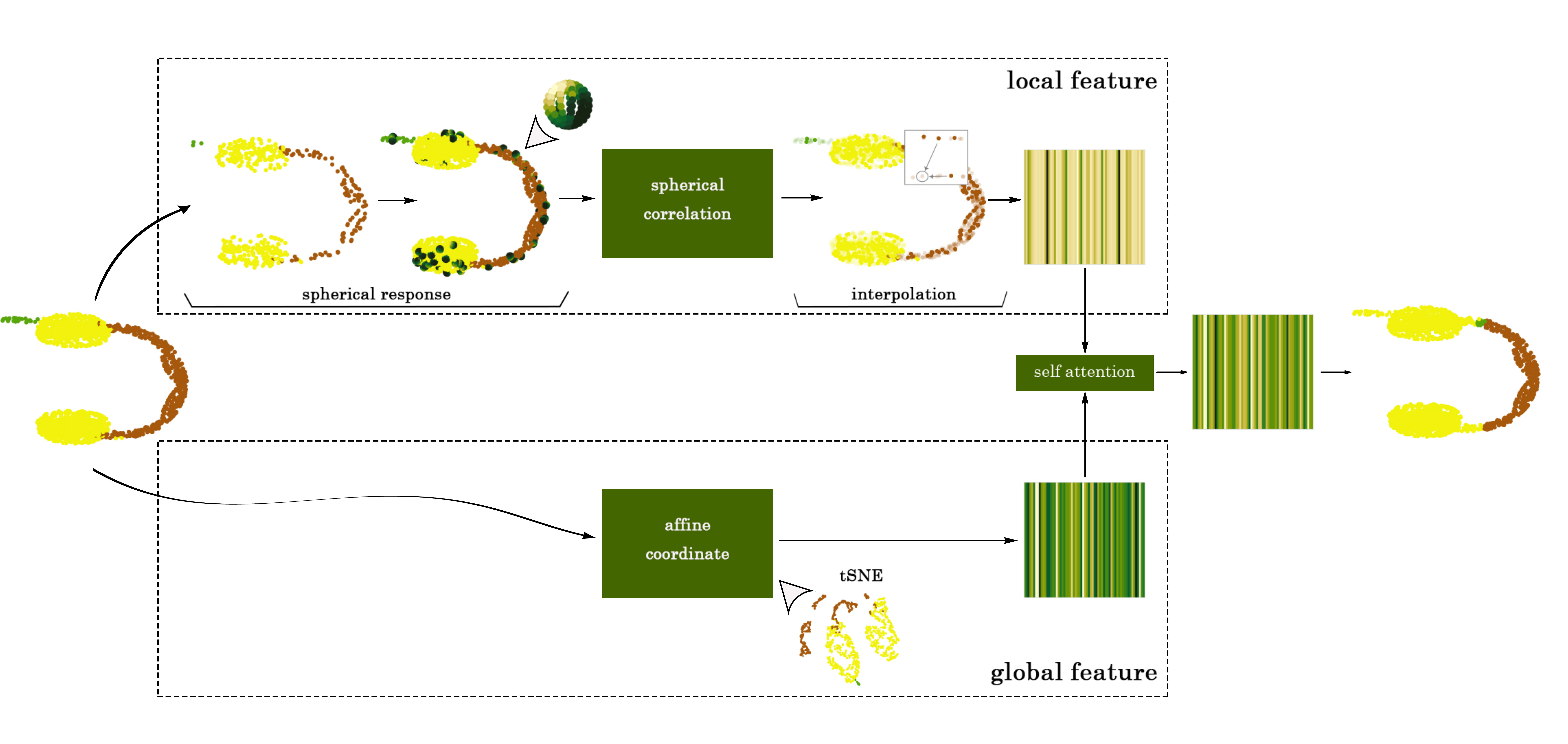}
       \caption{A schematic of the pipeline for segmentation block. As affine coordinates is in high dimension, we use a tSNE 2D embedding for visualization.}
      \label{classmodel:fig4}
\end{figure*}
 \subsubsection{Extraction of global features: Classification}
After extraction of rotation invariant spherical local features, we aggregate the local features using either one of the following two ways over the entire point-cloud $X$, \begin{inparaenum}[\bfseries (a)] \item maxpooling the local features over $S$. \item use $1\times 1$ correlation operator over $S$ to combine the local features.\end{inparaenum} The usage of maxpooling to extract global features will make sure that the extracted features are permutation invariant. A schematic of our proposed classification pipeline is shown in Fig.~\ref{classmodel:fig5}, which shows the rotation invariance property on the rotated and non-rotated examples.

\begin{figure*}[!ht]
 \centering
      \includegraphics[scale=0.18]{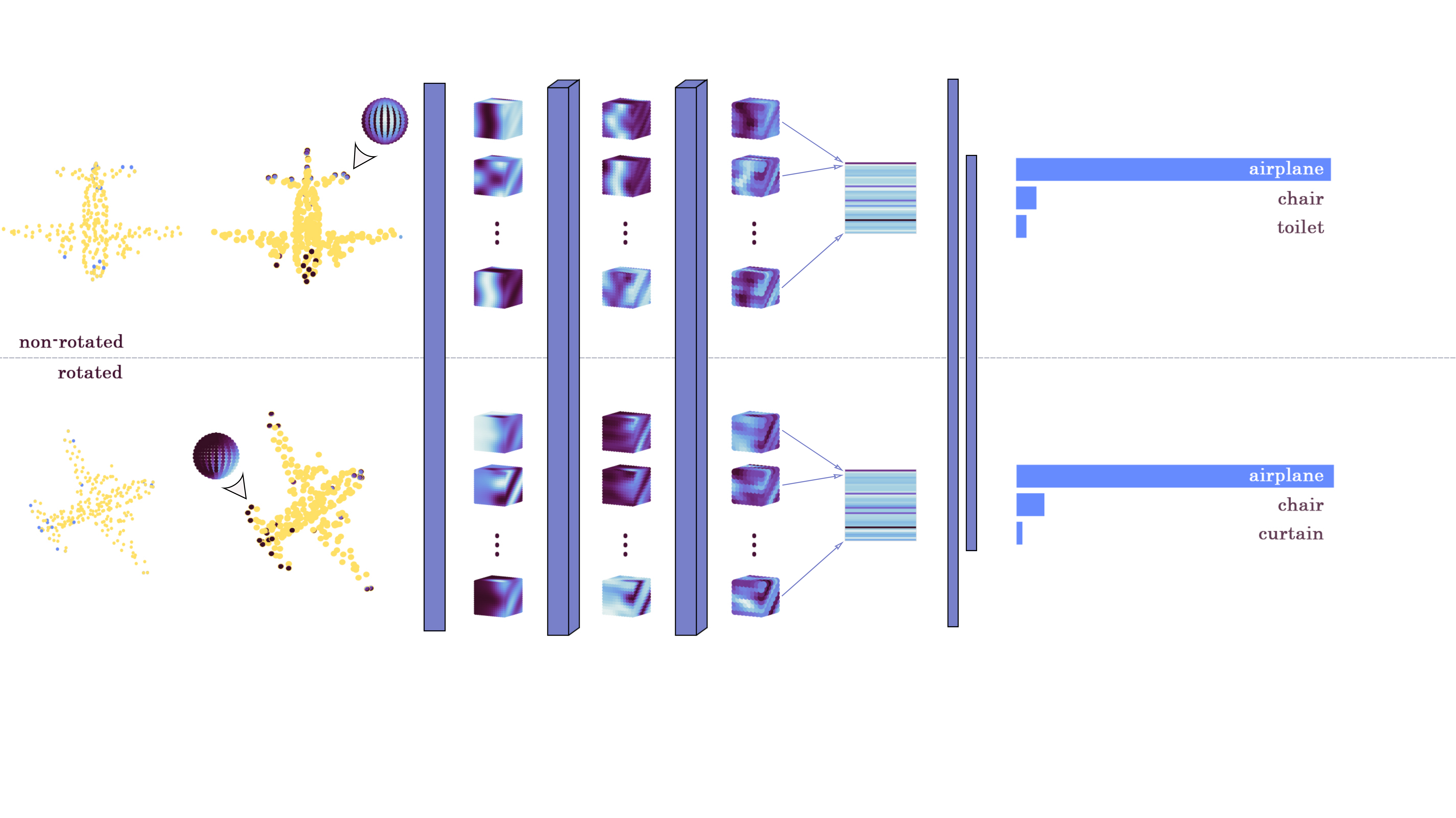}
       \caption{A schematic of the pipeline for classification block. The vertical blocks denote $\mathbf{S}^2$, $\textsf{SO}(3)$, $\textsf{SO}(3)$ and FC layers respectively with filter outputs in-between. The response before the FC layer denotes the output after integration.}
      \label{classmodel:fig5}
\end{figure*}

\subsection{(Implicit) Data augmentation}

In this section, we will talk about data augmentation in our proposed framework. As argued so far that our proposed model is rotational invariant, which makes our model implicitly {\it rotation invariant}, this entails an implicit data augmentation. As pointed out before, the proposed downsampling step amounts of doing explicit data augmentation of {\it random sampling}. In the following section, we will talk about another kind of explicit data augmentation, namely {\it random deformation}. Before explaining deformation augmentation in detail, we like to remind the readers about the benefits of these variants of augmentations: \begin{inparaenum}[\bfseries (a)]  \item implicit rotation augmentation makes it rotation invariant \item random sampling augmentation makes it robust to noise \item random deformation augmentation makes it small deformation invariant. \end{inparaenum}

\subsubsection{Deformation augmentation}
In this section, we propose a scheme to make our model robust to small deformations. Given $X$ as the input point-cloud, we apply some amount of deformation of the coordinates in the following way: use a neural network consists of fully connected layers and $\tanh$ activation functions. After deforming the points, we apply Algorithms \ref{classmodel:alg1}-\ref{classmodel:alg3} to extract the spherical features $\Phi^l(X)$.

\begin{figure*}[!ht]
    \centering
   \includegraphics[scale=0.45]{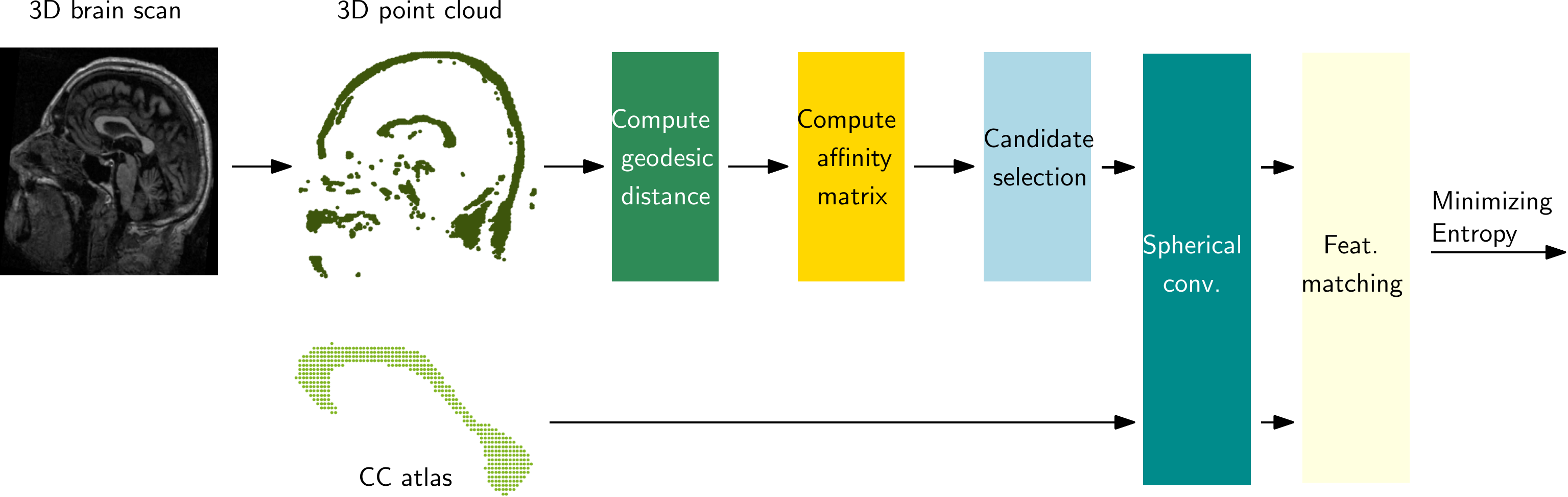}
\caption{A schematic for segmentation from 3D point-cloud extracted from 3D brain scan.}
\label{fig:cc_seg_flow}
\end{figure*}

We then combine these features with the global features. An example showing the deformed point-cloud is given in Fig. \ref{fig:cc_deform}.

\begin{figure}[!ht]
    \centering
   \includegraphics[scale=0.45]{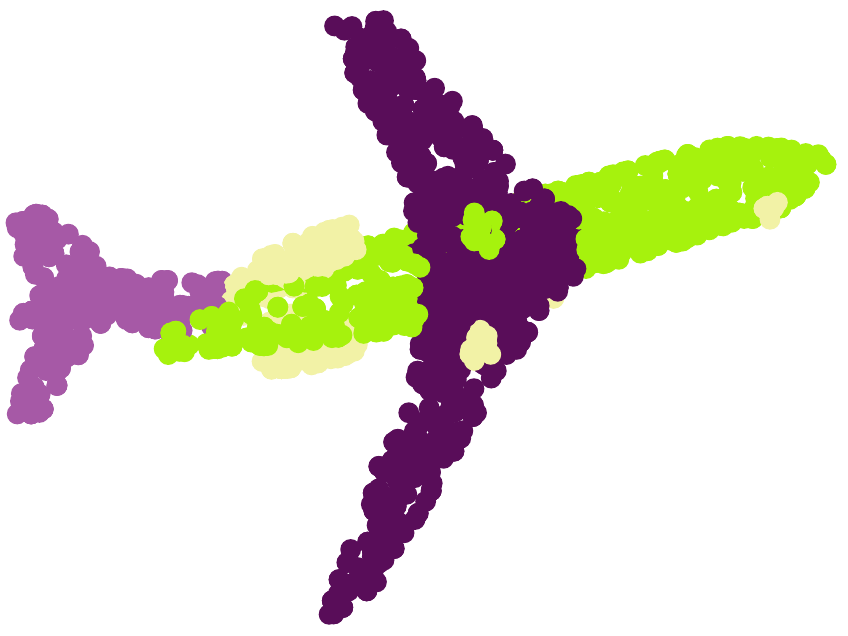}
   \includegraphics[scale=0.45]{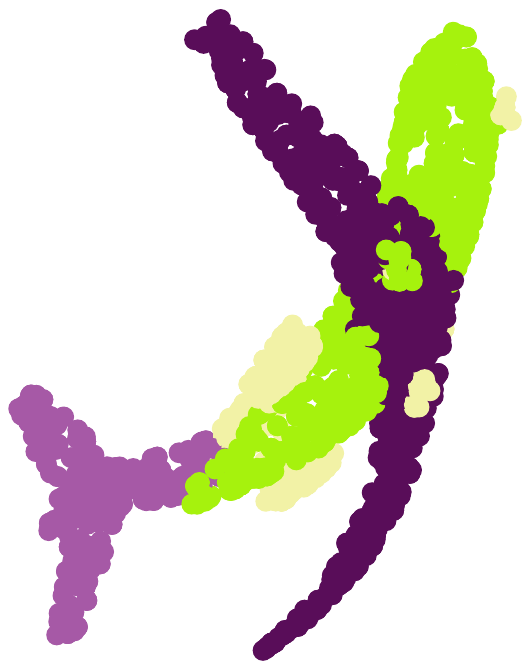}
\caption{Original point-cloud and together with its corresponding deformed coordinates.}
\label{fig:cc_deform}
\end{figure}

Due to the inherent property of rotation invariance, we name our model ``Rotation Invariant omni-directional network for pointsets'' (or `POIRot' as a short name). In the following section, we will show an extensive comparison on POIRot with the state-of-the-art methods on benchmark datasets for both classification and segmentation tasks. 

\section{Experimental Results}\label{sec:results}
In this section, we present experimental validations in a supervised setting for classification and segmentation tasks. In classification task, we use three datasets, including MNIST, Modelnet40 and OASIS dataset. In segmentation task, we do part segmentation on Shapenet dataset. We also present an experiment for unsupervised object detection task from 3D brain shape. 

\subsection{Part Segmentation}
One of the challenging tasks in 3D point-cloud data processing is to do part segmentation. This task entails to segment each point-cloud in separate categorical labels. We evaluate the performance of our proposed model on Shapenet part segmentation dataset \footnote{Due to noise present in the labels, we relabel some of the labels of the dataset.} \cite{chang2015shapenet}. This dataset consists of  $16,881$ shapes from $16$ types of objects annotated with $50$ parts. We have performed similar experimental setup as proposed in \cite{pointnet}. Our proposed segmentation scheme outperforms the state-of-the-art methods by $>3\%$ in terms of {\it mean intersection over union} (mIoU) evaluation metric. The comparative results in terms of mIoU for different categories of objects are given in Table \ref{tab:partseg}. We can see that the proposed method performs consistently well for all objects, while makes significant performance improvement for some objects including airplane and table. Some representative part segmentation results are shown in Fig. \ref{fig:seg_plots}. We show some non-obvious mistakes in annotations for ground truth and prediction by orange and blue circle respectively.  One can also see some obvious mistakes in ground truth annotations, e.g., in last row first and fifth columns. 

In terms of model complexity, we can see from Fig. \ref{fig:number_params} that our proposed method, POIRot takes $\approx 0.1-1\%$ compared with \cite{pointnet}, as this method is a representative of the most commonly used state-of-the-art algorithms. This clearly indicates the usefulness of a leaner model which not only makes it suitable for low-memory devices but also makes the optimization simpler and hence it achieves better optimum. This clearly indicates the usefulness of the proposed model. 

\begin{table*}[!ht]
\begin{center}

\resizebox{0.95\textwidth}{!}{

\begin{tabular}{>{\columncolor[gray]{0.95}}l|c|ccccccccccccccccc} 
\hline
\rowcolor{LightCyan}

                                       & \color{black}{Avg.}          & \color{black}{airplane}          & \color{black}{bag}           & \color{black}{cap}           & \color{black}{car}      & \color{black}{chair}          & \color{black}{earphone} & \color{black}{guitar} & \color{black}{knife} & \color{black}{lamp} & \color{black}{laptop} & \color{black}{motorbike} & \color{black}{mug} & \color{black}{pistol} & \color{black}{rocket} & \color{black}{skateboard} & \color{black}{table}         \\ 
\hline
\rowcolor{LightCyan2} \# shapes &              & 2690& 76& 55& 898& 3758& 69& 787& 392& 1547& 451& 202& 184& 283& 66& 152& 5271\\ 
\hline
3DCNN \cite{pointnet}                         & 79.4& 75.1& 72.8& 73.3& 70.0& 87.2& 63.5& 88.4& 79.6& 74.4& 93.9& 58.7& 91.8& 76.4& 51.2& 65.3& 77.1 \\
PointNet\cite{pointnet}                                & 83.7& 83.4& 78.7& 82.5& 74.9& 89.6& 73.0& 91.5& 85.9& 80.8& 95.3& 65.2& 93.0& 81.2& 57.9& 72.8& 80.6   \\
PointNet++ \cite{pointnet++}                            & 85.0& 82.4& 79.0& 87.7& 77.3& 90.8& 71.8& 91.0& 85.9& 83.7& 95.3& 71.6& 94.1& 81.3& 58.7& 76.4& 82.6      \\
FCPN   \cite{rethage2018fully}                                & 81.3& 84.0& 82.8& 86.4& \textbf{88.3}& 83.3& 73.6& \textbf{93.4}& 87.4& 77.4& \textbf{97.7}& \textbf{81.4}& 95.8& 87.7& 68.4& 83.6& 73.4   \\
DGCNN  \cite{DGCNN}             & 81.3& 84.0& 82.8& 86.4& 78.0& \textbf{90.9}& 76.8 & 91.1& 87.4& 83.0& 95.7& 66.2& 94.7& 80.3& 58.7& 74.2& 80.1 \\ 
PointCNN \cite{li2018pointcnn} & 84.9 & 82.7 & 82.8 & 82.5 & 80.0   & 90.1 & 75.8 & 91.3 & \textbf{87.8} & 82.6 & 95.7 & 69.8 & 93.6 & 81.1 & 61.5 & 80.1 & 81.9\\
SplatNet \cite{su2018splatnet}               & 84.6 & 81.9 & \textbf{83.9} & \textbf{88.6} & 79.5 & 90.1 & 73.5 & 91.3 & 84.7 & \textbf{84.5} & 96.3 & 69.7 & 95.0   & 81.7 & 59.2 & 70.4 & 81.3      \\
\hline
POIRot                   & \textbf{88.3} & \textbf{85.0} & 82.7 & \textbf{88.6}   & 79.5 & 89.6 & \textbf{79.58} & 90.96 & 84.5 & 81.5 &  96.3 & 80.8   & \textbf{96.1} & \textbf{88.8}   & \textbf{69.1} & \textbf{87.1} &\textbf{92.3} \\
\hline
\end{tabular}
} 
\end{center}

\caption{Part segmentation results in terms of mIoU(\%) on ShapeNet PartSeg dataset.}
\label{tab:partseg}
\vspace{1em}
    \centering
   \includegraphics[scale=0.28]{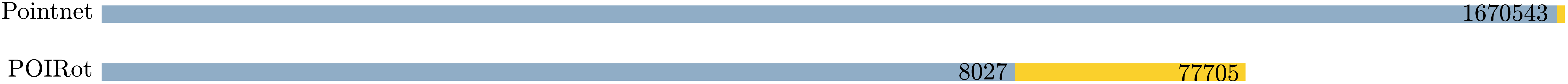}
\captionof{figure}{Number of parameters of POIRot compared to Pointnet on ShapeNet dataset.}
\label{fig:number_params}
\end{table*}

\subsection{Ablation Studies}
In this section, we perform ablation studies in the following two ways \begin{inparaenum}[\bfseries (a)] \item by removing global features \item by including deformation to show its usefulness. \end{inparaenum}From Table \ref{tab:ab_coor}, we can see that global features may or may not help. We conjecture that global feature may be useful for relatively complicated shapes. In Table \ref{tab:ab_deform}, we show the result with deformation, we can see that though the deformation data augmentation makes the model robust to deformations, it affects the performance, which is due to the trade-off between performance and achieving invariance.
\begin{table}[!ht]
\begin{center}

\resizebox{0.20\textwidth}{!}{

\begin{tabular}{>{\columncolor[gray]{0.95}}l|c|cc} 
\hline
\rowcolor{LightCyan}

                                       & \color{black}{w}          & \color{black}{w/o}       \\ 
\hline
cap  & \textbf{88.7} & 84.2 \\
earphone & \textbf{79.6} & 77.1 \\
mug & \textbf{96.1} & \textbf{96.1} \\
rocket & 67.0 & \textbf{69.1} \\
\hline
\end{tabular}
} 
\end{center}

\caption{Ablation of global features on ShapeNet PartSeg dataset in terms of mIoU(\%).}
\label{tab:ab_coor}
\end{table}

\begin{table}[!ht]
\begin{center}

\resizebox{0.20\textwidth}{!}{

\begin{tabular}{>{\columncolor[gray]{0.95}}l|c|cc} 
\hline
\rowcolor{LightCyan}

                                       & \color{black}{w/o}          & \color{black}{w/}       \\ 
\hline
earphone & \textbf{79.6} & 72.5 \\
rocket & \textbf{69.1} & 65.7 \\
\hline
\end{tabular}
} 
\end{center}

\caption{With and without deformation on ShapeNet PartSeg dataset in terms of mIoU(\%).}
\label{tab:ab_deform}
\end{table}

\subsection{Classification of point-clouds}

In this section, we give some experiments on point-clouds for classification task. We have performed classification on three datasets, namely, \begin{inparaenum}  \item classifying digits in MNIST data \item classifying objects in ModelNet40 data \item classifying demented vs. non-demented subjects based on shape of corpus callosum.\end{inparaenum}

{\bf Classifying MNIST digits and Modelnet40 shapes:} We use MNIST digits \cite{deng2012mnist} and Modelnet40 shapes \cite{wu20153d} for classification. We follow the pre-processing step as given in \cite{li2018pointcnn}. We compare performance of POIRot with PointCNN \cite{li2018pointcnn} in terms of model complexity and classification accuracy. We perform two sets of experiments \begin{inparaenum}  \item training and testing both on non-rotated datasets \item training and testing on non-rotated and rotated datasets respectively. \end{inparaenum} For all the experiments, we use $256$ sampled points as the point-cloud. We generate the rotated data by randomly drawing rotation matrices from uniform distribution. Using only $<4\%$ parameters of PointCNN, we can achieve significantly better performance for rotated testing data as given in Table \ref{tab:synth_res}. We like to point out that, even using explicit rotated data augmentation, PointCNN can achieve $48.7\%$ on MNIST and $18.4\%$ classification accuracy on Modelnet40. This clearly indicates the implicit data augmentation is more powerful than the explicit data augmentation to achieve invariance. In Fig. \ref{fig:mnist_vis}, we show representative filter responses for digits `2' and `3' for the selected points highlighted as orange. 

\begin{table}[!ht]
\begin{center}

\resizebox{0.45\textwidth}{!}{

\begin{tabular}{>{\columncolor[gray]{0.95}}l|c|ccc} 
\hline
\rowcolor{LightCyan}

                                       & \color{black}{\# params}          & \color{black}{NR/NR (\%)} & \color{black}{NR/R (\%)}        \\ 
\hline
MNIST  & \textbf{9700} (251770) & 91.1 (\textbf{94.9})& \textbf{86.3} (24.0) \\
Modelnet40  & \textbf{6847} (599340) & 82.8 (\textbf{83.2})& \textbf{70.7} (7.3) \\
\hline
\end{tabular}
} 
\end{center}

\caption{Classification results in terms of accuracy for MNIST and Modelnet40 dataset. The results are shown for POIRot (PointCNN), `NR/X' denotes non-rotated training and `X' testing, where `X' may be rotated or non-rotated. POIRot approximately retains the accuracy for rotated testing with $\approx 4\%$ of the number of parameters as PointCNN model. For both these models, we use $256$ sampled points from point-cloud.}
\label{tab:synth_res}
\end{table}

\ \\
{\bf Classifying demented vs. non-demented subjects based on corpus callosum shapes:} In this section, we use OASIS data \cite{fotenos2005normative} to address the
classification of demented vs. non-demented subjects using
our proposed framework. This dataset contains at least two MR brain
scans of $150$ subjects, aged between $60$ to $96$ years old. For each
patient, scans are separated by at least one year. The dataset
contains patients of both sexes. In order to avoid gender effects, we take MR scans of male patients alone from three visits, which
results in the dataset containing $69$ MR scans of $11$ subjects with
dementia and $12$ subjects without dementia. This gives $33$ scans for subjects with dementia and $36$ scans for subjects without dementia. We first compute an atlas
(using the method in \cite{avants2009advanced}) from the $36 (=12
\times 3)$ MR scans of patients without dementia. After rigidly registering each MR scans to the atlas, we segment out the corpus callosum region from each scan. We represent the shape of the corpus callosum as a 3D point-cloud.  

{\bf Computing the centroid of a point-cloud:} Given the point-cloud $X = \left\{\mathbf{x}_i\right\}_{i=1}^N\subset \mathbf{R}^3$, we compute the centroid of the point-cloud to be the nearest point to the mean of $X$. Formally, let us denote the centroid of $X$ to be $\mathbf{m}$. Then, $\mathbf{m}$ is defined as
\begin{align*}
\mathbf{m} = \Pi_{X}\left(\frac{1}{n} \sum_{i=1}^n \mathbf{x}_i\right), 
\end{align*}
where, $\Pi_{X}(\mathbf{x})$ is the projection of $\mathbf{x}$ in the set $X$. 

{\bf Extracting the ``attention'' from a point-cloud:} Given the point-cloud $X$, we extract the region of interest, i.e., ``attention'' to be a subset $Y \subset X$ as follows: \begin{inparaenum}[\bfseries (a)] \item Compute the directional part of the vector from $\mathbf{m}$ to each point, $\mathbf{x}_i$. Let the vector be denoted by $\mathbf{v}_i$. \item Pass the vector $\mathbf{v}_i$ through a FC layer to get the confidence, $c_i \in [0, \infty)$, for selecting $\mathbf{x}_i$. \item Define a random variable following multinomial distribution with $\mathbf{c}$ as the parameter. \item Draw samples from this random variable to generate $Y$. \end{inparaenum}We call this subset $Y$ to be our region of interest. We follow the steps as described in Section \ref{sec:classmodel}. 



We achieve $\mathbf{90.72\pm 0.79}\%$ classification accuracy with the sensitivity and specificity to be $87.88\%$ and $94.44\%$ respectively. If we remove the ``attention'' module, we can achieve $72.46\%$ classification accuracy. This clearly indicates the usefulness of the ``attention'' module used in this work. We show the overlayed attention region and the selected convex hull points in Fig. \ref{fig:cc_hull}. We can see that the attention block focuses on the thinning of corpus callosum shape in order to classify demented vs. non-demented subjects. In Fig. \ref{fig:cc_rot}, we show that for rotated and non-rotated CC shapes the integrated responses are similar which proves the rotational invariance property. 

\begin{figure}[!ht]
 \centering
\vspace*{-1em}
\begin{tikzpicture}
\node[inner sep=0pt] (fig1) at (0,0)
    {\includegraphics[width=0.12\textwidth, height=2.3cm]{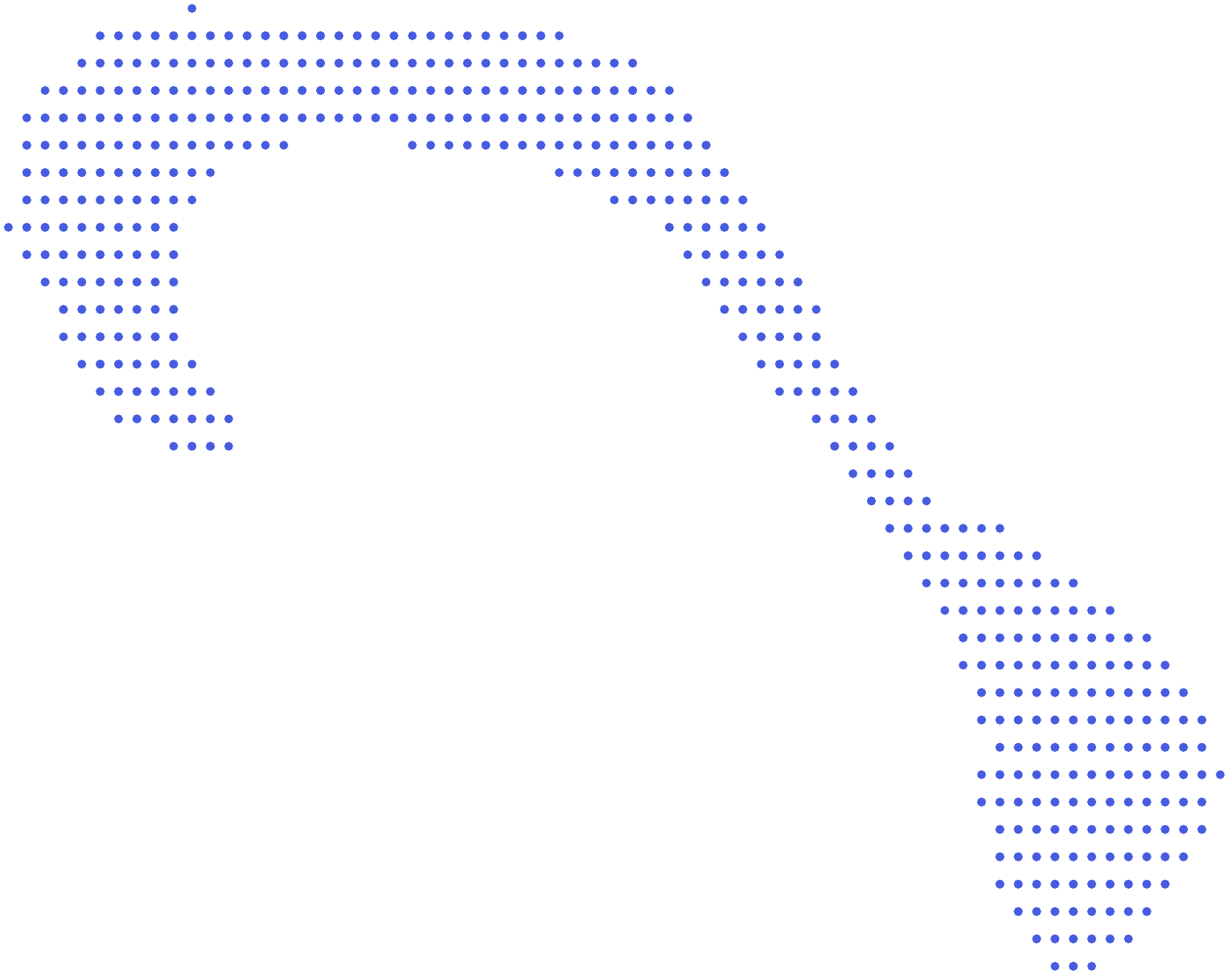}};
\node[inner sep=0pt] (fig2) at (4,0)
    {\includegraphics[width=0.12\textwidth, height=2.3cm]{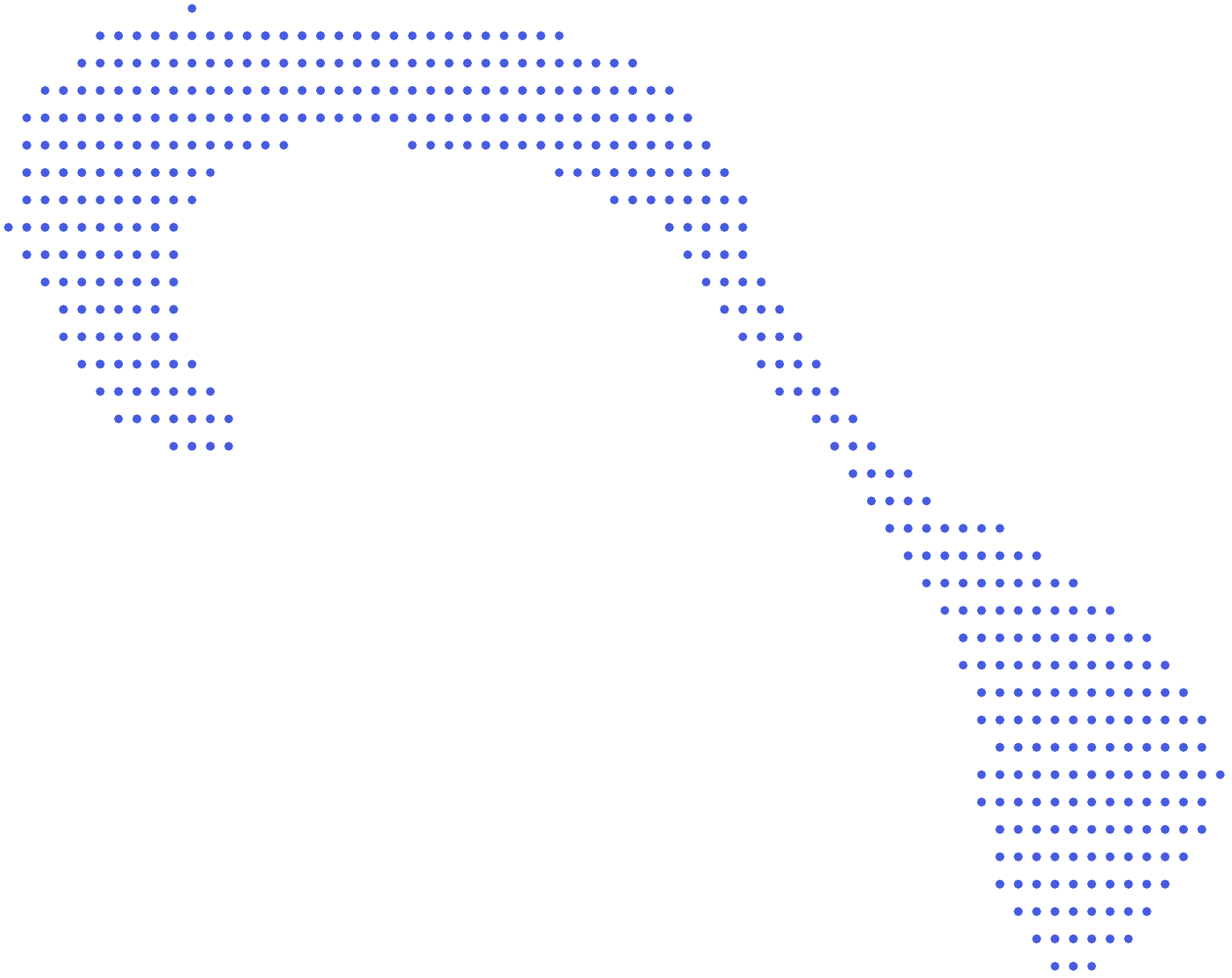}};
\end{tikzpicture}
\\
\begin{tikzpicture}
\node[inner sep=0pt] (fig1) at (0,0)
    {\includegraphics[width=0.19\textwidth, height=2.9cm]{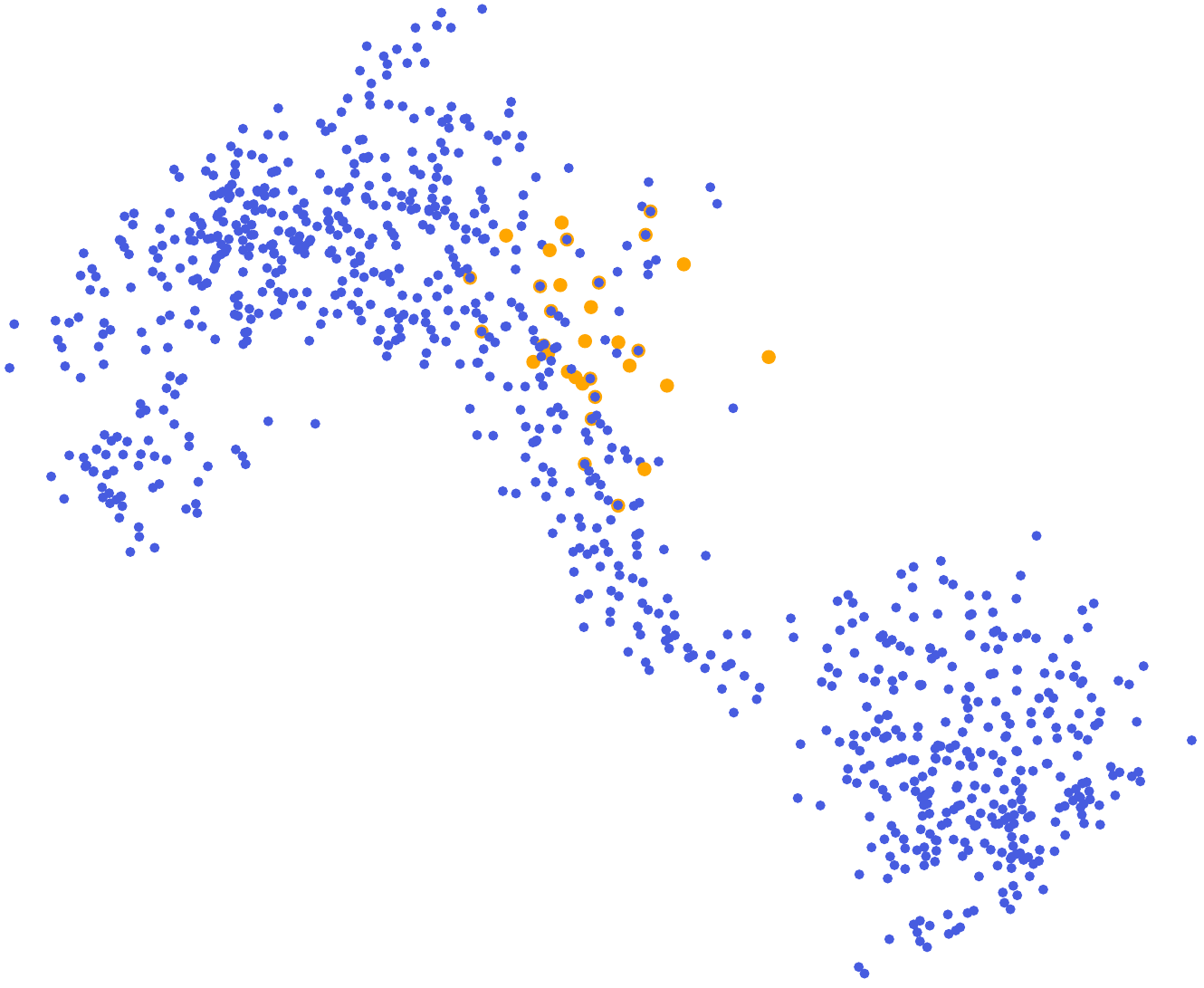}};
\node[inner sep=0pt] (fig2) at (4,0)
    {\includegraphics[width=0.17\textwidth, height=2.7cm]{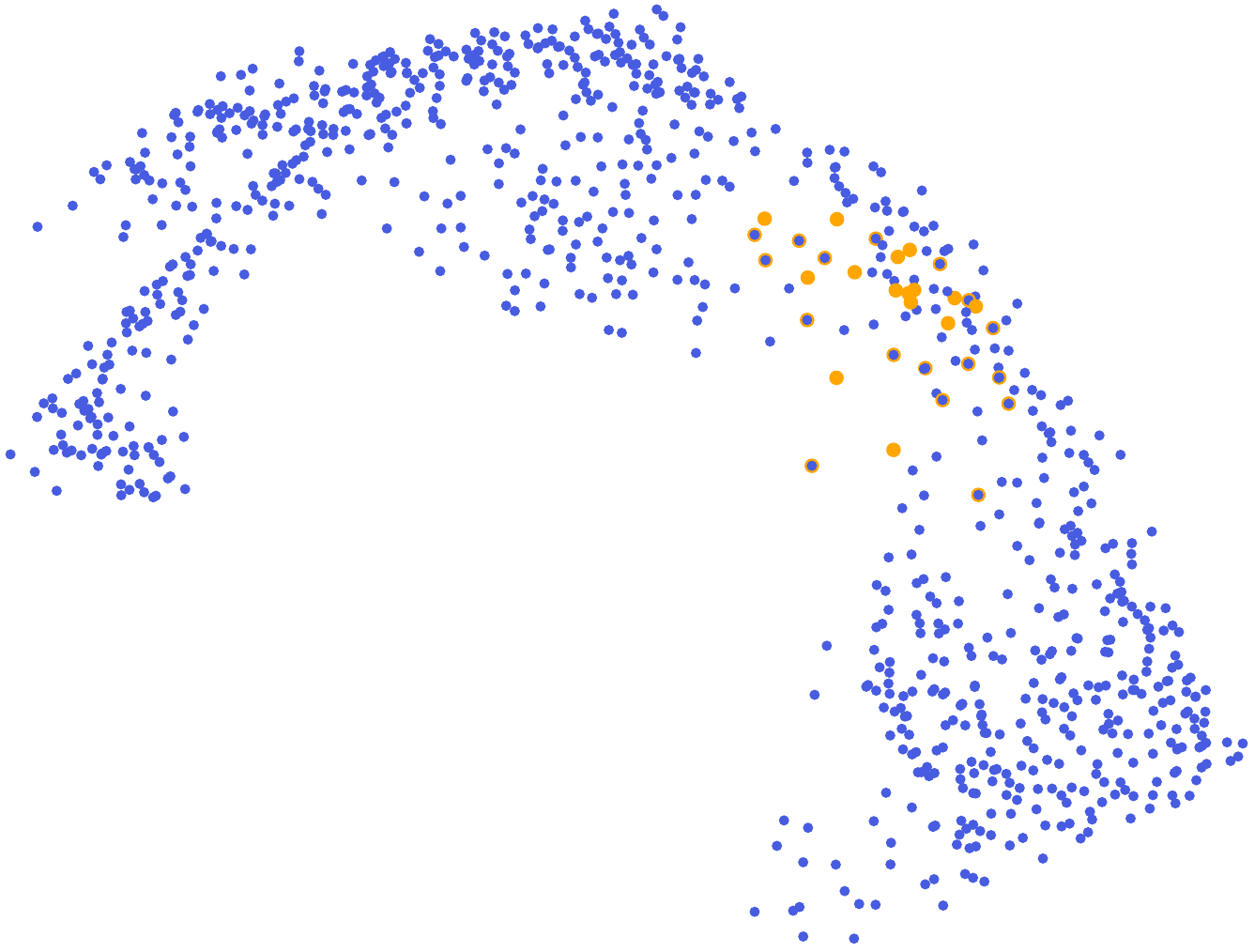}};
\end{tikzpicture}
\\
\begin{tikzpicture}
\node[inner sep=0pt] (fig1) at (0,0)
    {\includegraphics[width=0.17\textwidth, height=2.7cm]{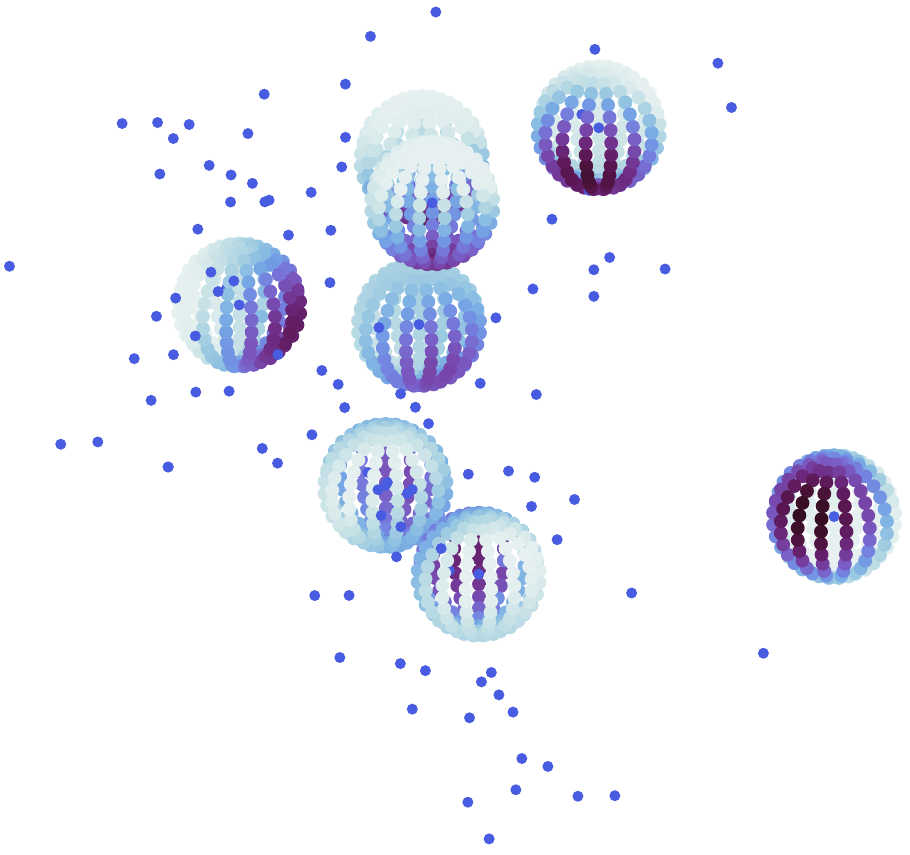}};
\node[inner sep=0pt] (fig2) at (4,0)
    {\includegraphics[width=0.12\textwidth, height=2.3cm]{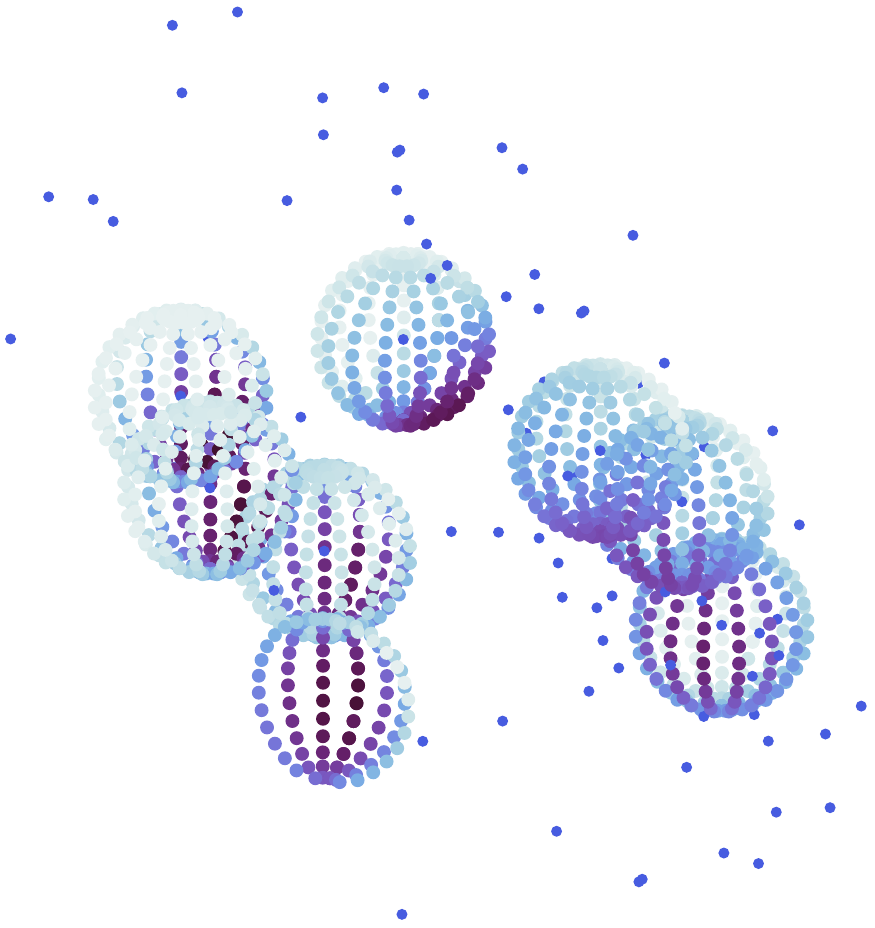}};
\end{tikzpicture}
\vspace*{-1em}
\caption{{\it (Top):} Sample point-cloud {\it (Middle:)} with attention region marked with ``orange'' {\it (Bottom:)} putting sphere around the convex hull points. (The first and second columns represent samples with non-dementia and dementia respectively.)}
      \label{fig:cc_hull}
\end{figure}

\begin{figure}[!ht]
 \centering
\begin{tikzpicture}
\node[inner sep=0pt] (fig1) at (0,0)
    {\includegraphics[width=0.12\textwidth, height=2.3cm]{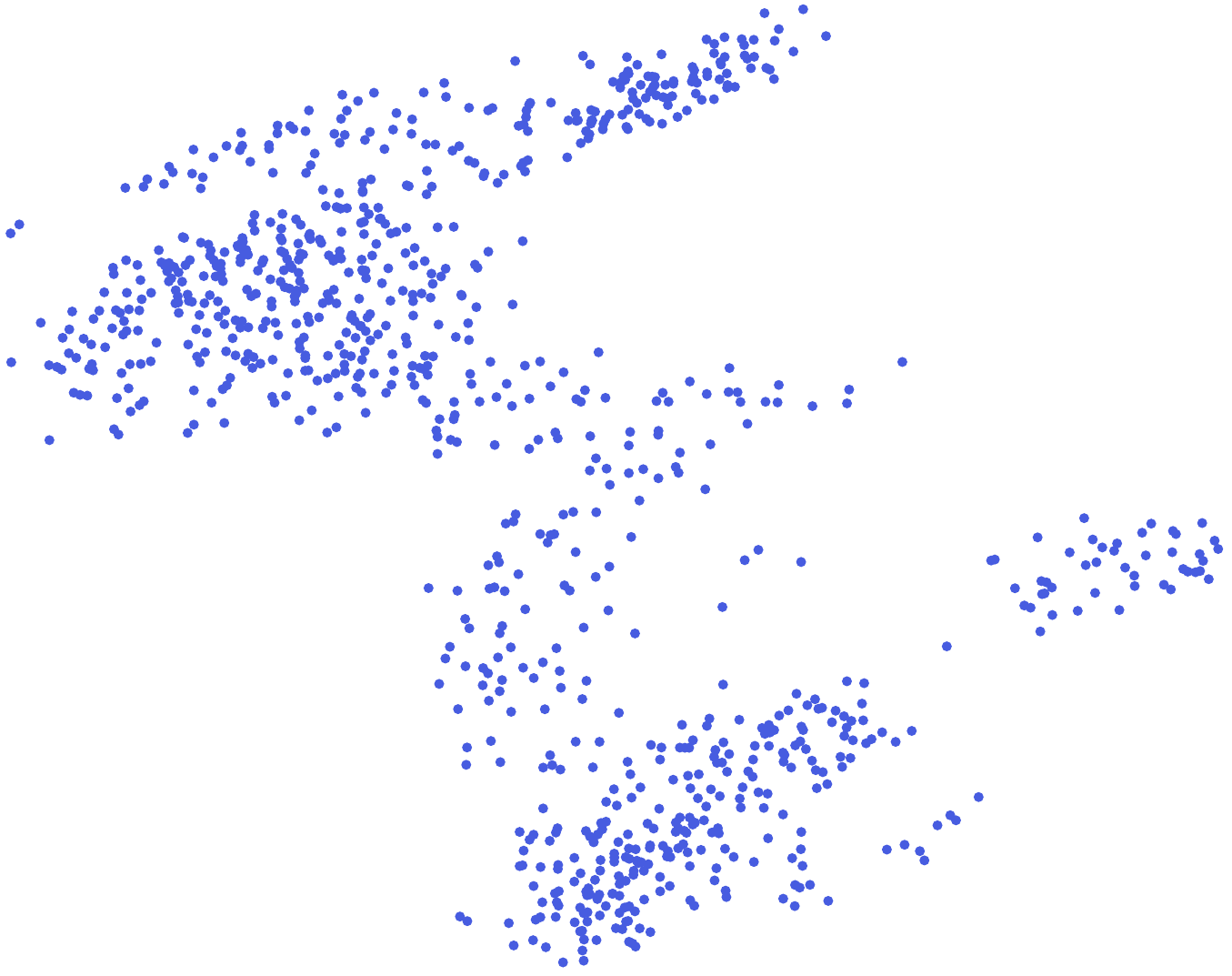}};
\node[inner sep=0pt] (fig2) at (6,0)
    {\includegraphics[width=0.12\textwidth, height=2.3cm]{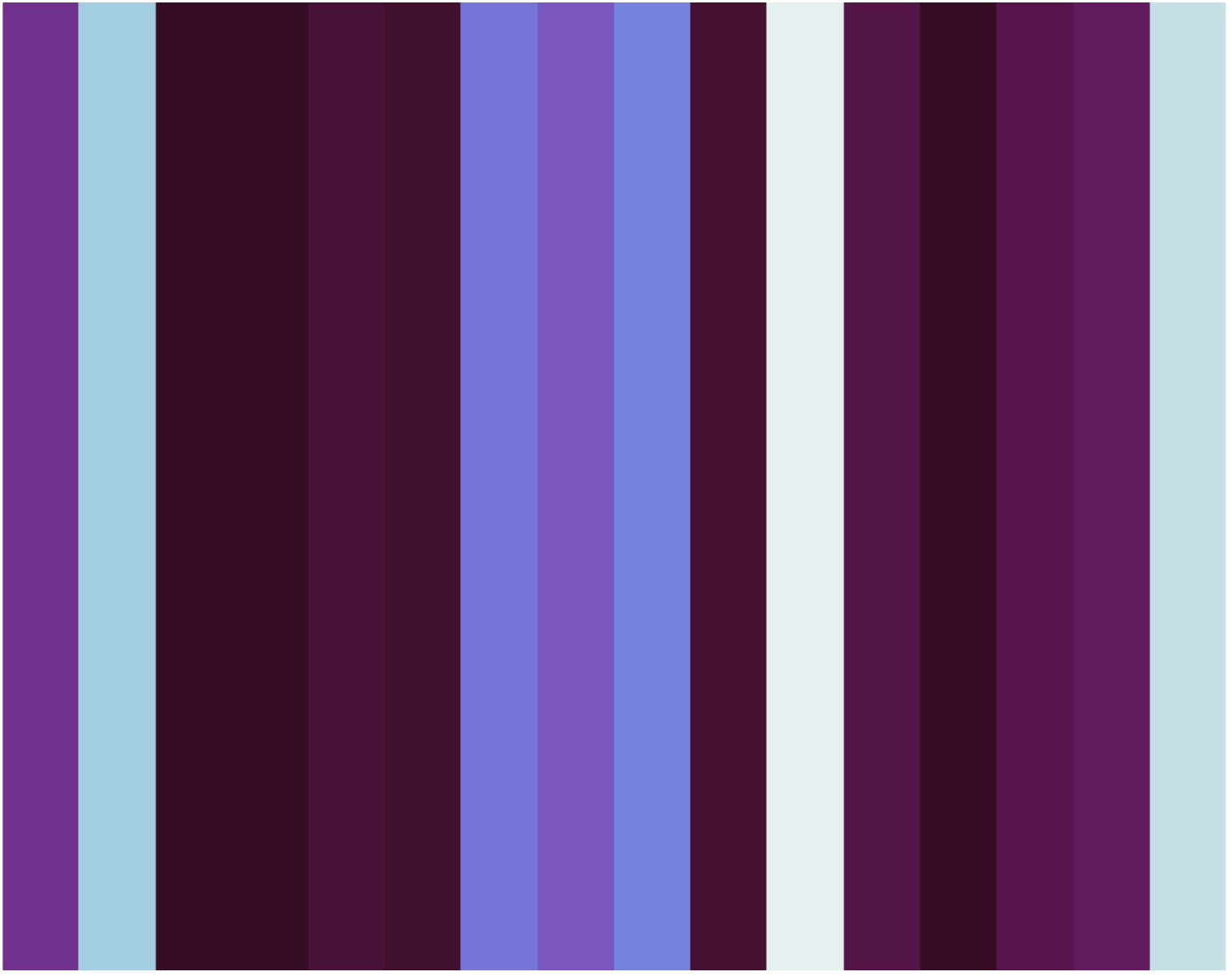}};
    \draw[->, line width=0.2mm] (2.3,0) -- (3.9, 0);
\end{tikzpicture}
\\
\begin{tikzpicture}
\node[inner sep=0pt] (fig1) at (0,0)
    {\includegraphics[width=0.12\textwidth, height=2.3cm]{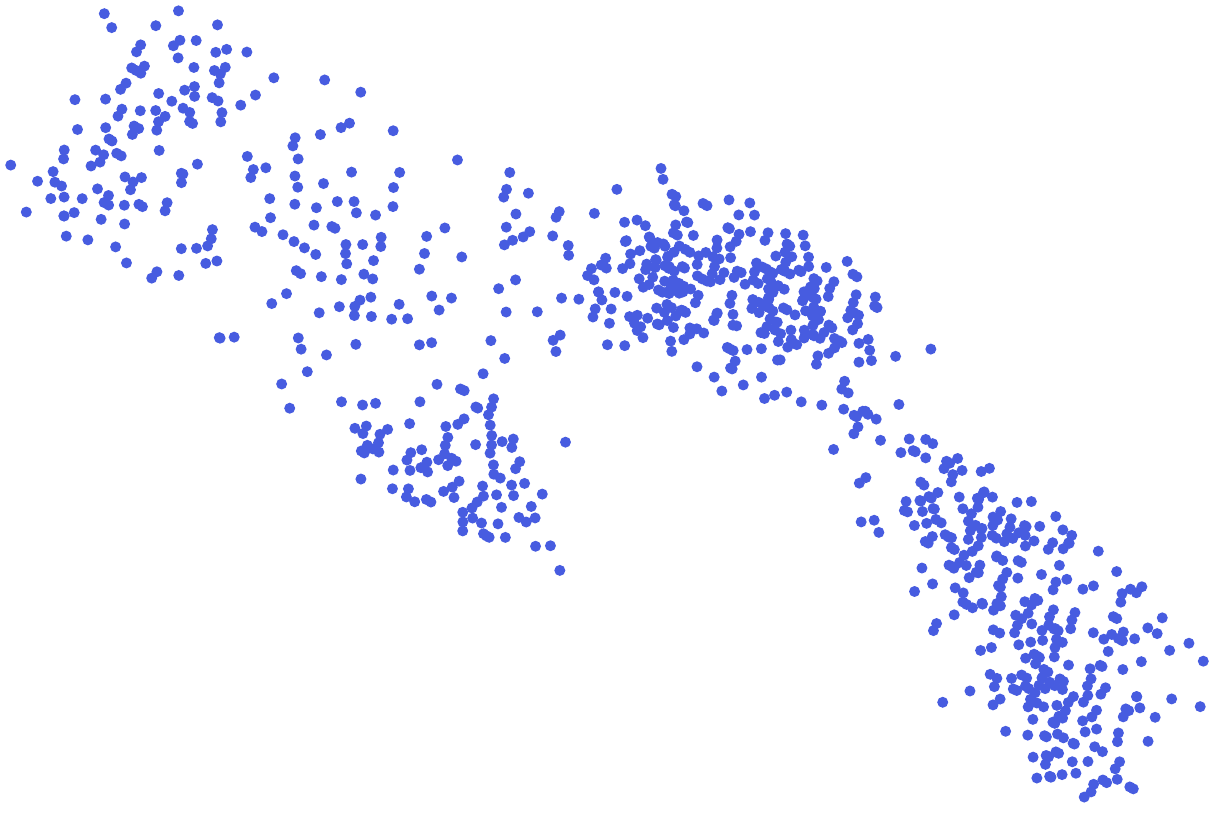}};
\node[inner sep=0pt] (fig2) at (6,0)
    {\includegraphics[width=0.12\textwidth, height=2.3cm]{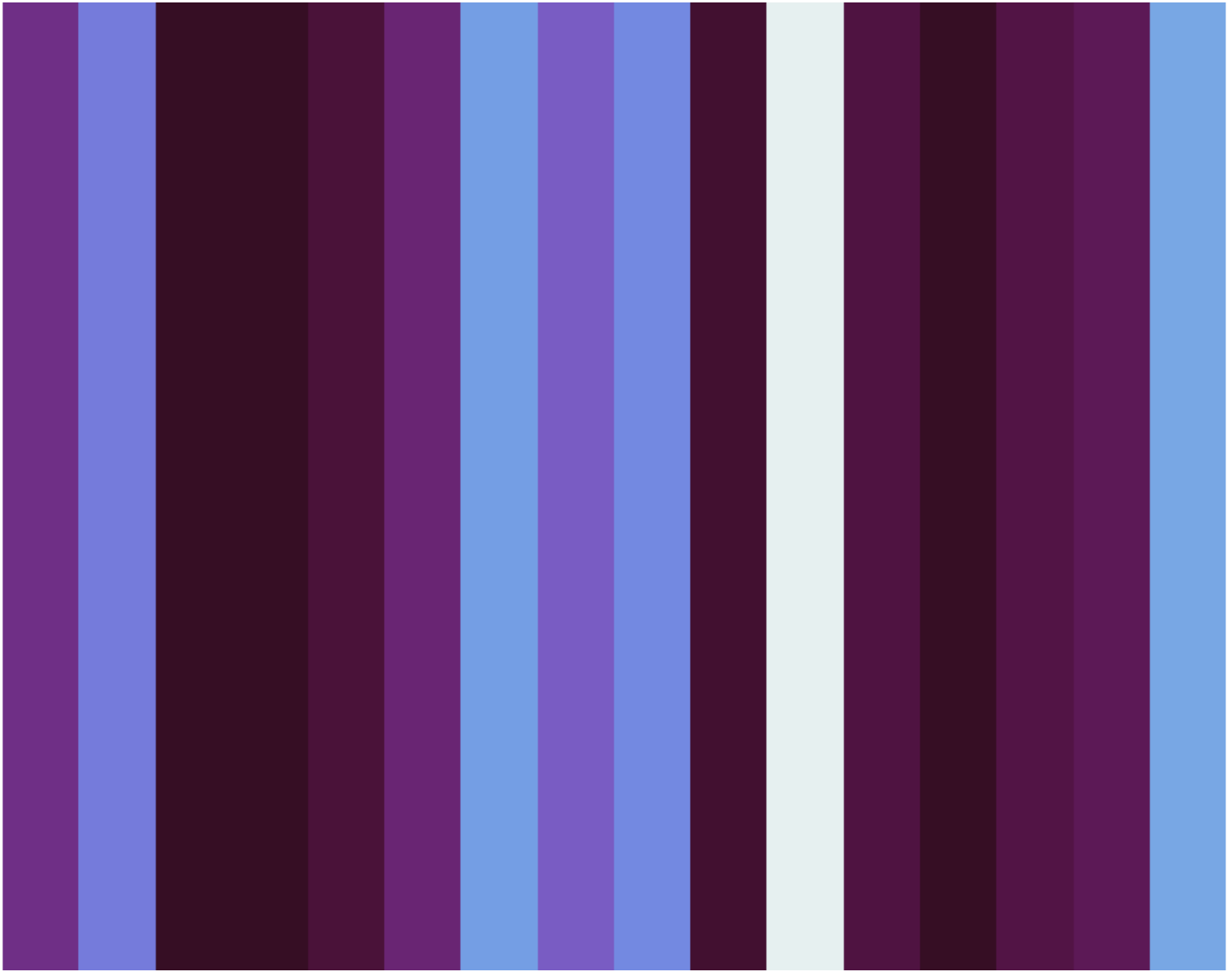}};
    \draw[->, line width=0.2mm] (2.3,0) -- (3.9, 0);
\end{tikzpicture}
\caption{{\it (Top):} non-rotated {\it (Bottom:)} rotated point-cloud with their respective invariant feature responses.}
      \label{fig:cc_rot}
\end{figure}

\begin{figure*}[!ht]
    \centering
   \includegraphics[scale=0.35]{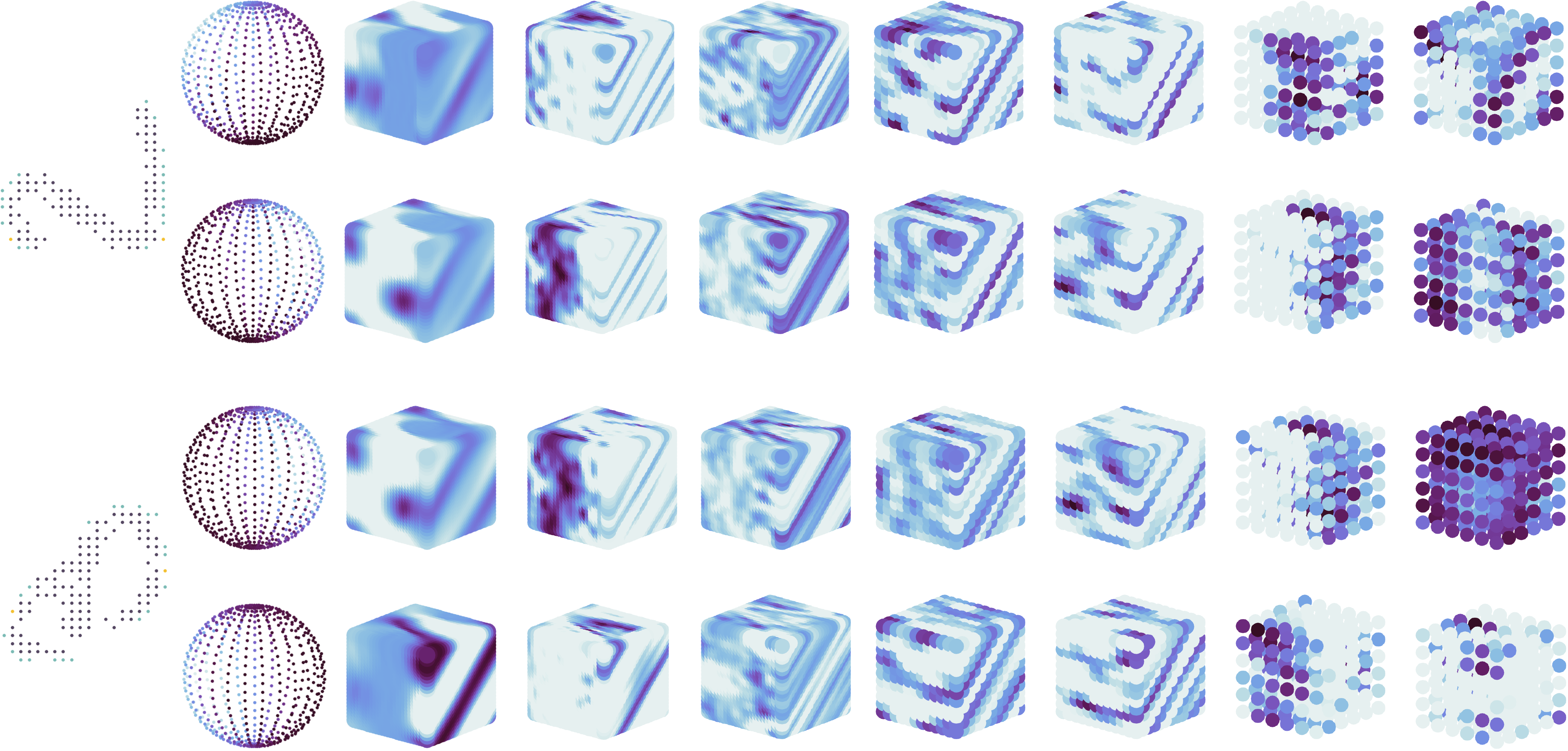}
\caption{Filter responses for mnist digits $2$ and $3$.}
\label{fig:mnist_vis}
\end{figure*}

\subsection{Object detection}
In this section, we want to investigate whether the proposed method can learn a 3D geometric shape of an object of interest. We demonstrated it in an object detection framework where the task is to detect corpus callosum region from a 3D brain scan. One of the major hurdle to deal with medical images is to do image registration which incurs some error that can carry forward to the following processing steps. This motivates us to solve the detection problem in an unsupervised way where we only present an atlas of corpus callosum shape. The only way to detect the corpus callosum shape correctly during testing is by learning the 3D geometric shape of the corpus callosum region. Hence, in our object detection task, our inputs are a 3D brain scan and an atlas for CC region. Our proposed framework consists of five steps as given below.

{\bf Extracting point-cloud:} Given a 3D brain scan and an atlas for CC region, we extract the corresponding 3D point-clouds for the brain scan and the CC atlas, denoted by $X = \left\{\mathbf{x}_i\right\}_{i=1}^N\subset \mathbf{R}^3$ and $M = \left\{\mathbf{m}_i\right\}_{i=1}^m \subset \mathbf{R}^3$ respectively. 

{\bf Constructing affinity matrix using geodesic distance:} In order to capture the shape of the corpus callosum, one needs to use geodesic distance instead of standard $\ell_2$ distance. For each point $\mathbf{x} \in X$, we look at the nearest neighbors and calculate their $\ell_2$ distance. In this way, for each $\mathbf{x}$, we can construct a locally flat neighborhood, $\mathcal{N}_{\epsilon}(\mathbf{x})$, which essentially gives us an adjacency matrix, $E$. Now we construct a graph $G= (X, E)$ where each point is treated as node of the graph and there exists an edge between vertex $\mathbf{x}_i$ and $\mathbf{x}_j$ if $\mathbf{x}_i \in \mathcal{N}_{\epsilon}(\mathbf{x}_j)$ or $\mathbf{x}_j \in \mathcal{N}_{\epsilon}(\mathbf{x}_i)$. Given the graph $G$, we use Floyd-Warshall algorithm to find all pair shortest paths and use the weighted length of the path as the geodesic distance. Analogous idea to get geodesic distance has been used in \cite{vemurimultiatlas}. Using the geodesic distance, $d_g: X \times X \rightarrow \mathbf{R}$, we construct the affinity matrix $A=\left\{a_{ij}\right\}$ with $a_{ij} = d_g(\mathbf{x}_i, \mathbf{x}_j)$. We choose $\epsilon = 5$ in terms of pixels for our purpose. 

{\bf Candidate selection:} After construction of the affinity matrix $A$, we choose a potential candidate for the matching CC shape as follows. Let the candidate pool be denoted by $\mathcal{S} = \left\{S_j\right\}$ where $S_j$ is a {\it potential match}. For notational simplicity, we denote   each {\it potential match} $S$ by $\mathbf{x}$ and $m$. Thus for a given $m$, we call $\mathbf{x}_j$ to be a {\it potential match} if $S_j$ contains $m-1$ nearest neighbors of $\mathbf{x}_j$. At each $\mathbf{x}_j \in S$, we construct a sphere around and collect responses on the sphere from the $m-1$ nearest neighbors (in an analogous way as in Section \ref{sec:classmodel}). Let the spherical response be denoted by $\left\{f_j: \mathbf{S}^2 \rightarrow \mathbf{R}\right\}_{\mathbf{x}_j \in S}$. We capture responses from the atlas as well and let it be denoted by $m_f: \mathbf{S}^2\rightarrow \mathbf{R}$.

{\bf Spherical feature extractor and feature matching:} We use spherical correlations as described in Section \ref{sec:classmodel} from both $\left\{f_j\right\}$ and $m_f$ to get the rotation invariant features, let it be denoted by $\left\{\mathfrak{F}_j \in \mathbf{R}^f \right\}$ and $\mathfrak{M} \in \mathbf{R}^f$, where $f$ is the dimension of the spherical feature.  Now, for each $\mathfrak{F}_j$, we compute the similarity with $\mathfrak{M}$ to be $\mathfrak{F}_j^t\mathfrak{M}$. Thus, for each {\it potential match}, we get a similarity score, let it be denoted by $s_j$.

{\bf Minimizing the entropy:} After getting the similarity scores $\left\{s_j\right\}$, we compute the probability of $S_j$ to be the chosen segmented region (denoted by $p_j$) as:
\begin{align*}
P(S_j = M) = \frac{exp(s_j)}{\sum_{i=1}^{|S|} exp(s_i)}
\end{align*}
Finally we minimize the entropy given by $-\sum_{j=1}^{|S|} p_j \log(p_j)$. We select the candidate as the one with maximum probability, i.e., the chosen candidate $j^* = \displaystyle\argmax_j p_j$. A schematic of our proposed unsupervised algorithm is given in Fig. \ref{fig:cc_seg_flow}. The detection results are shown in Fig. \ref{fig:seg_medical} which indicates that the proposed method performs well to detect CC shapes.

\begin{figure}[!ht]
    \centering
    \includegraphics[width=0.15\textwidth]{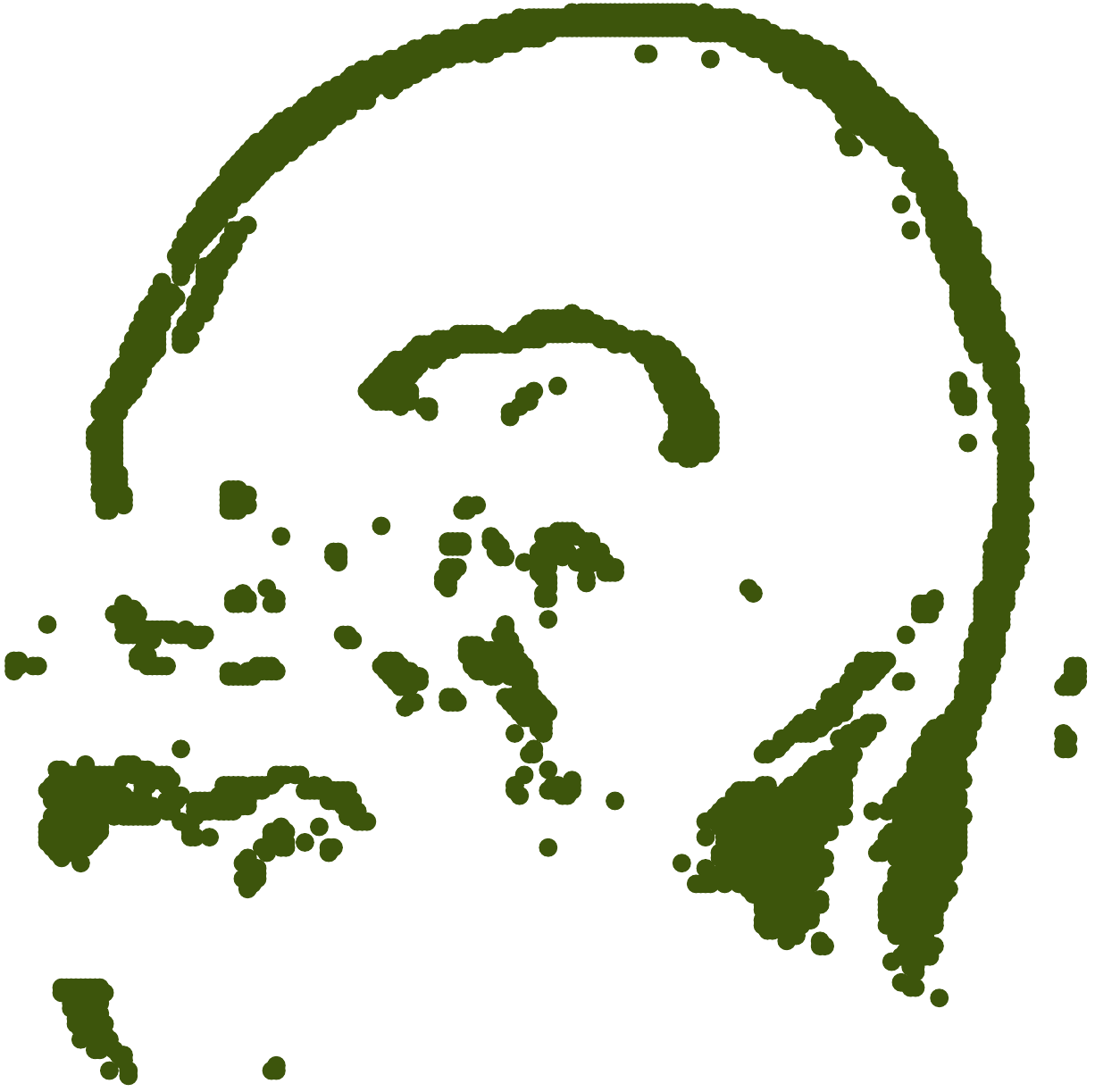}
    \includegraphics[width=0.15\textwidth]{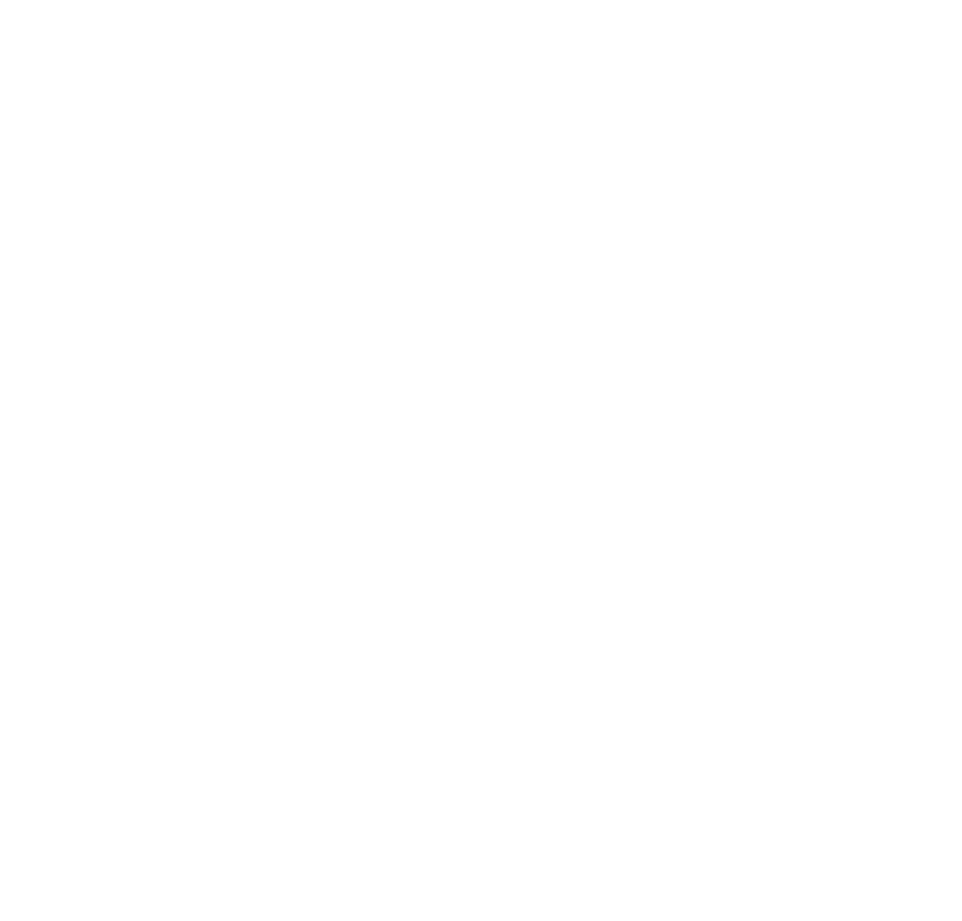}
    \includegraphics[width=0.15\textwidth]{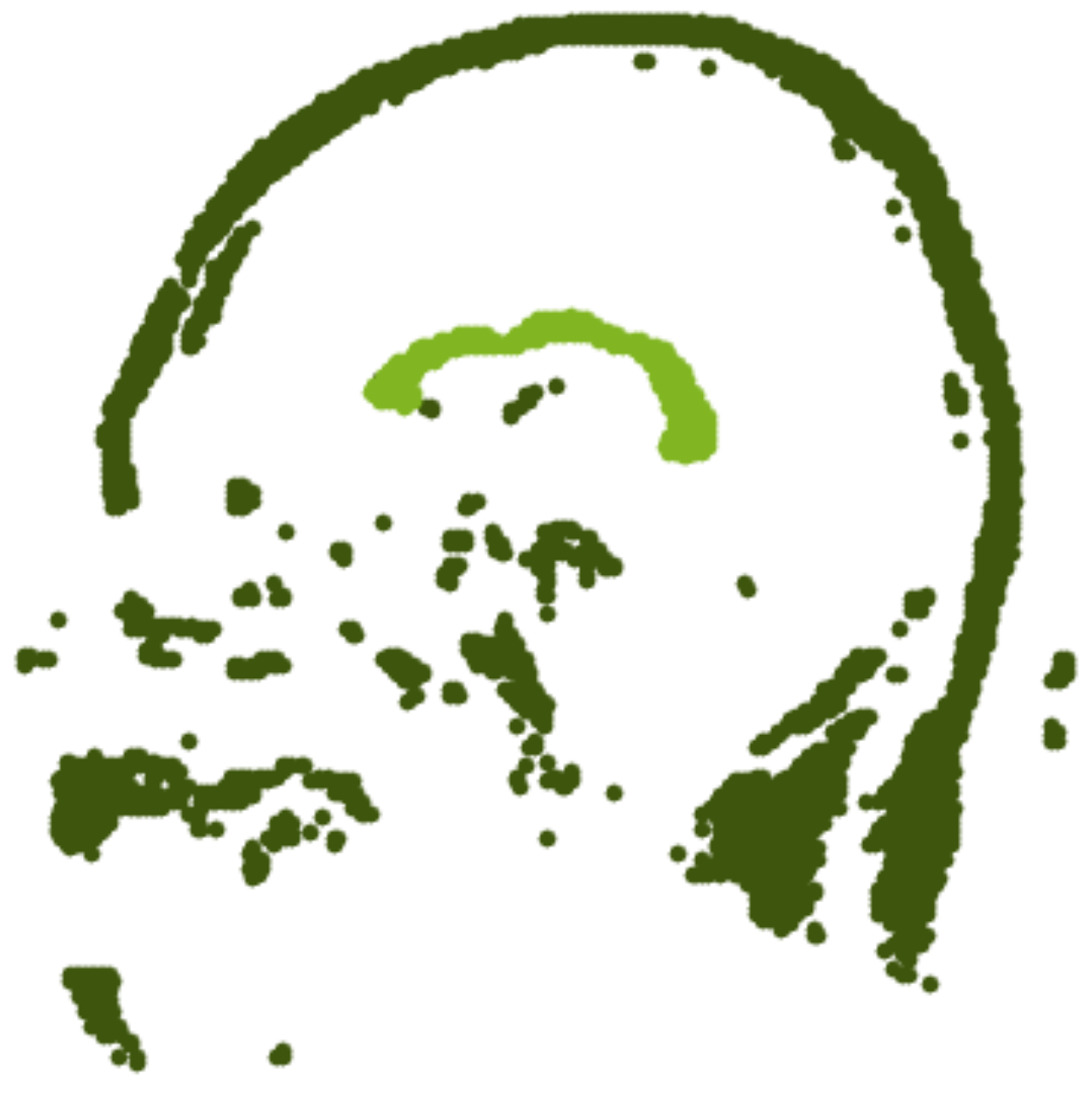}
    \\
    ~~~~~~
    \includegraphics[width=0.15\textwidth]{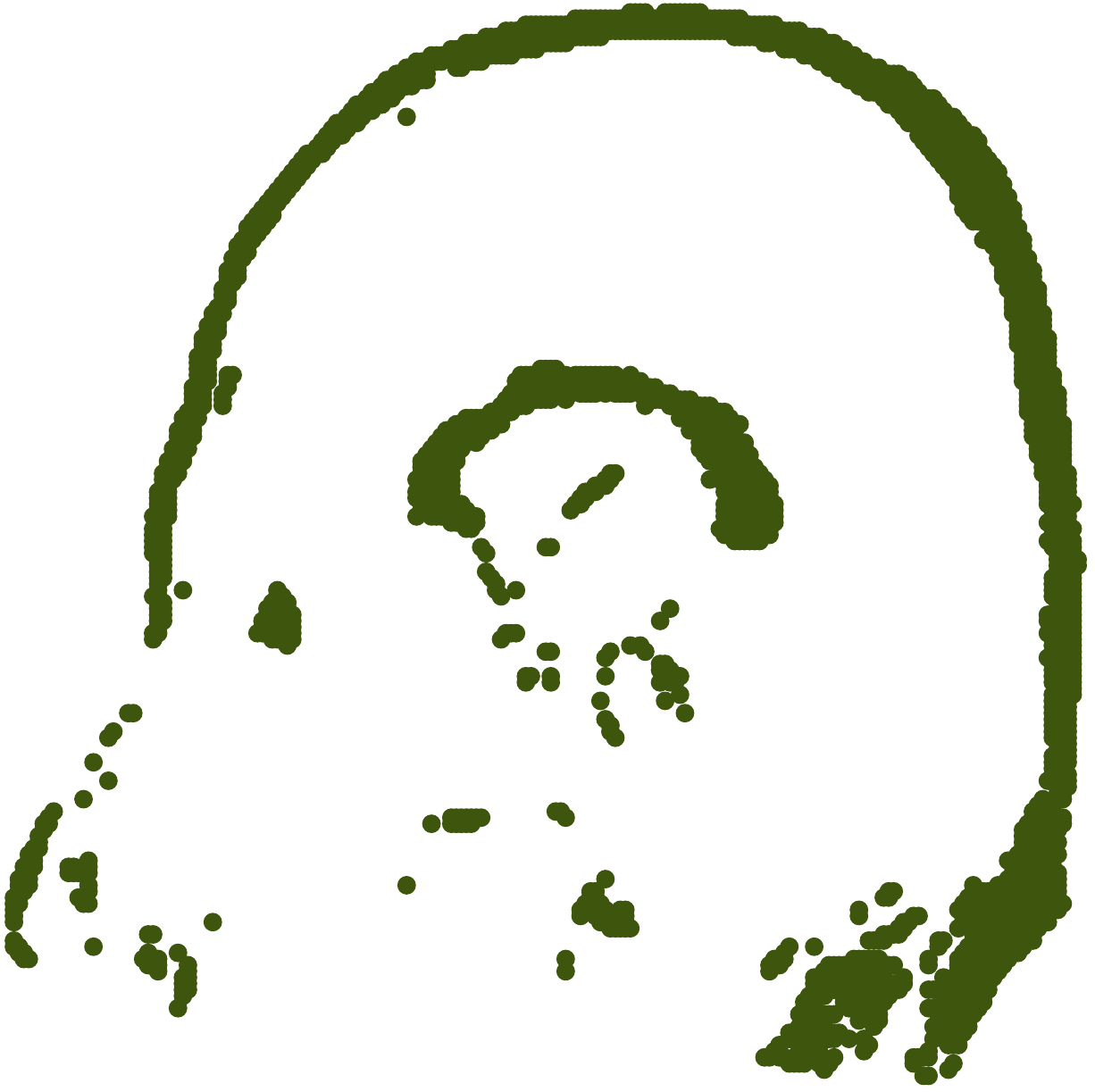}
    \includegraphics[width=0.15\textwidth]{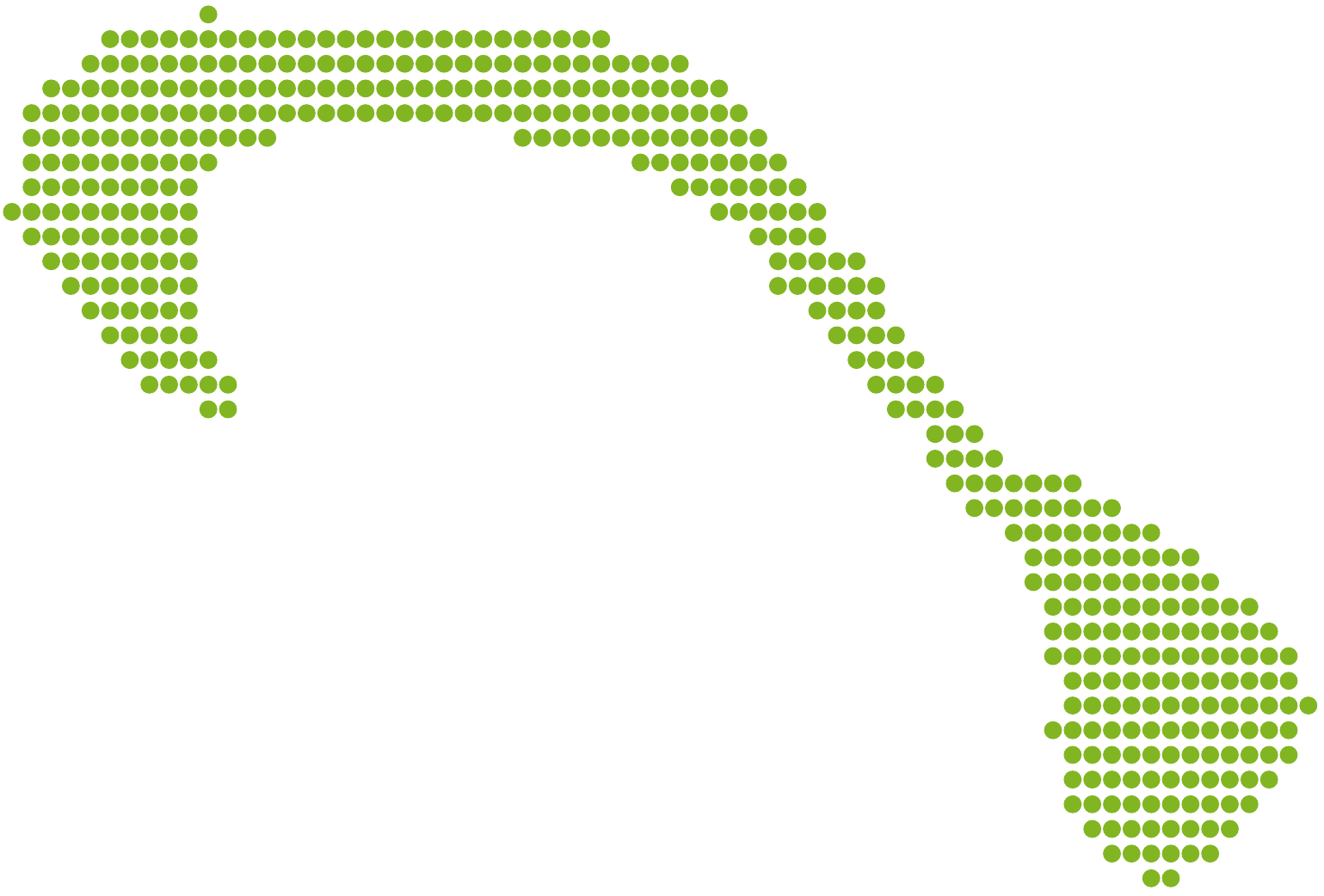}
    \includegraphics[width=0.15\textwidth]{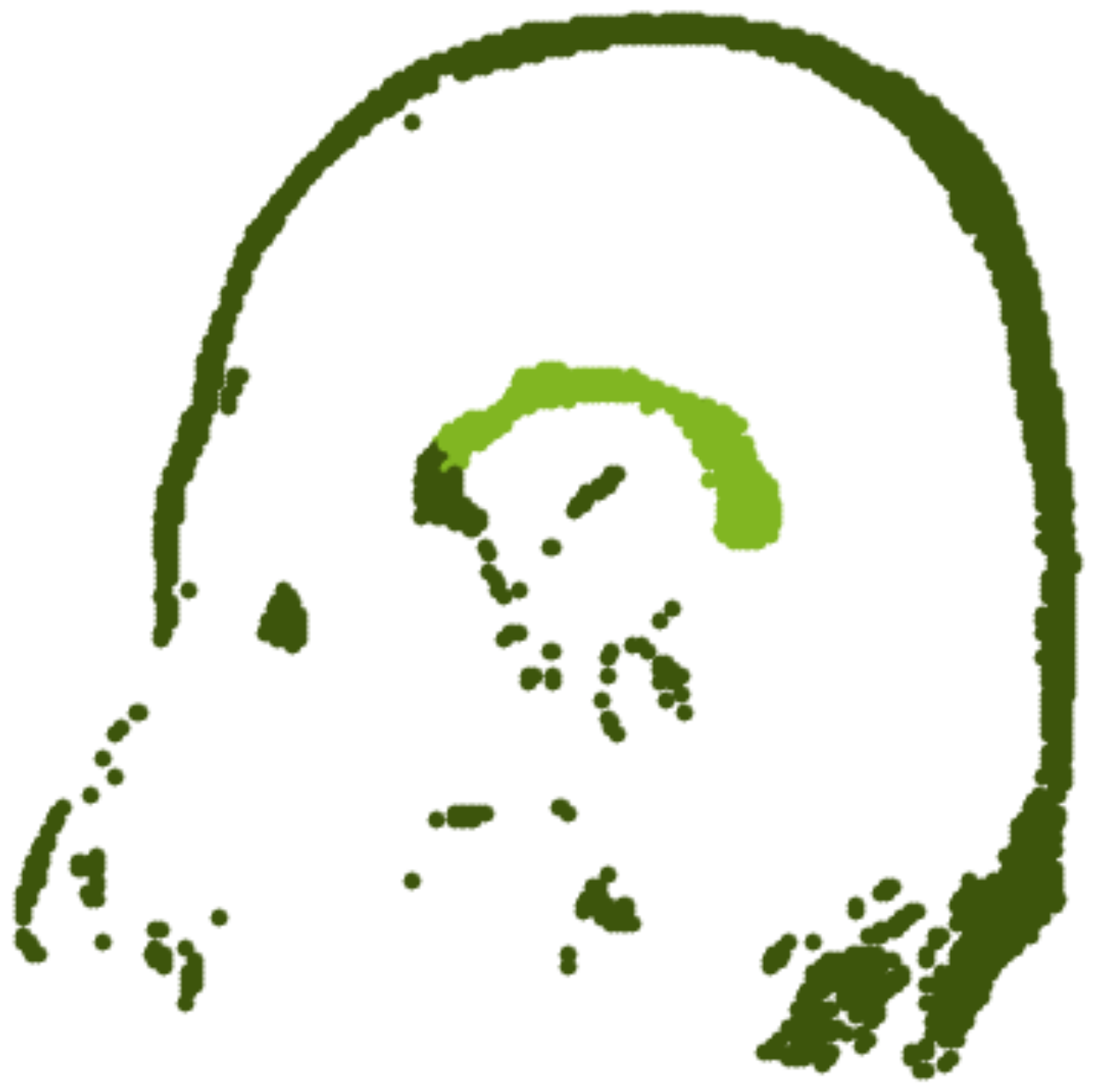}
    \\
     ~~~~~~
    \includegraphics[width=0.15\textwidth]{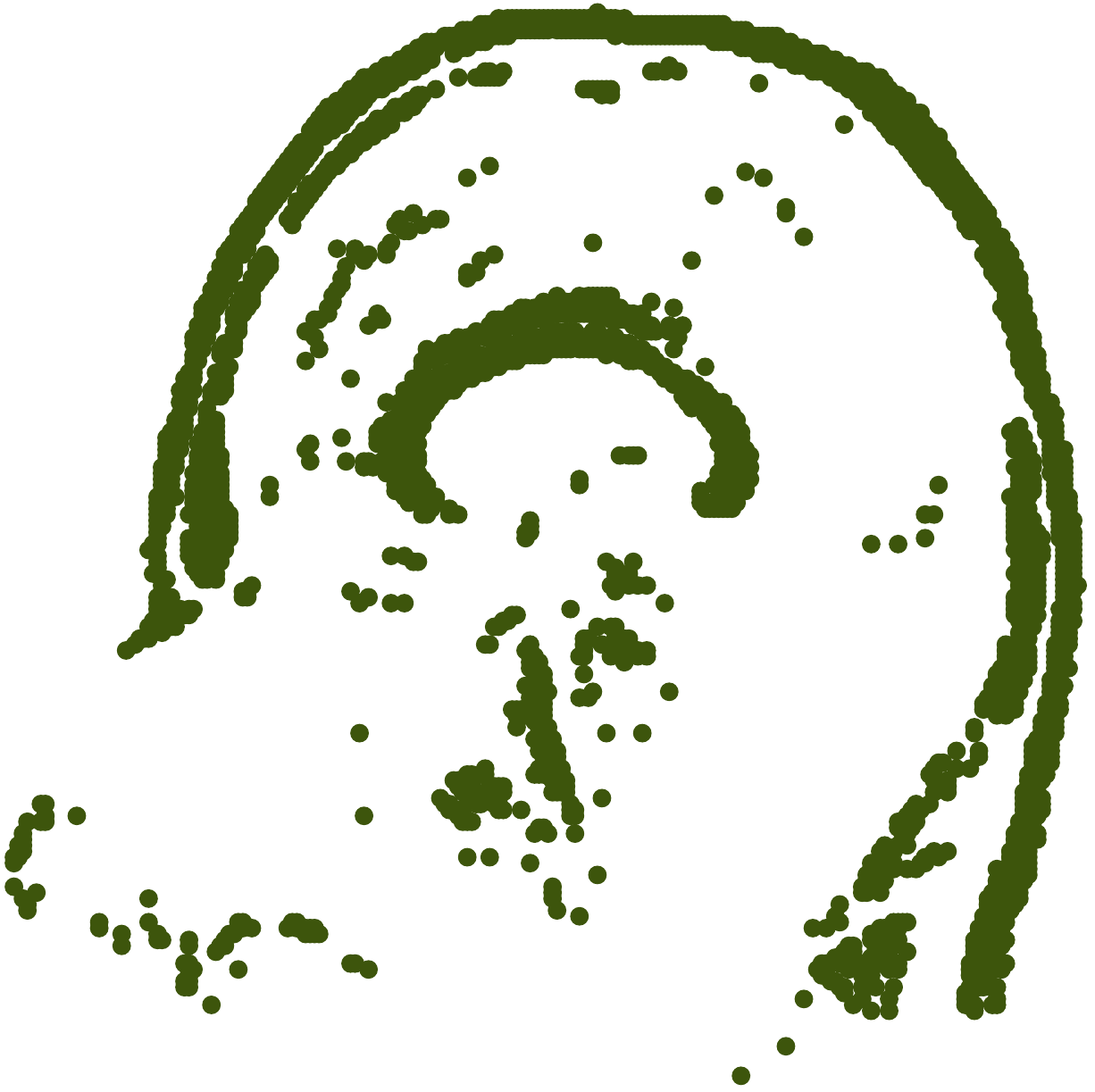}
    \includegraphics[width=0.15\textwidth]{blank}
    \includegraphics[width=0.15\textwidth]{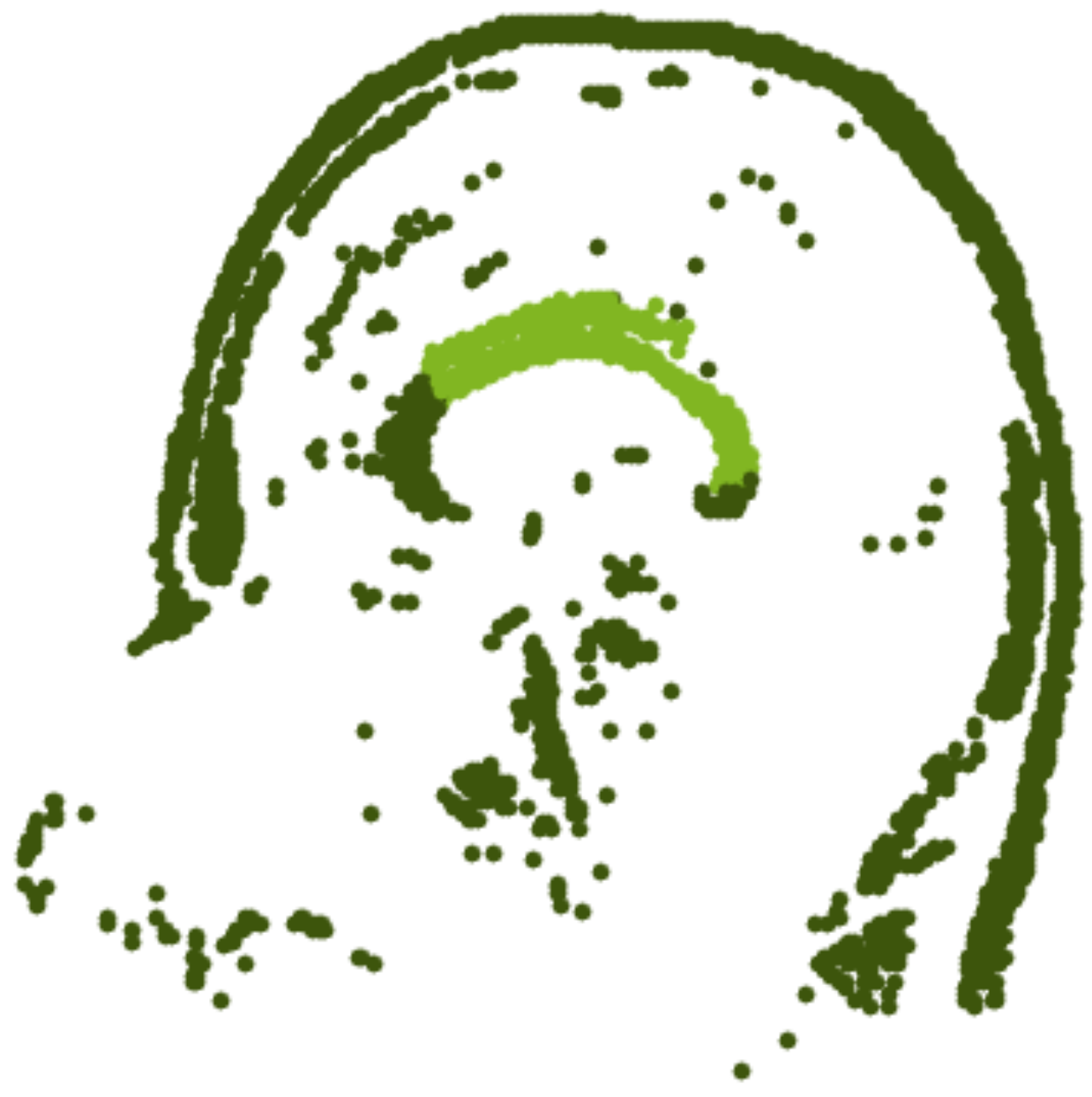}
\caption{point-cloud from 3D brain scans ({\it Left}) with the corresponding segmented corpus callosum  region ({\it Right}). The middle subplot shows an atlas of corpus callosum shape.}
\label{fig:seg_medical}
\end{figure}

\begin{figure*}[!ht]
    \centering
    \subfigure[airplane]{\includegraphics[scale=0.16]{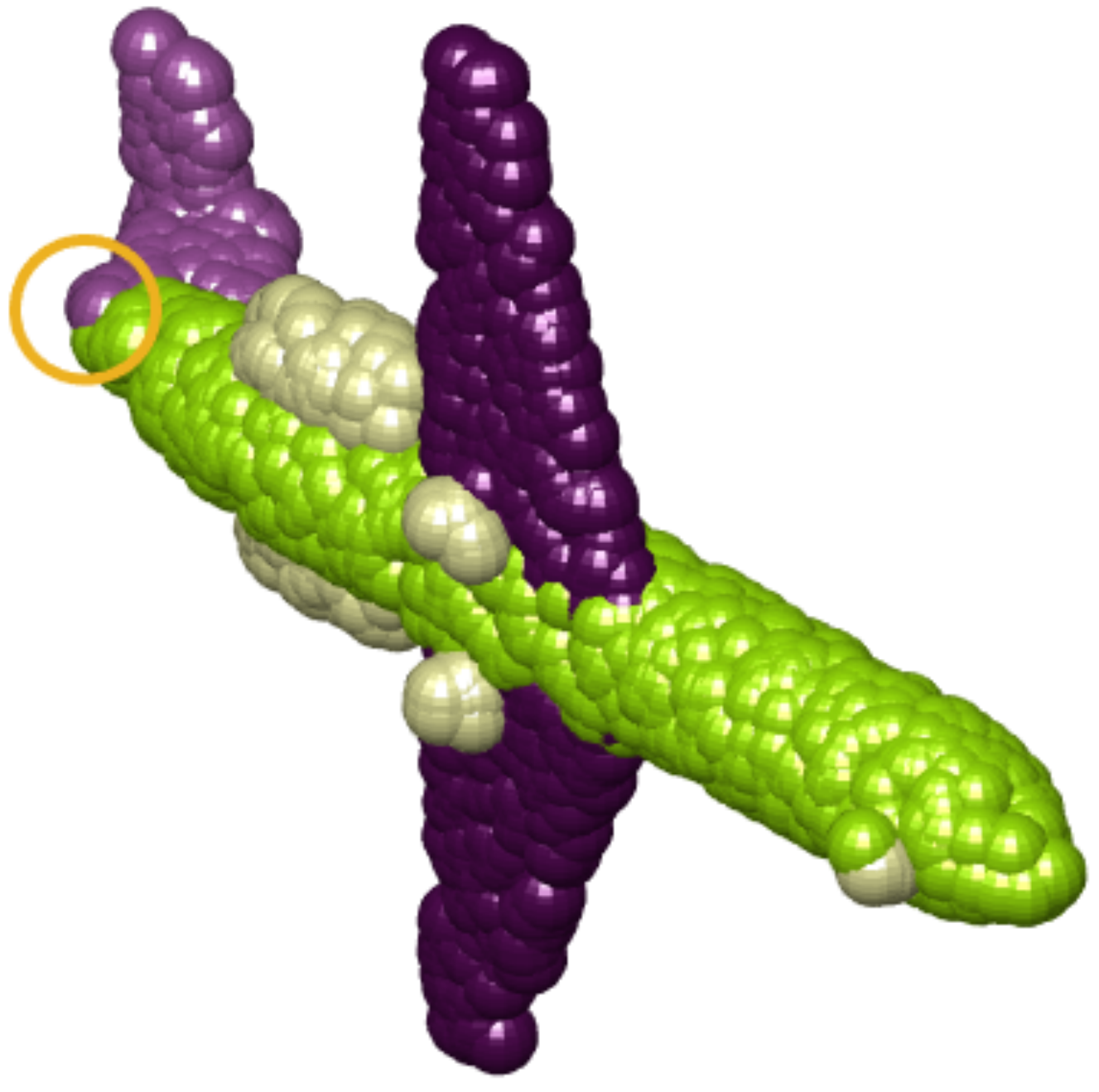}
    \includegraphics[scale=0.16]{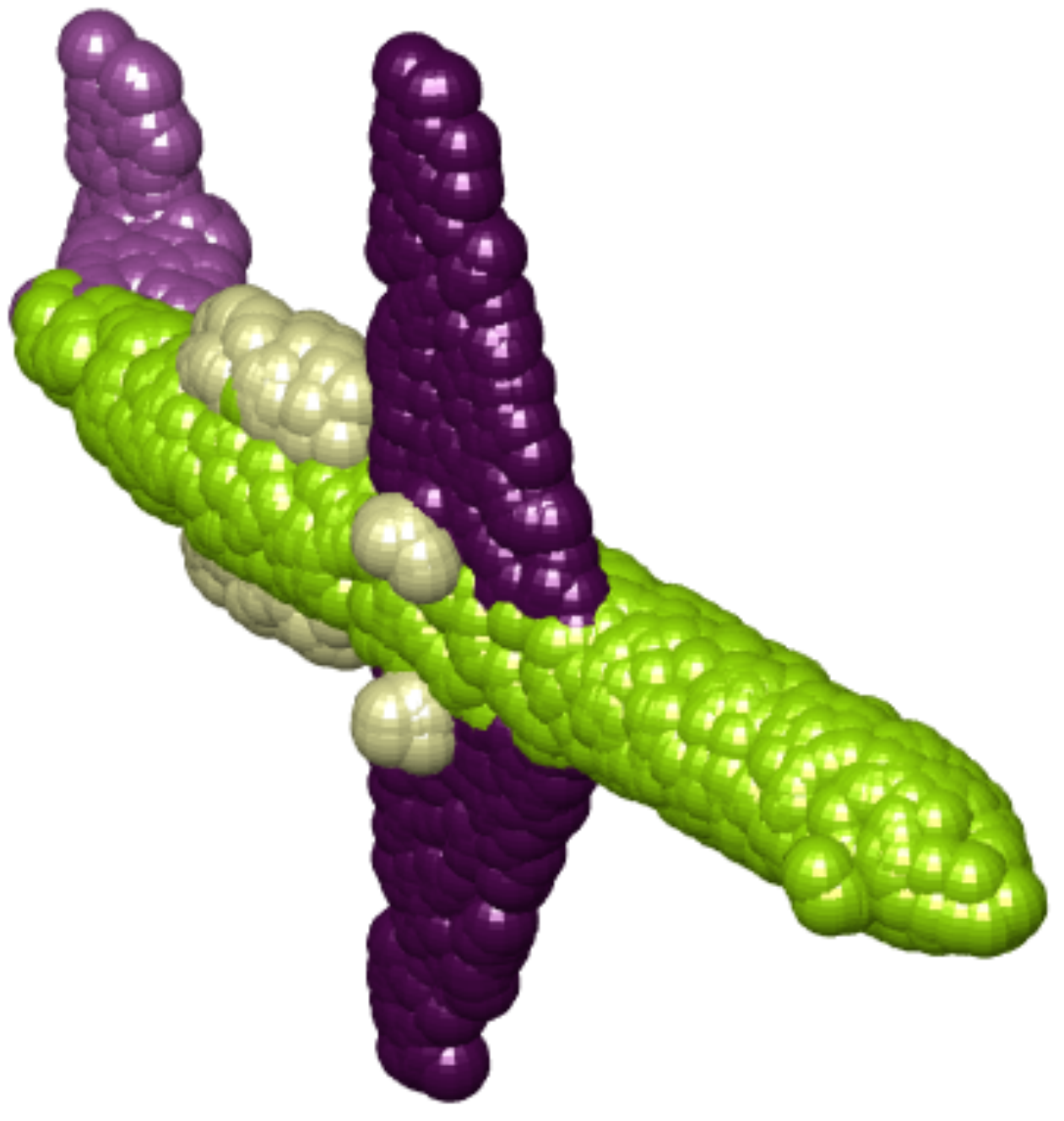}}
    \subfigure[airplane]{\includegraphics[scale=0.18]{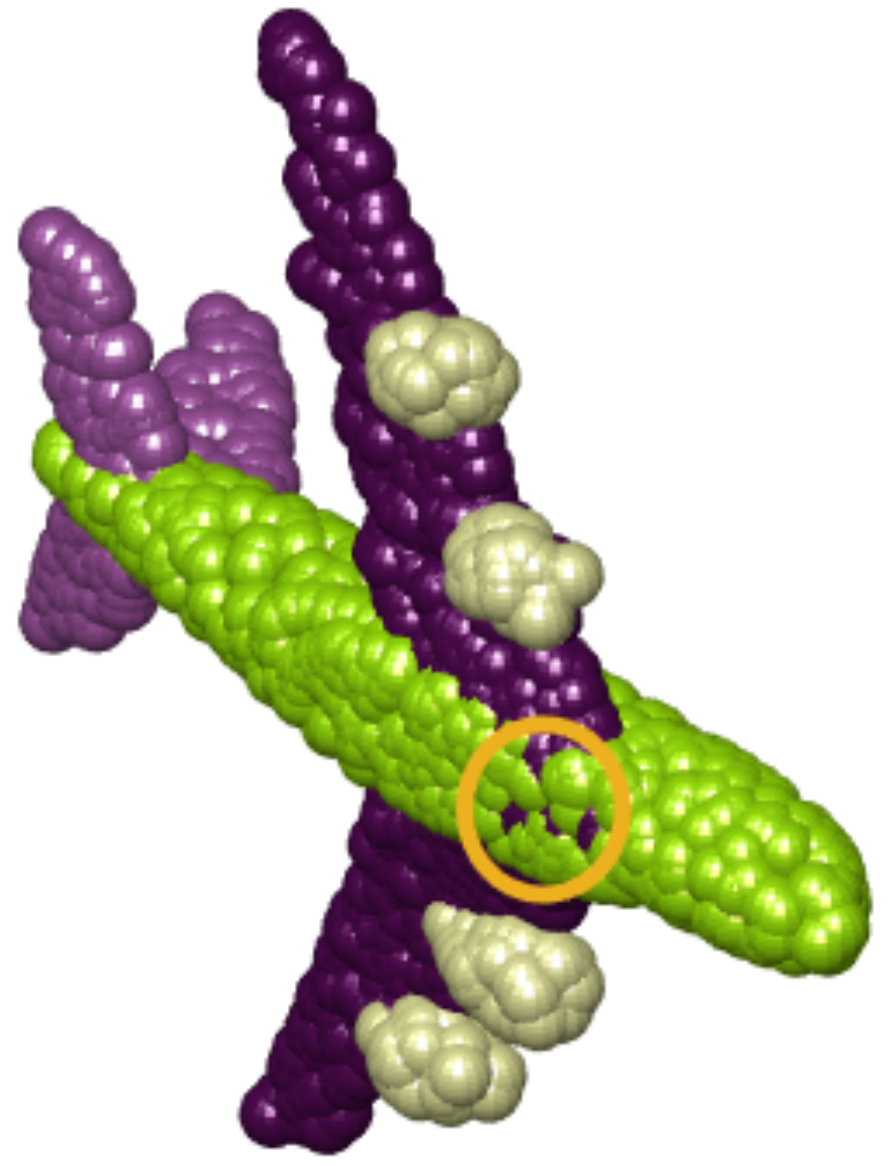}
    \includegraphics[scale=0.18]{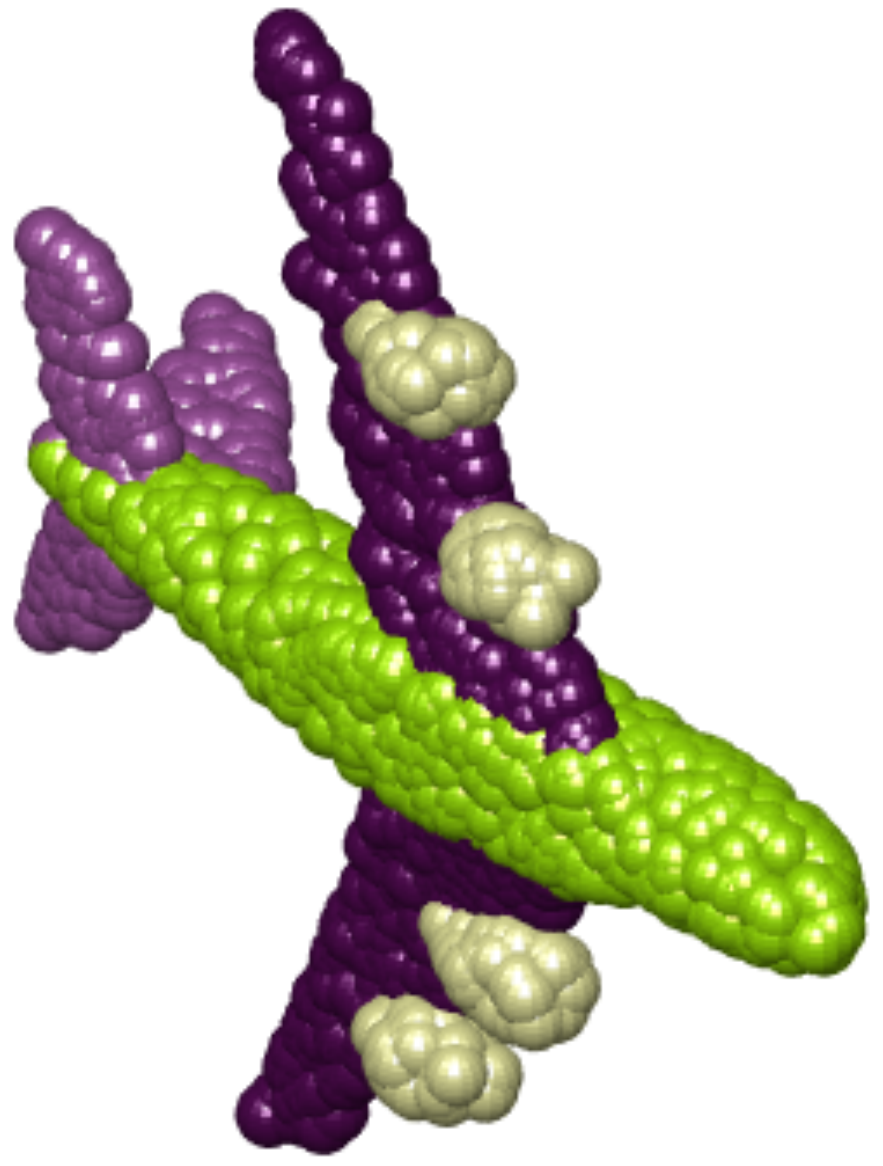}}
        \subfigure[airplane]{\includegraphics[scale=0.20]{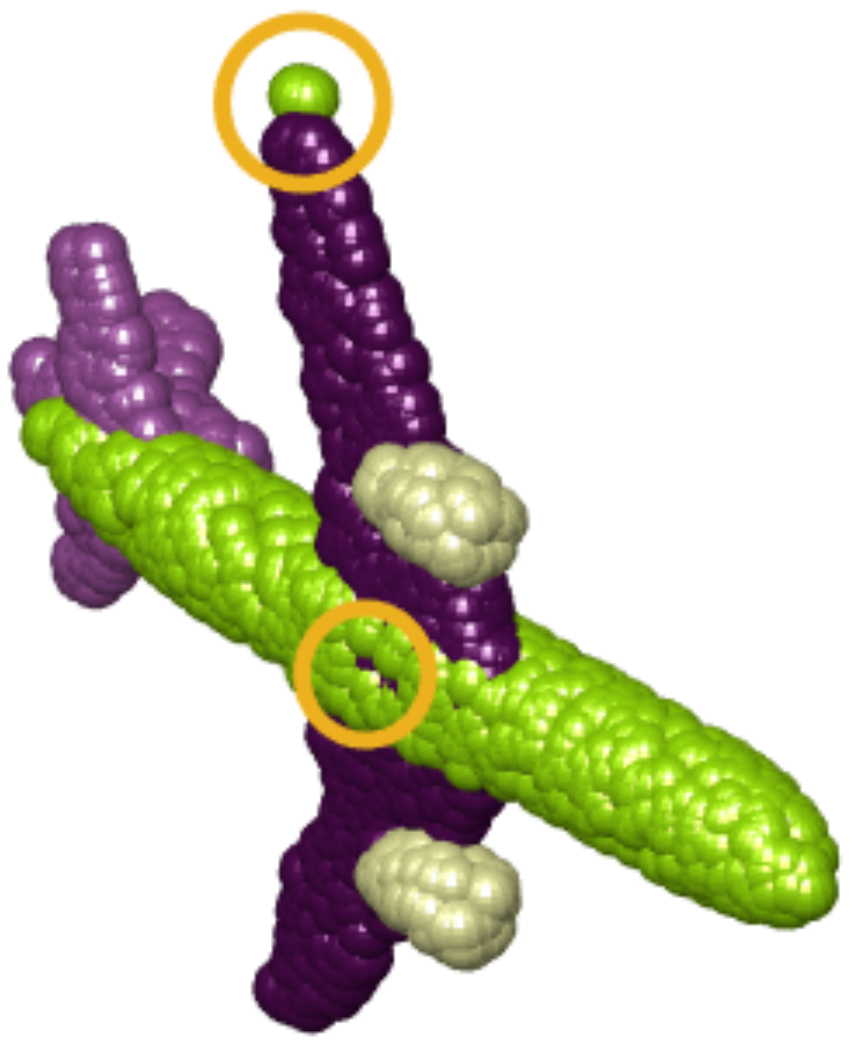}
   \includegraphics[scale=0.20]{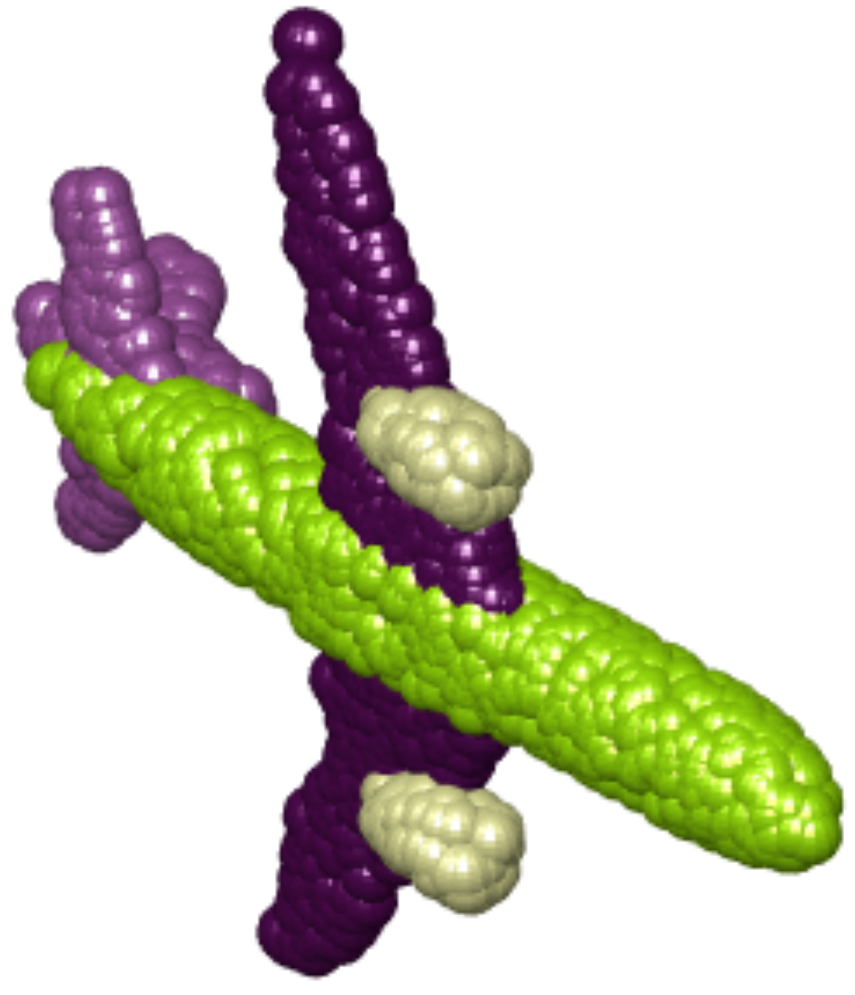}}
    ~~~~~~
    \subfigure[airplane]{\includegraphics[scale=0.15]{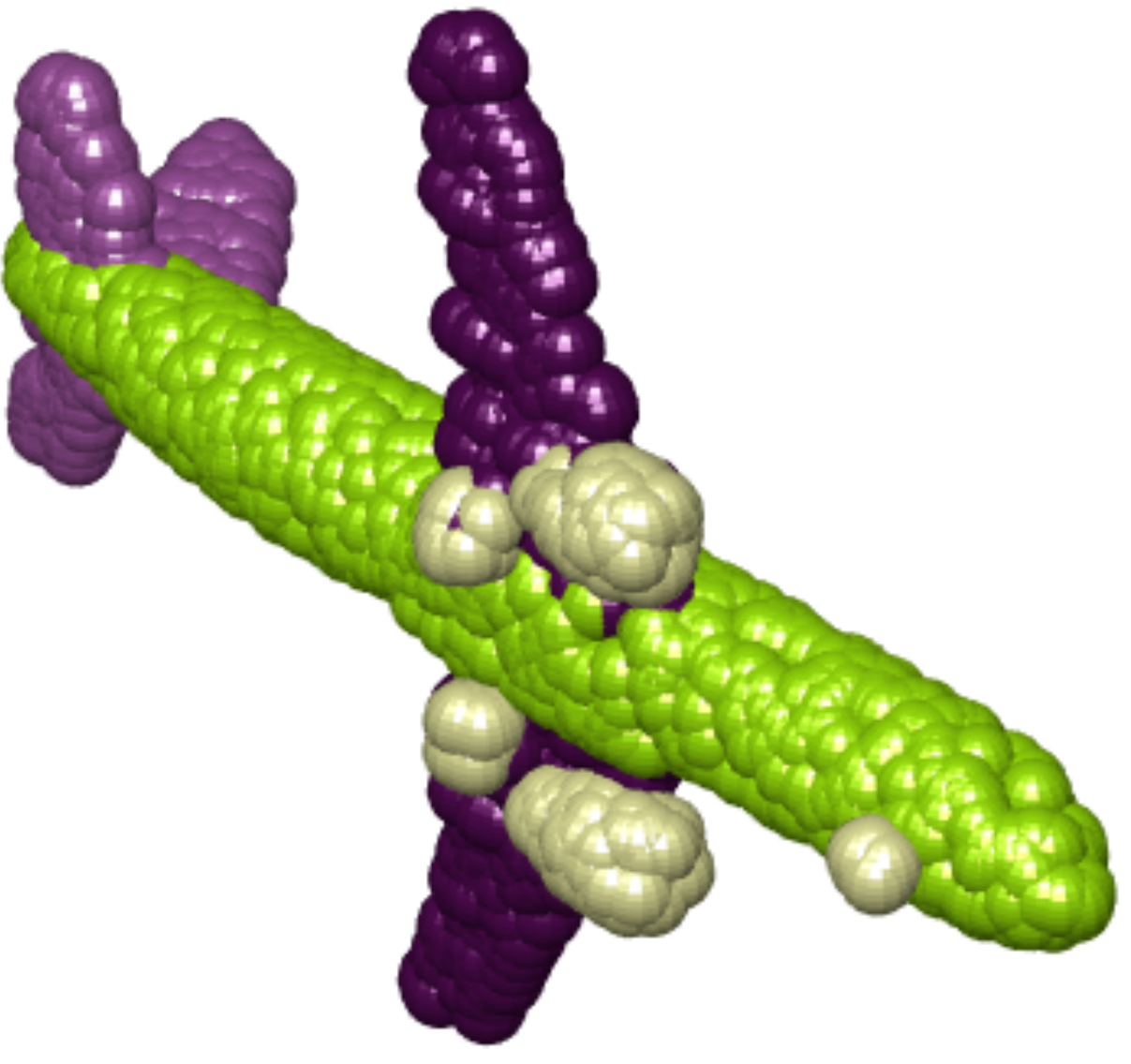}
    \includegraphics[scale=0.15]{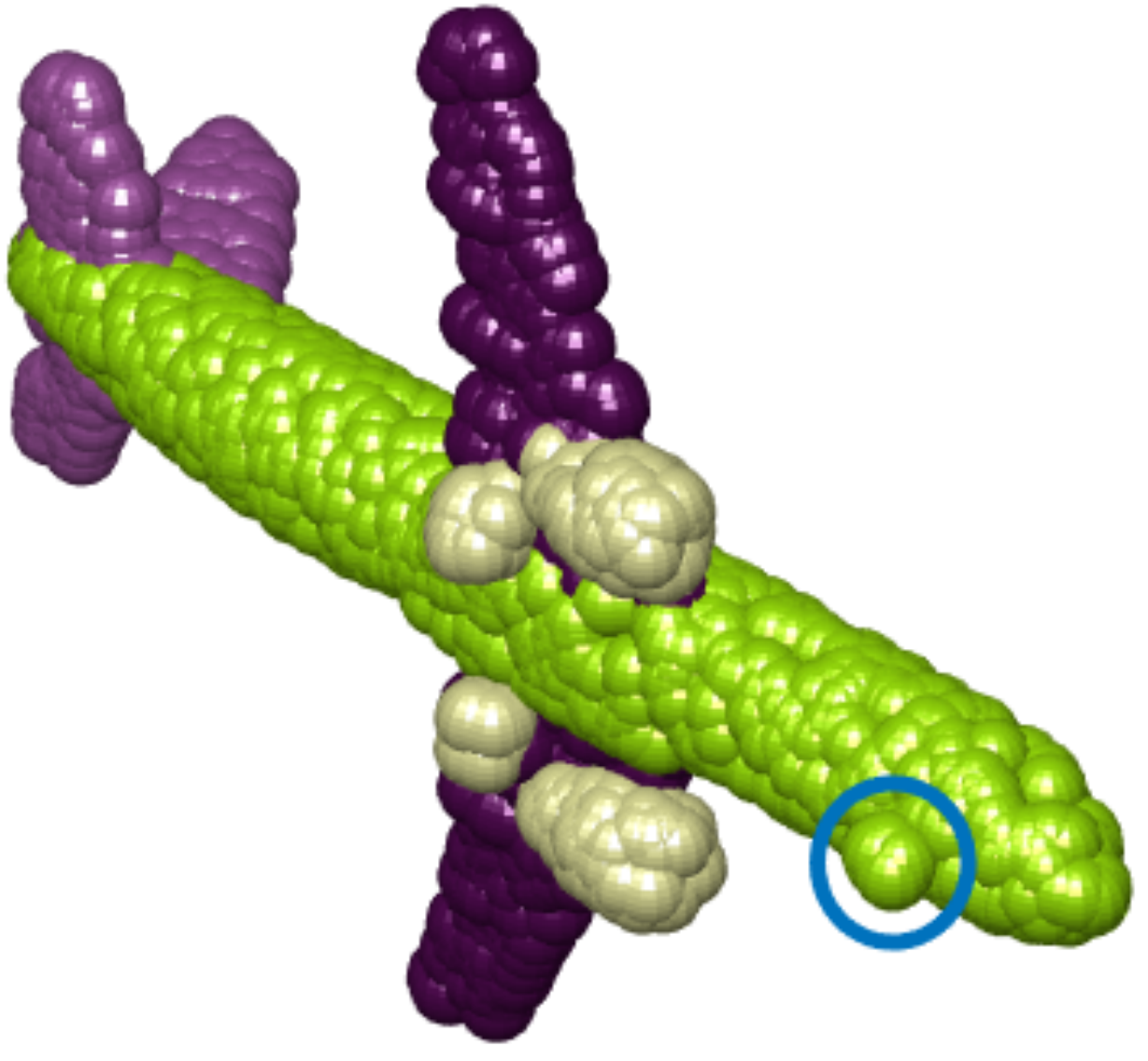}}
        \subfigure[bag]{\includegraphics[scale=0.13]{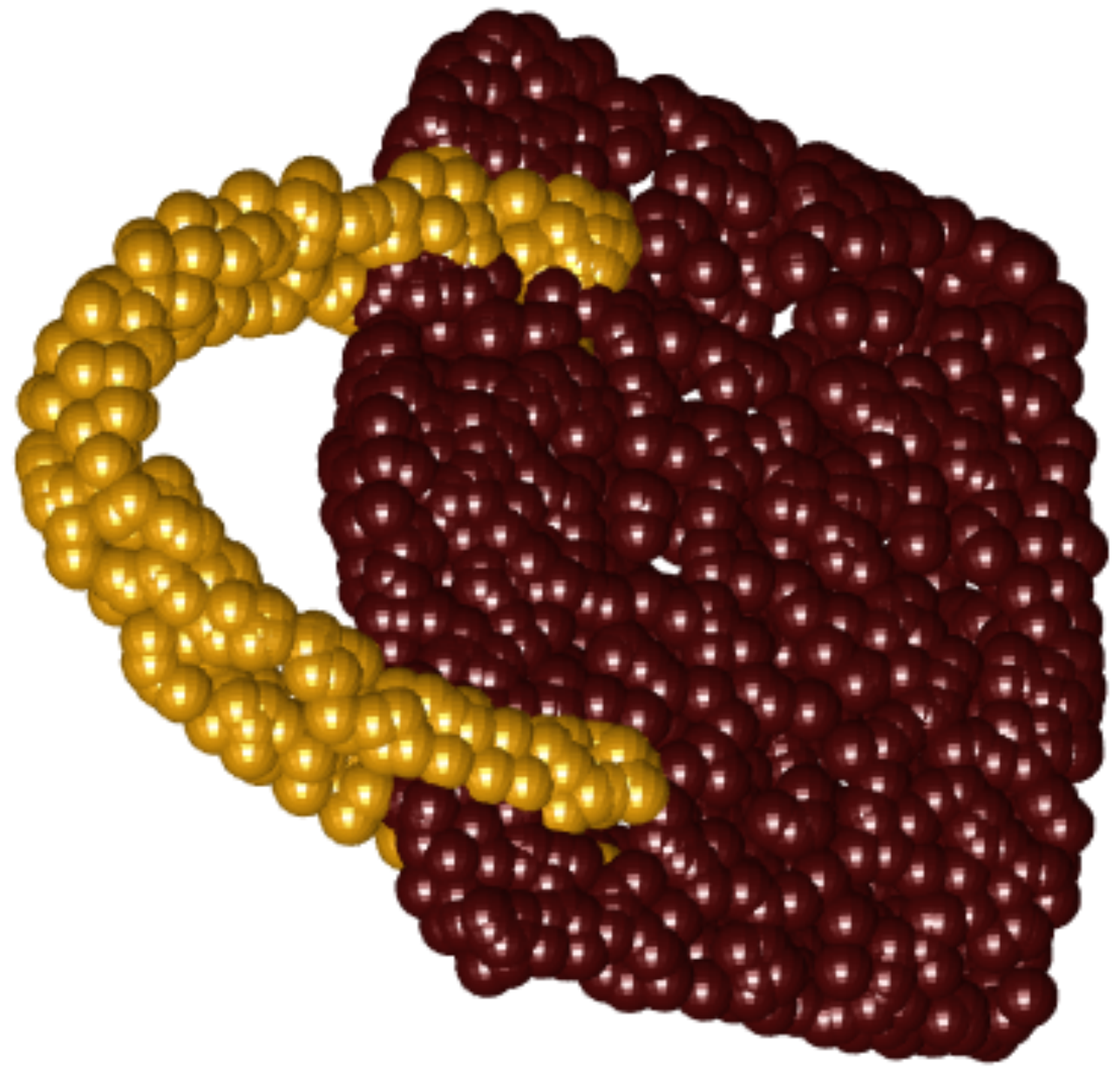}
    \includegraphics[scale=0.13]{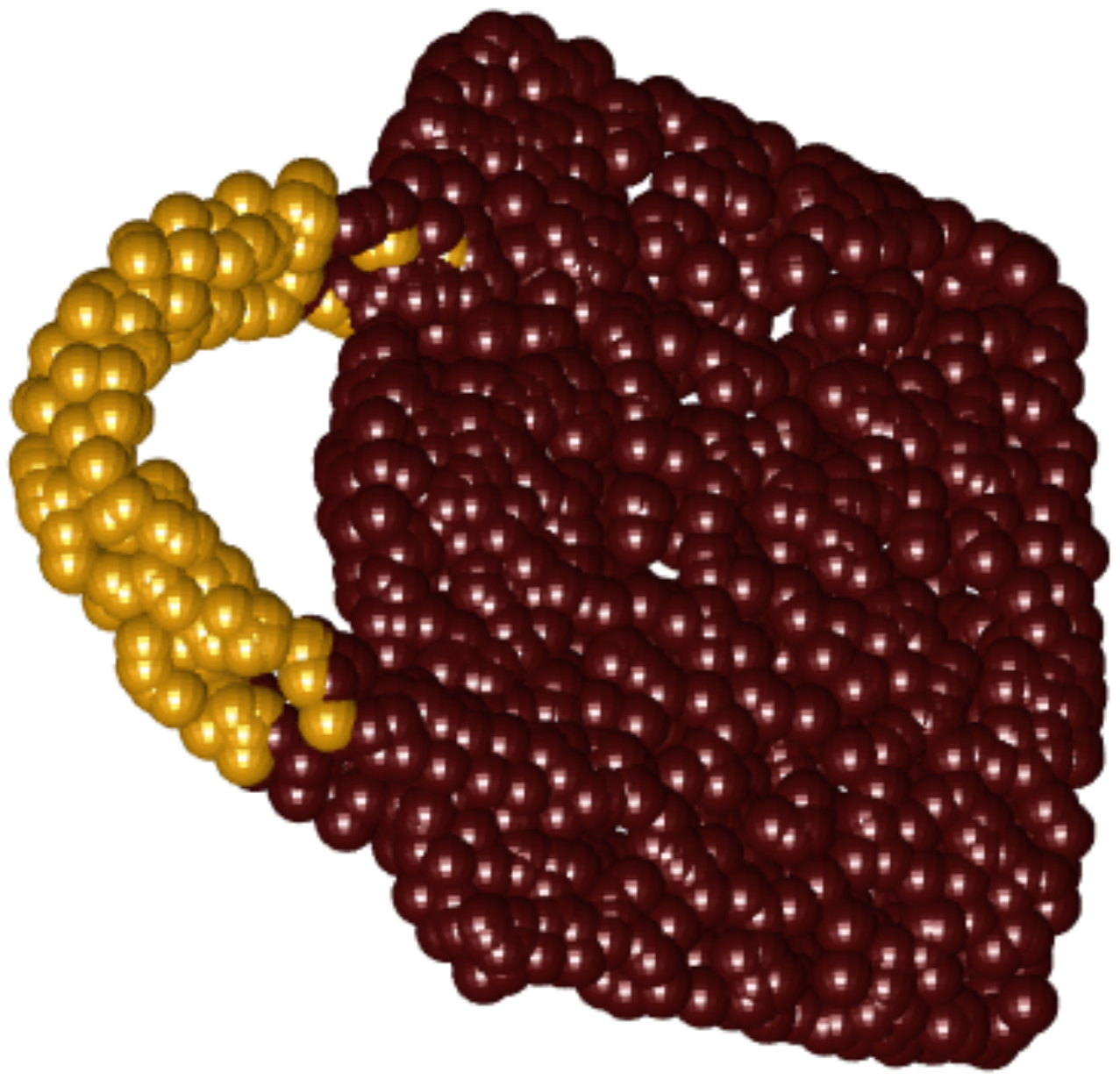}}
    \subfigure[bag]{\includegraphics[scale=0.10]{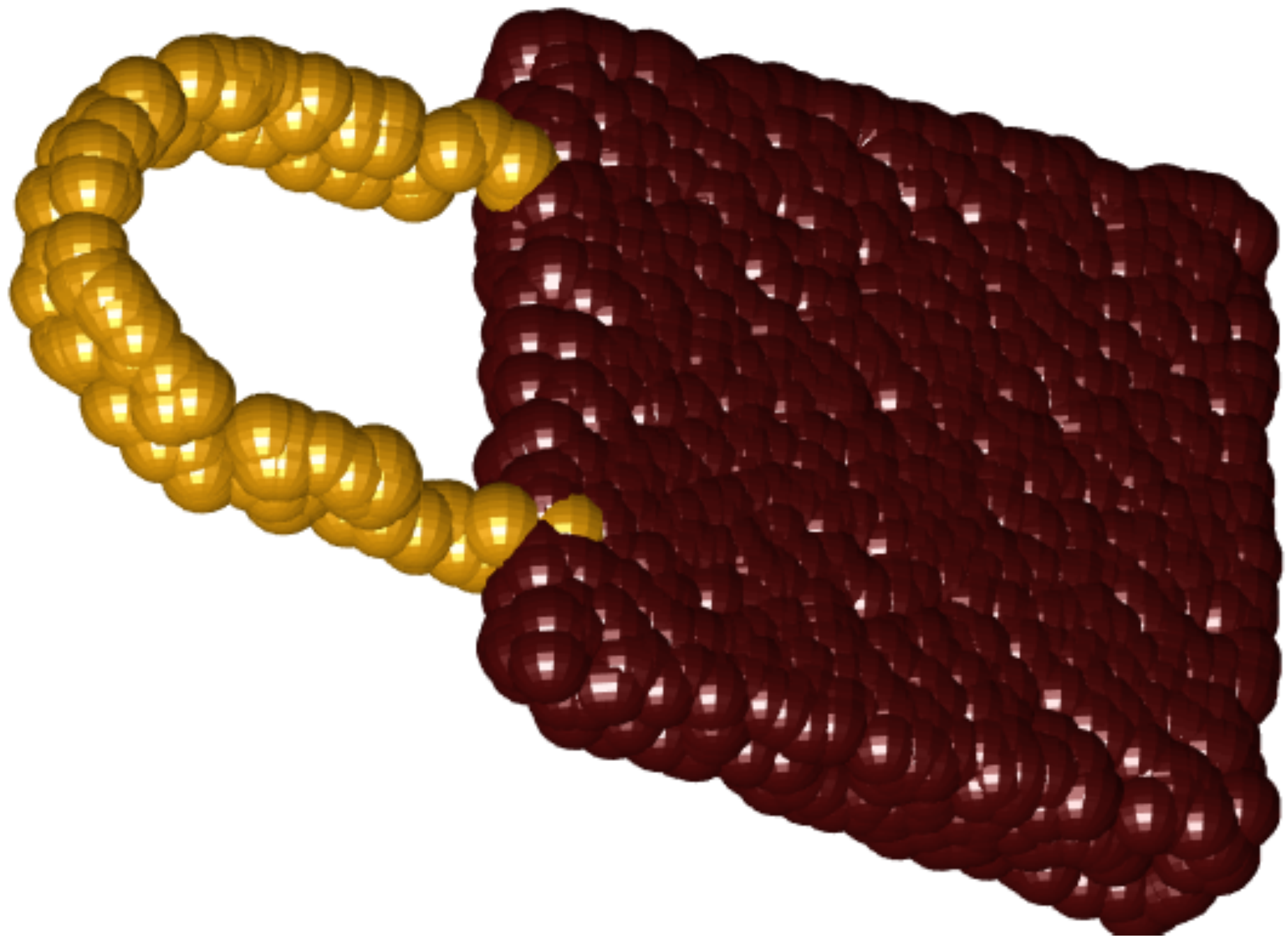}
    \includegraphics[scale=0.10]{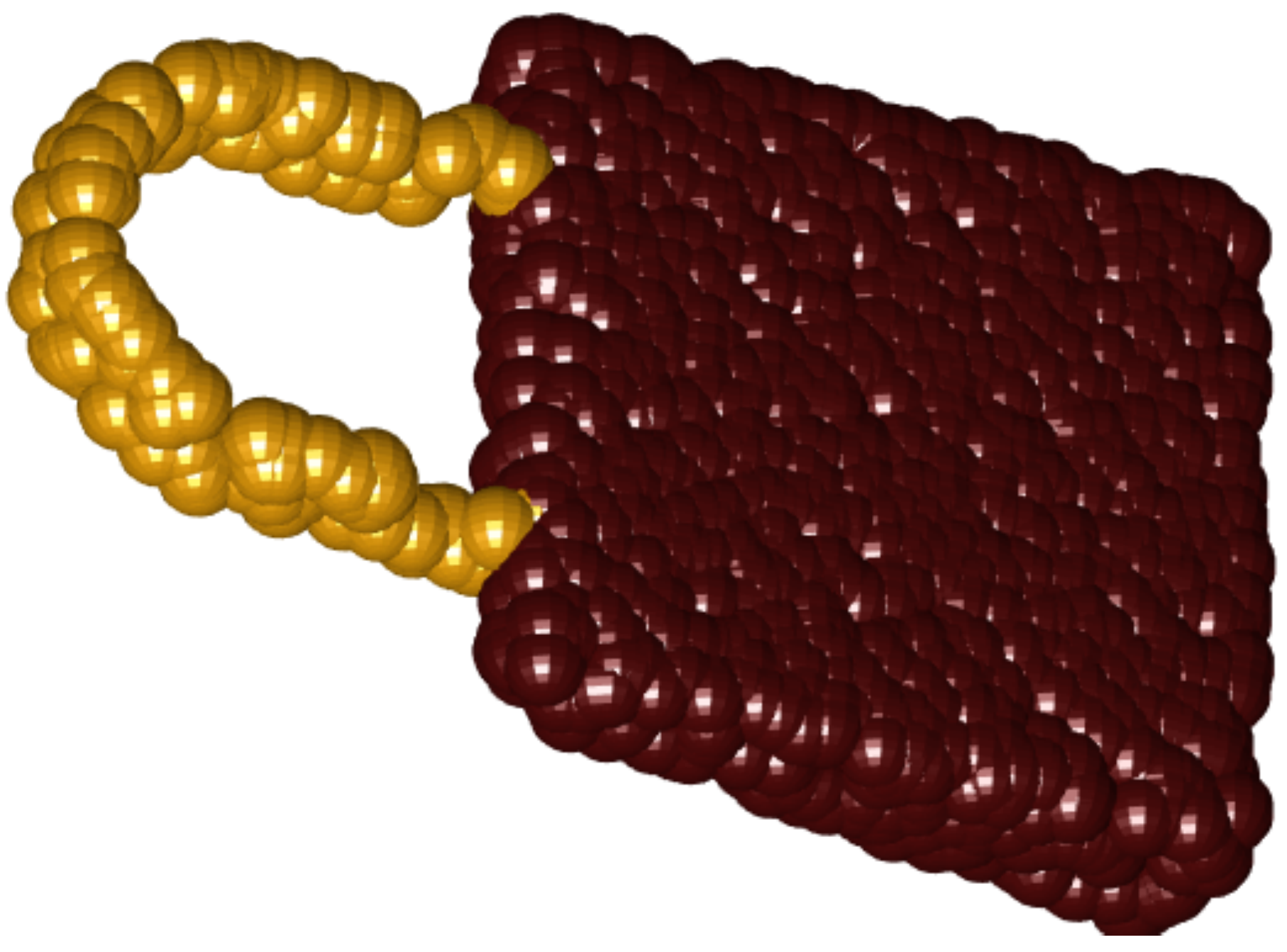}}
     ~~~~~~
    \subfigure[cap]{\includegraphics[scale=0.15]{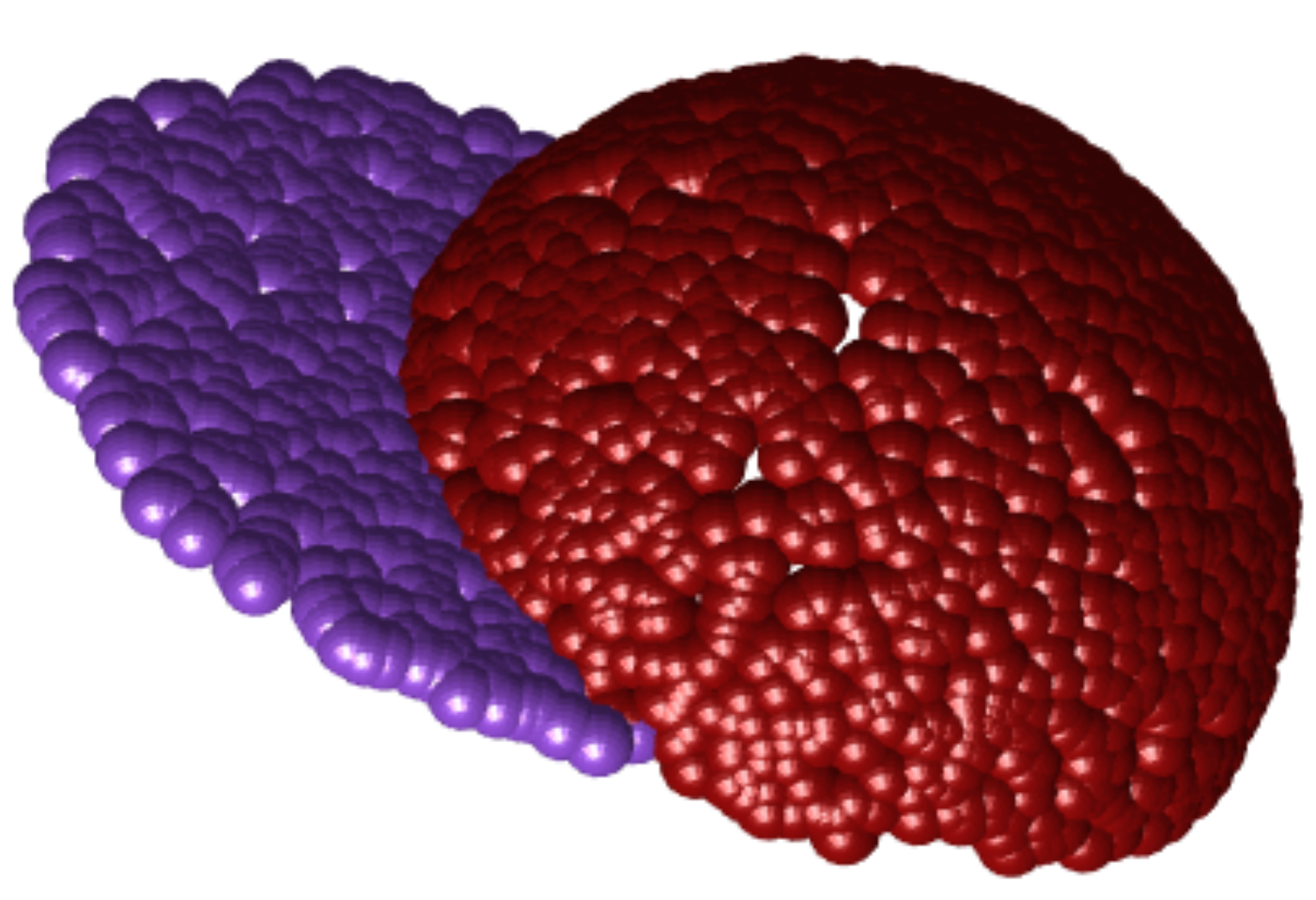}
   \includegraphics[scale=0.15]{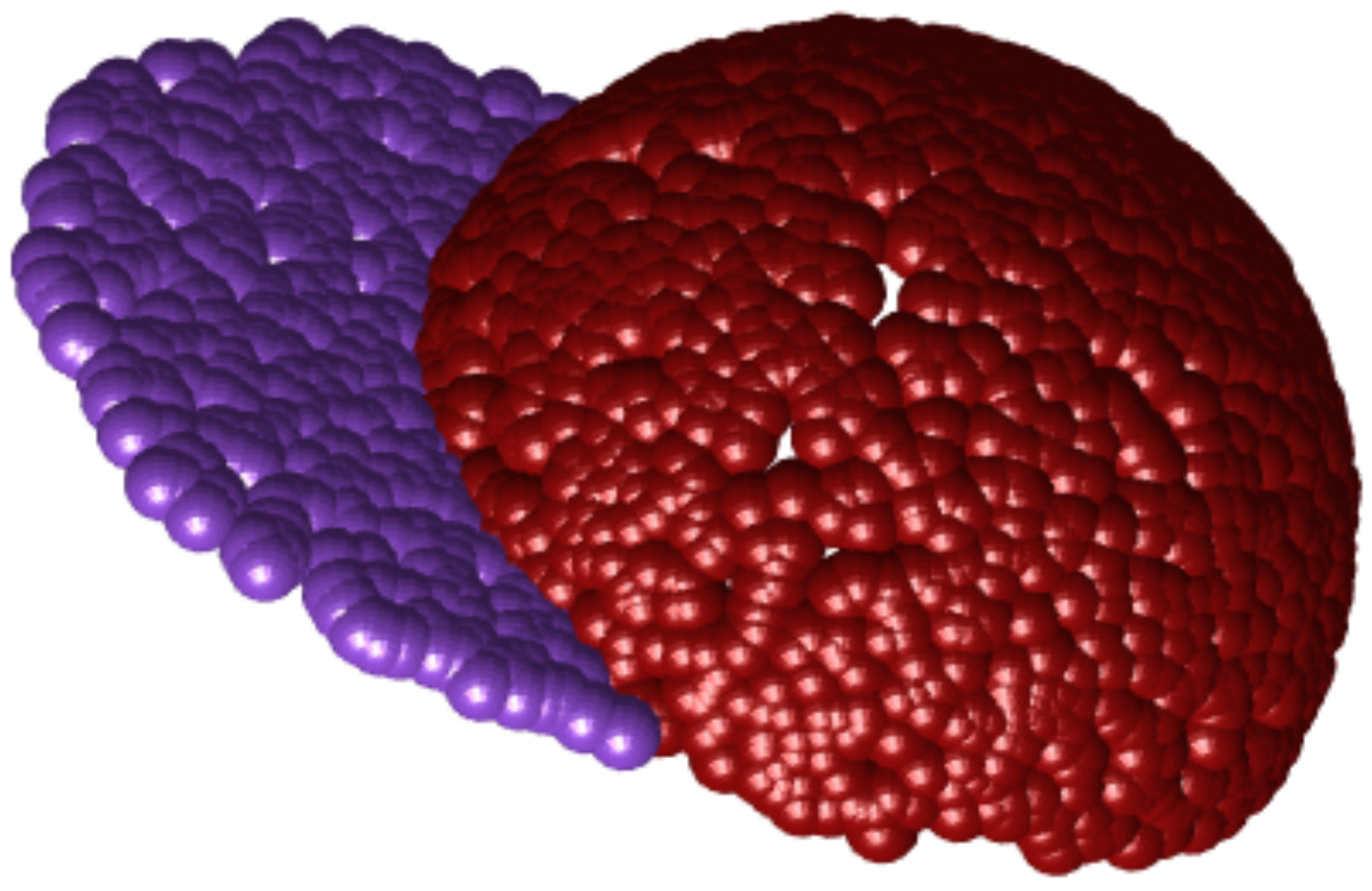}}
        \subfigure[cap]{\includegraphics[scale=0.13]{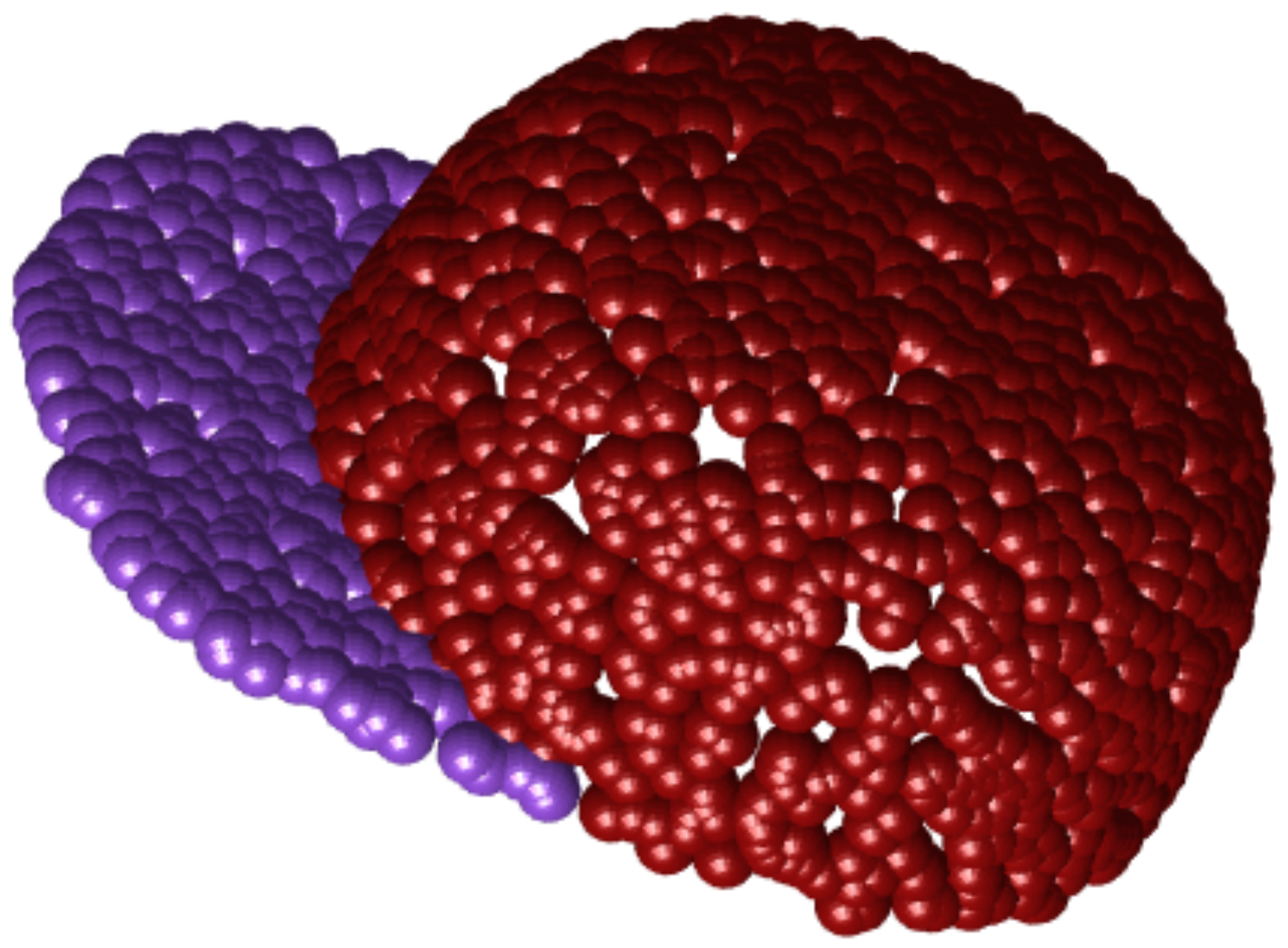}
   \includegraphics[scale=0.13]{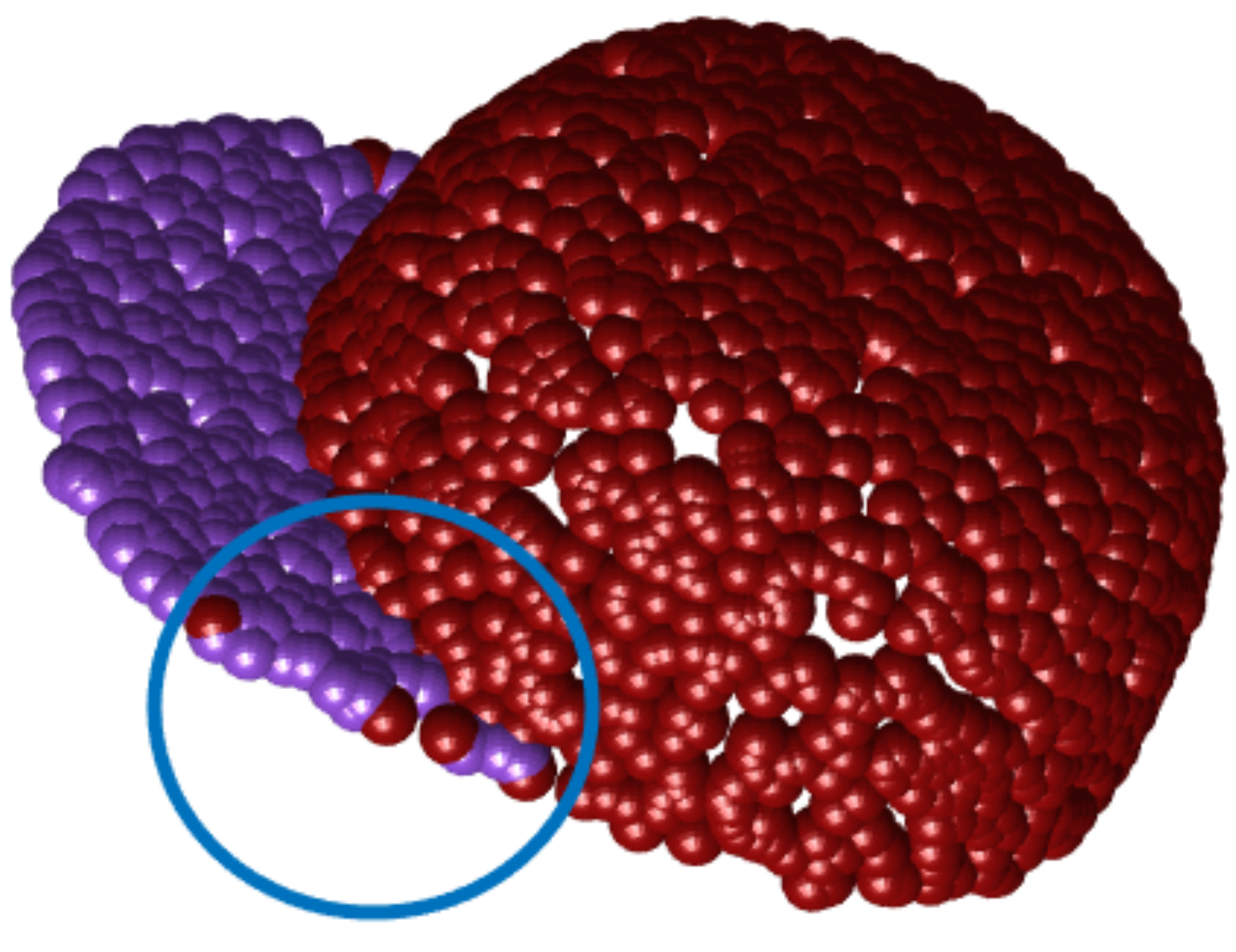}}
    \subfigure[earphone]{\includegraphics[scale=0.15]{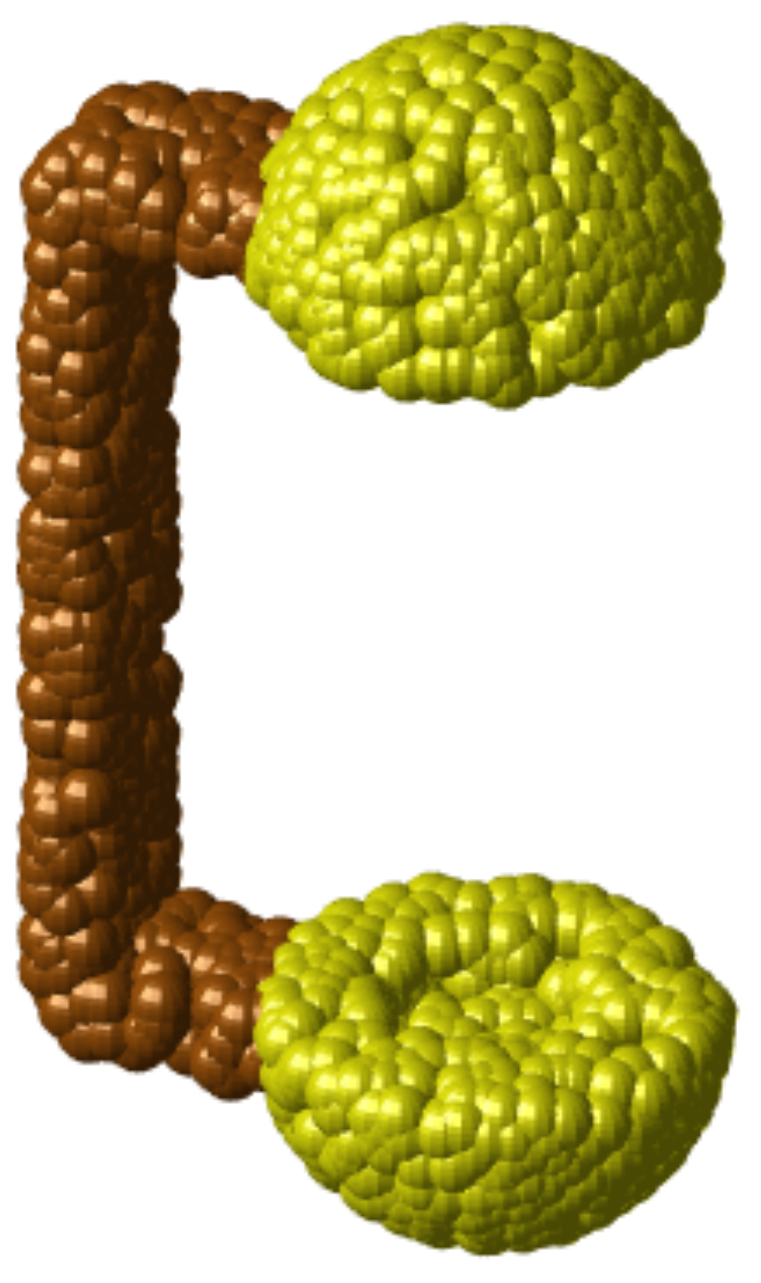}
    \includegraphics[scale=0.15]{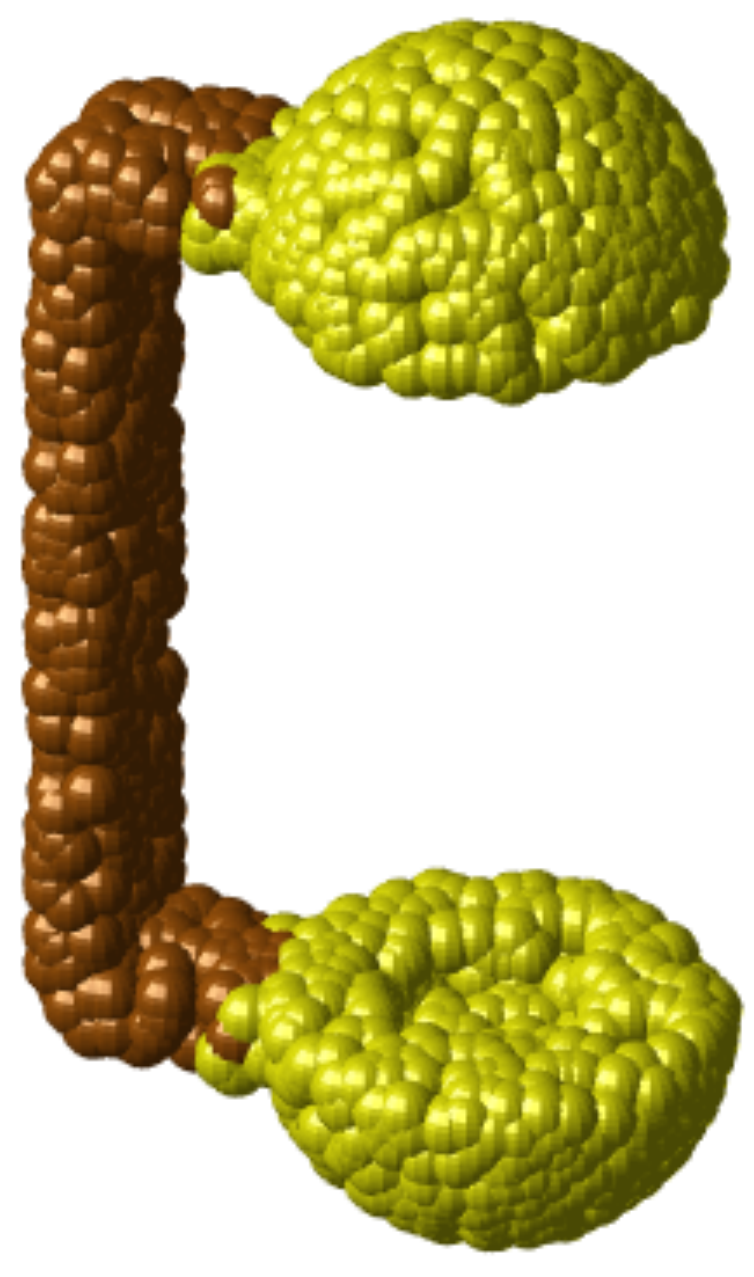}}
     ~~~~~~
    \subfigure[earphone]{\includegraphics[scale=0.13]{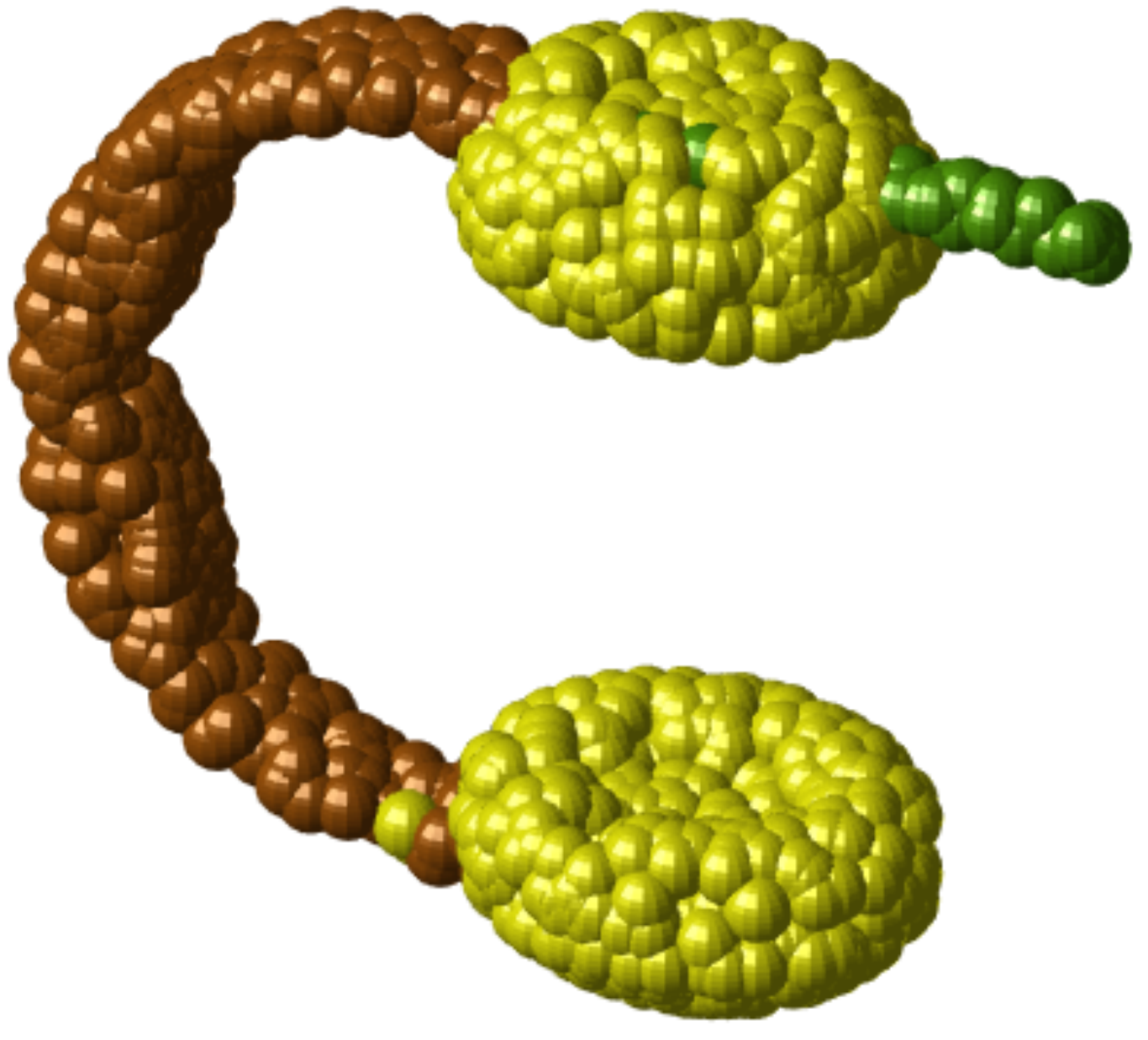}
    \includegraphics[scale=0.13]{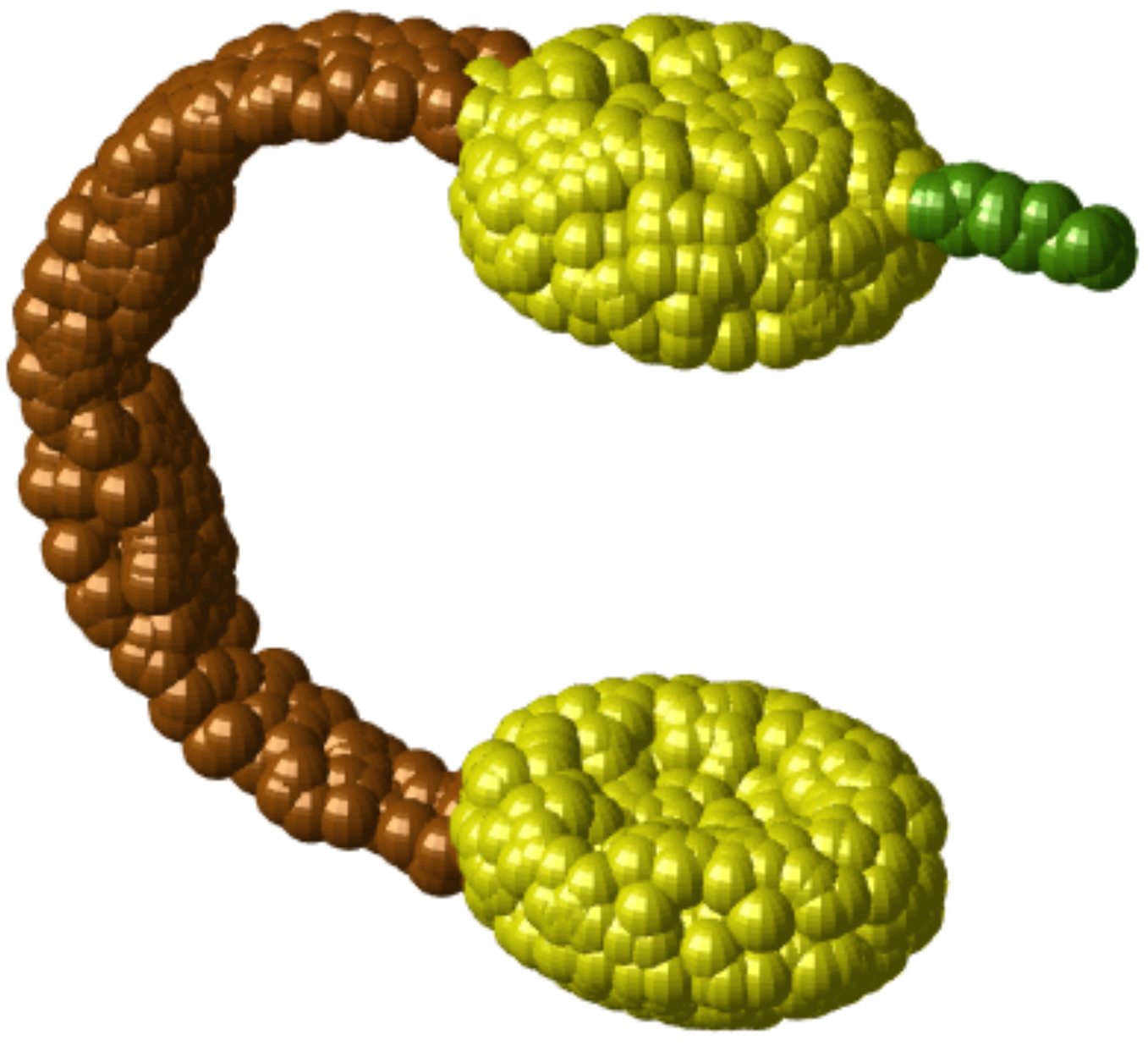}}
        \subfigure[earphone]{\includegraphics[scale=0.13]{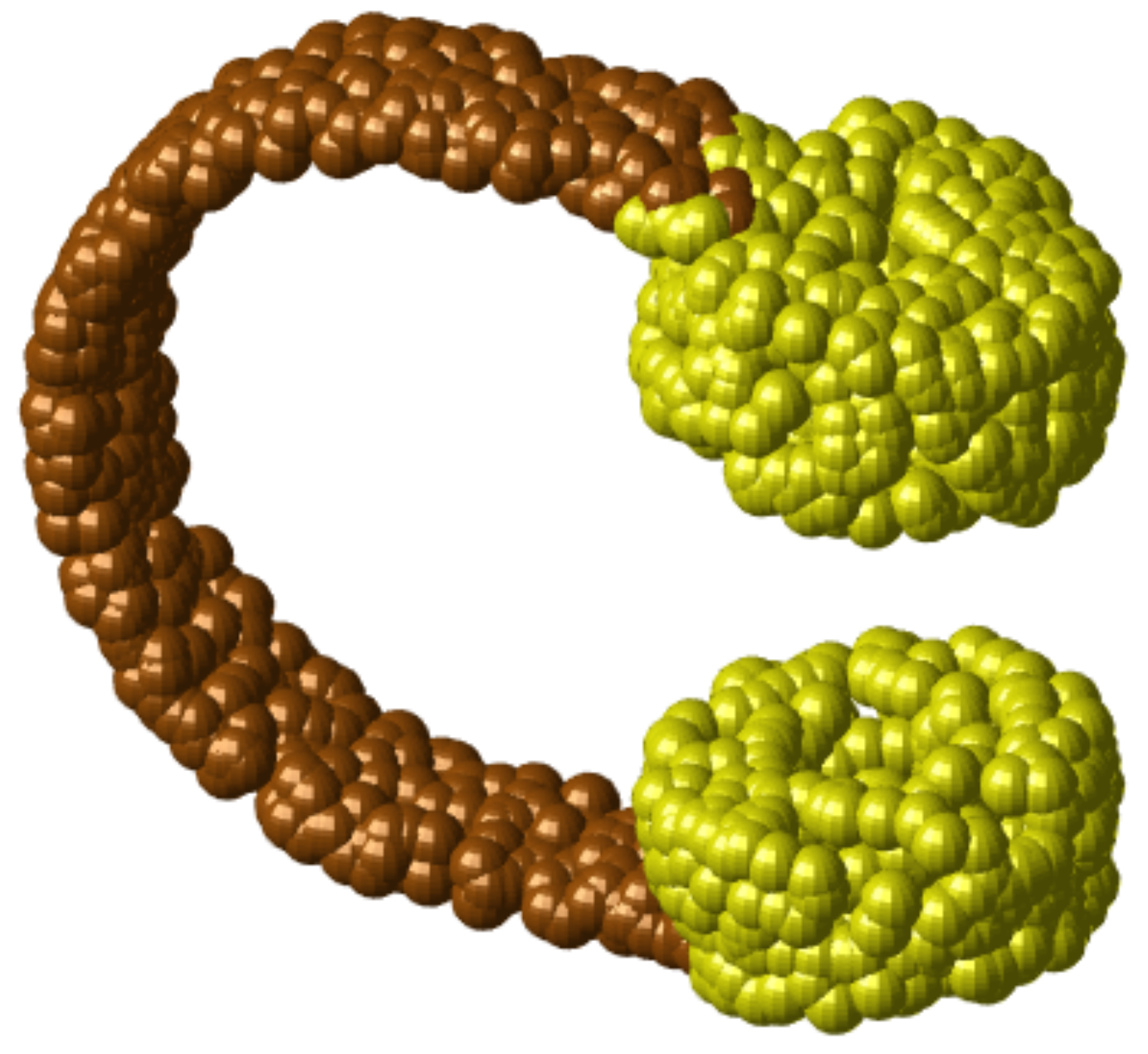}
    \includegraphics[scale=0.13]{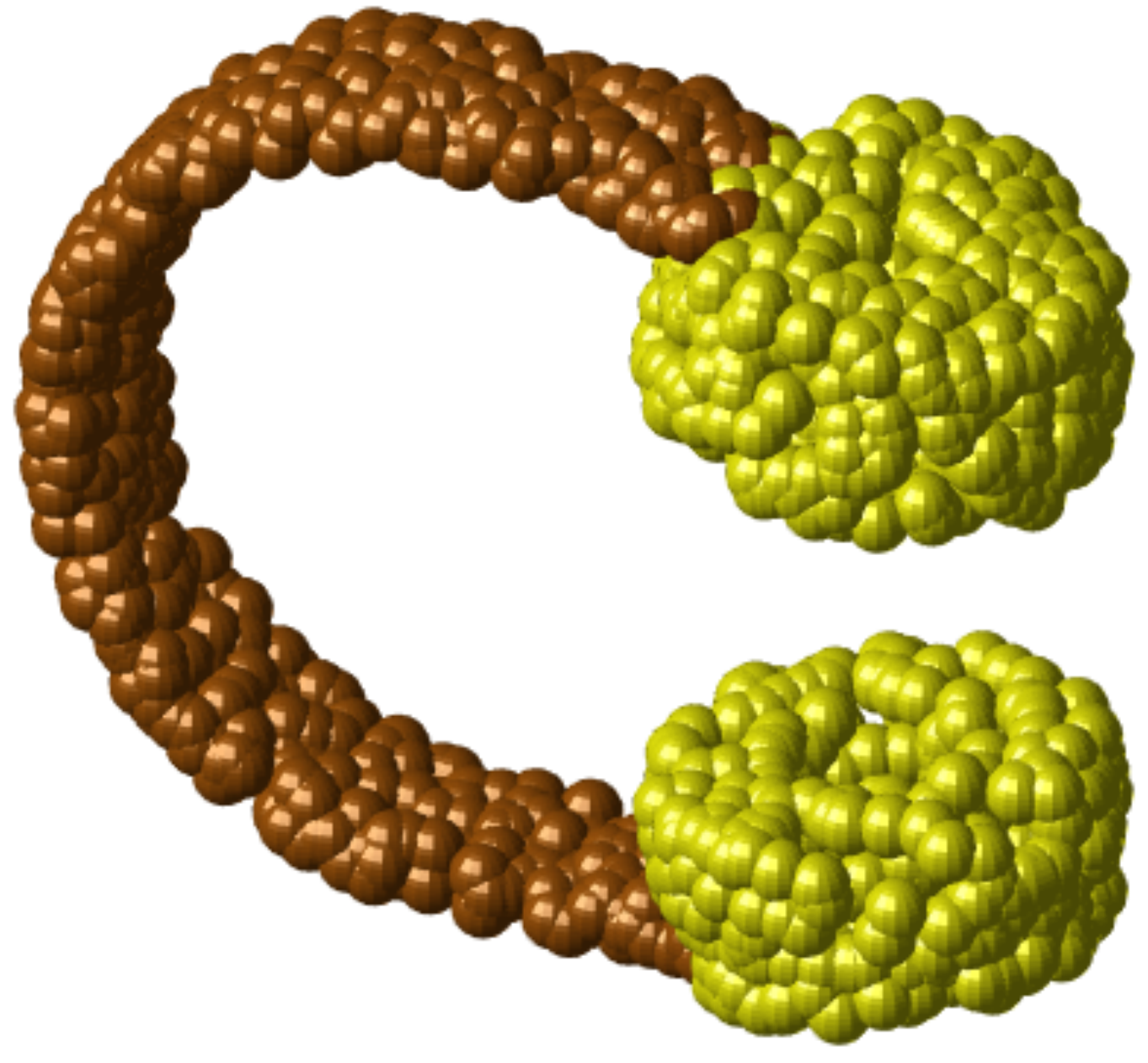}}
    \subfigure[earphone]{\includegraphics[scale=0.13]{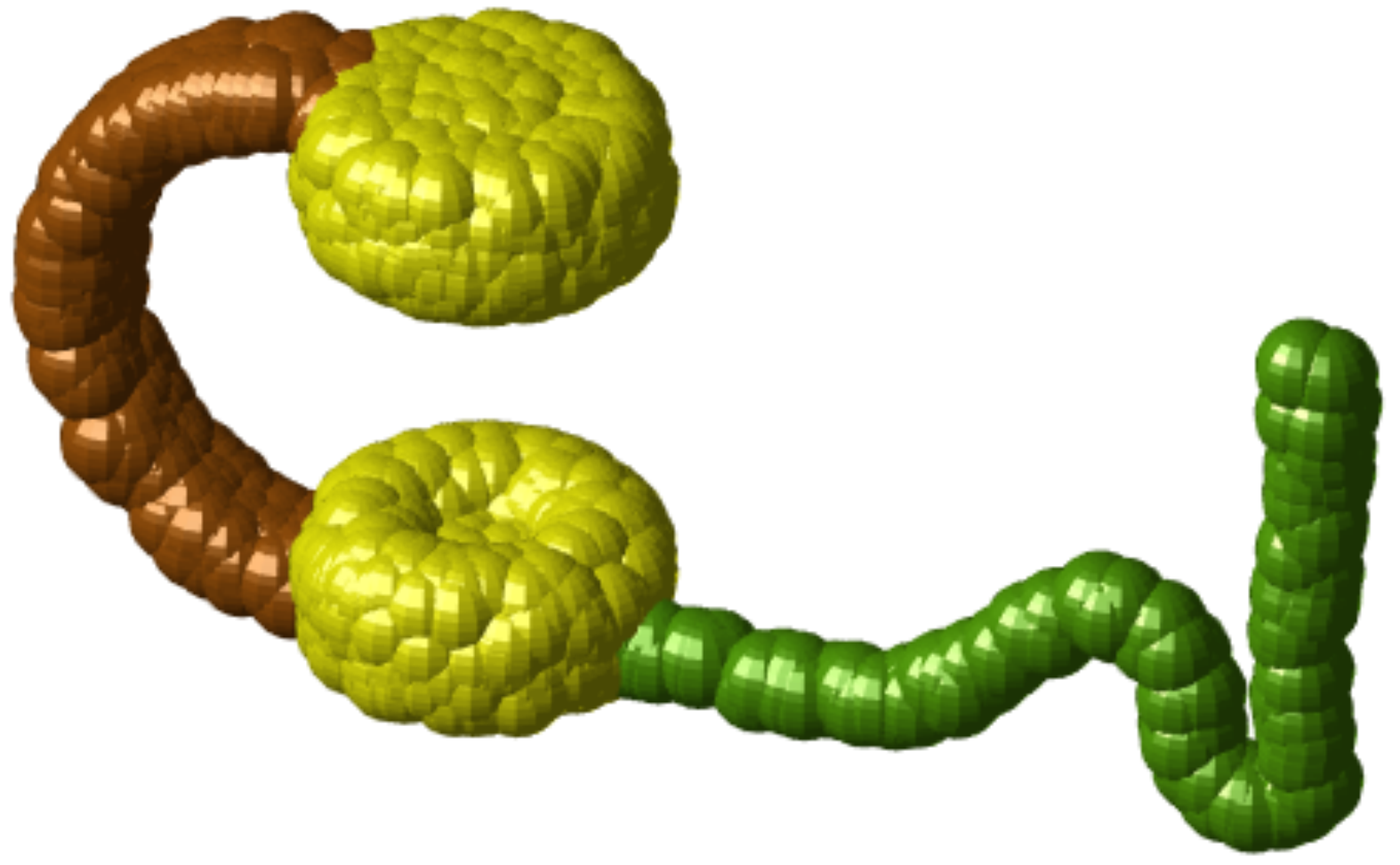}
   \includegraphics[scale=0.13]{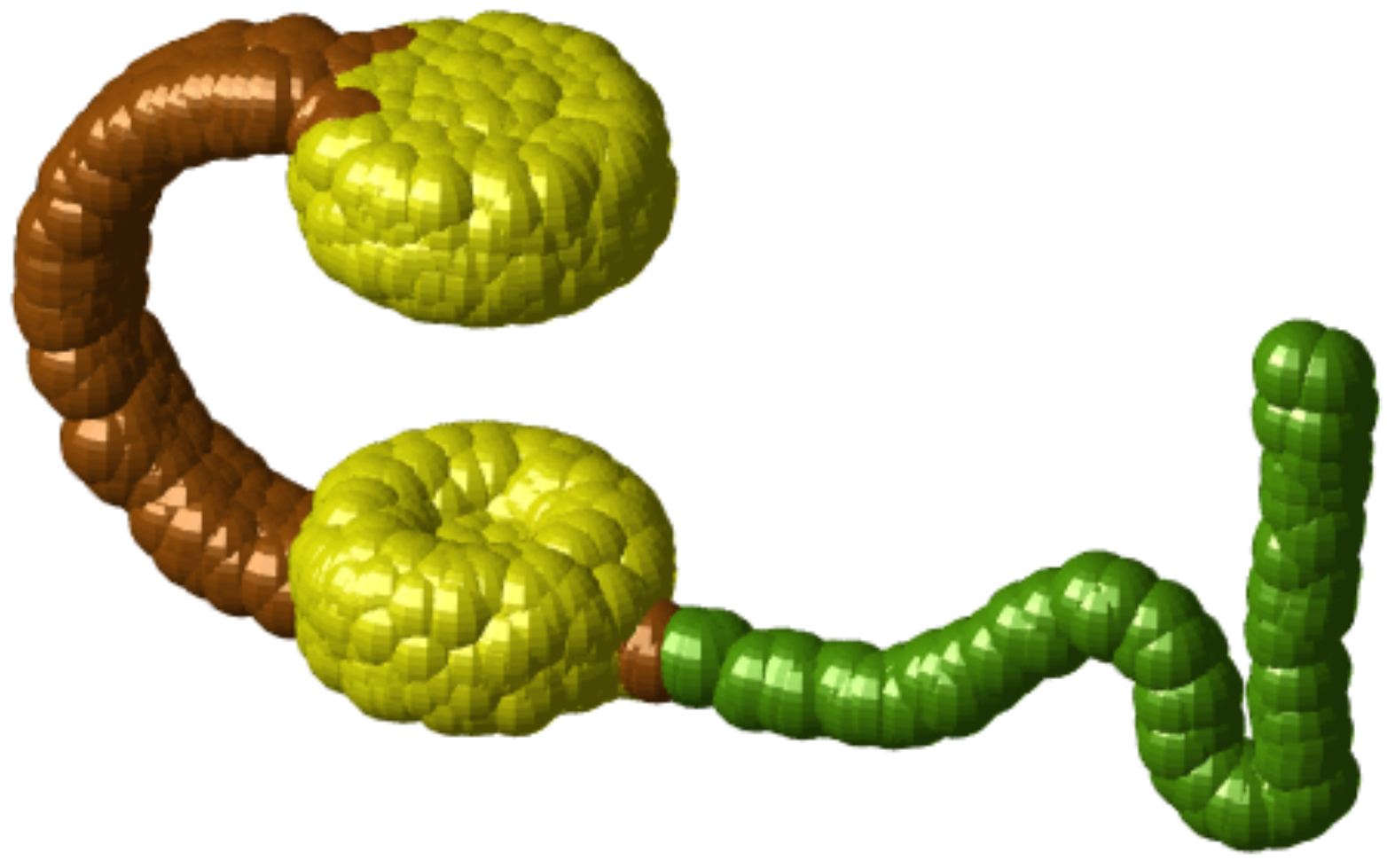}}
    ~~~~~~
    \subfigure[rocket]{\includegraphics[scale=0.13]{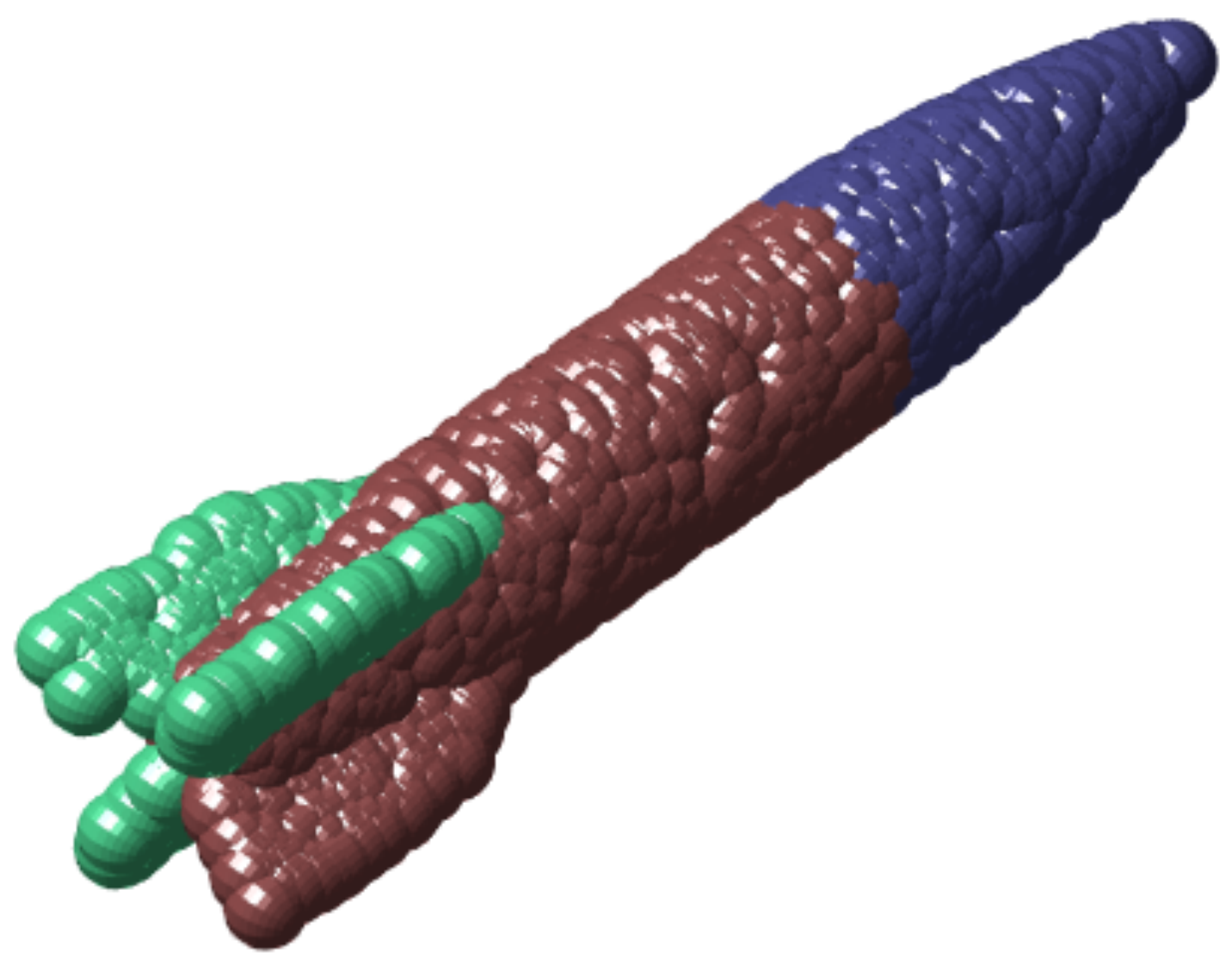}
    \includegraphics[scale=0.13]{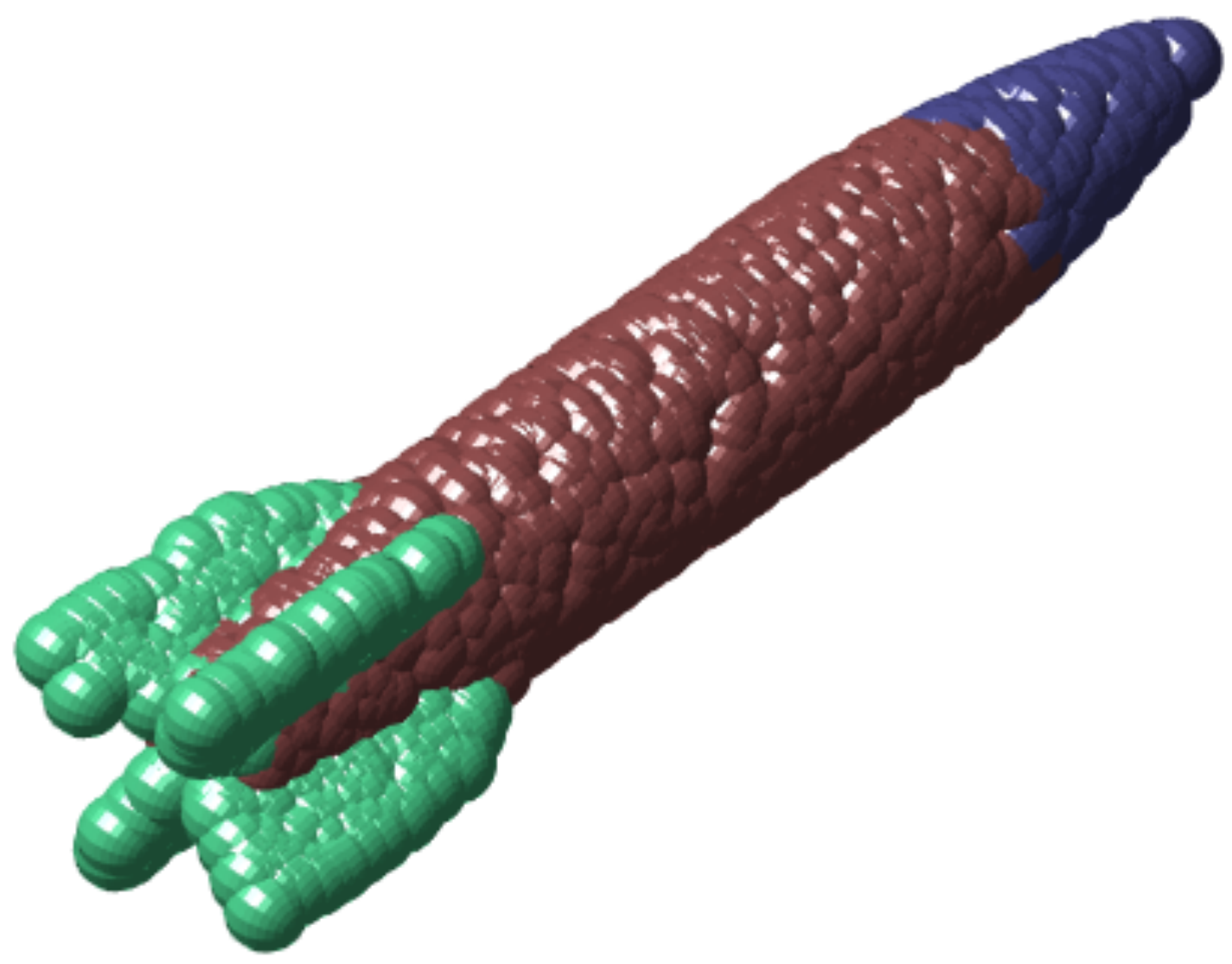}}
    \subfigure[rocket]{\includegraphics[scale=0.13]{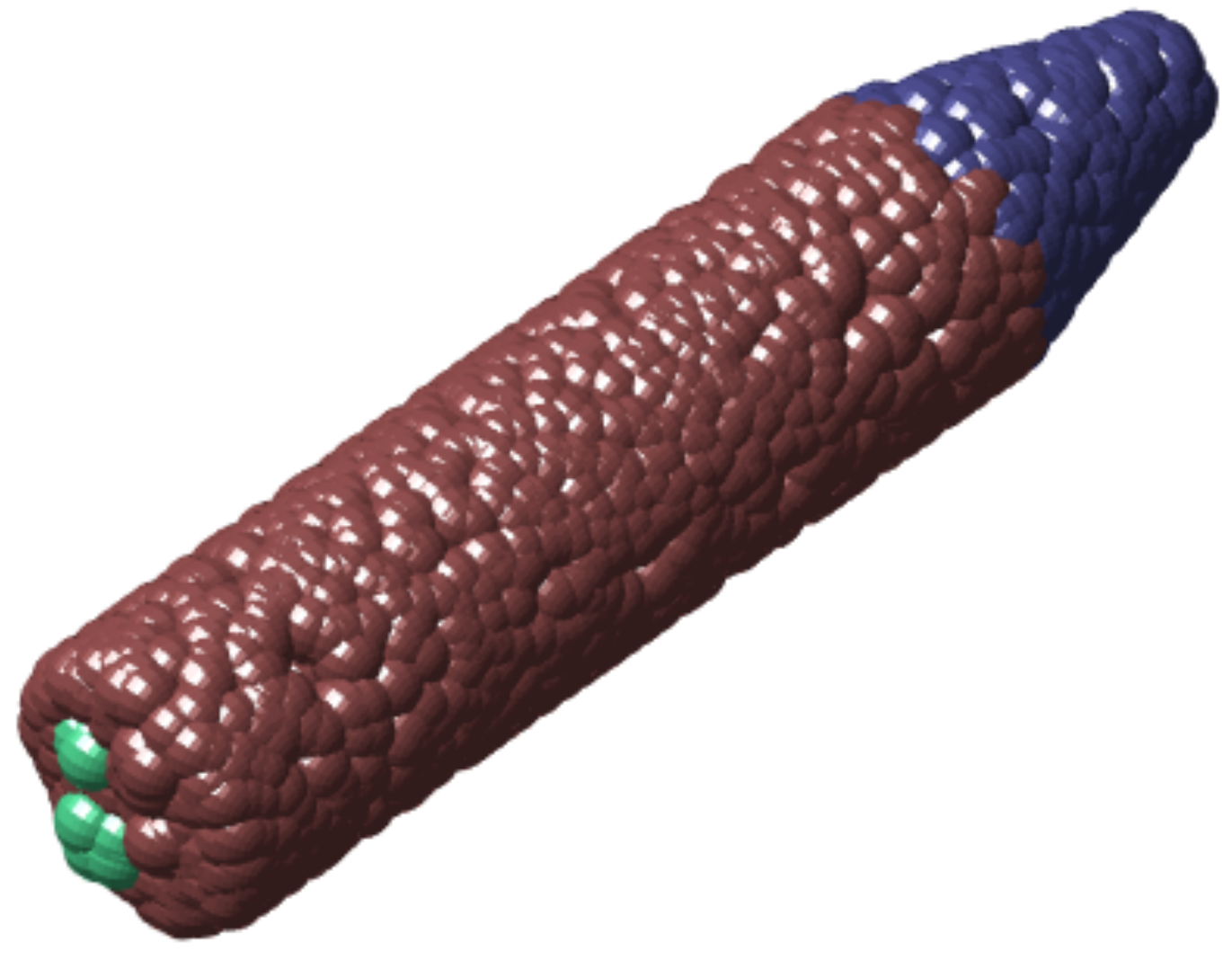}
    \includegraphics[scale=0.13]{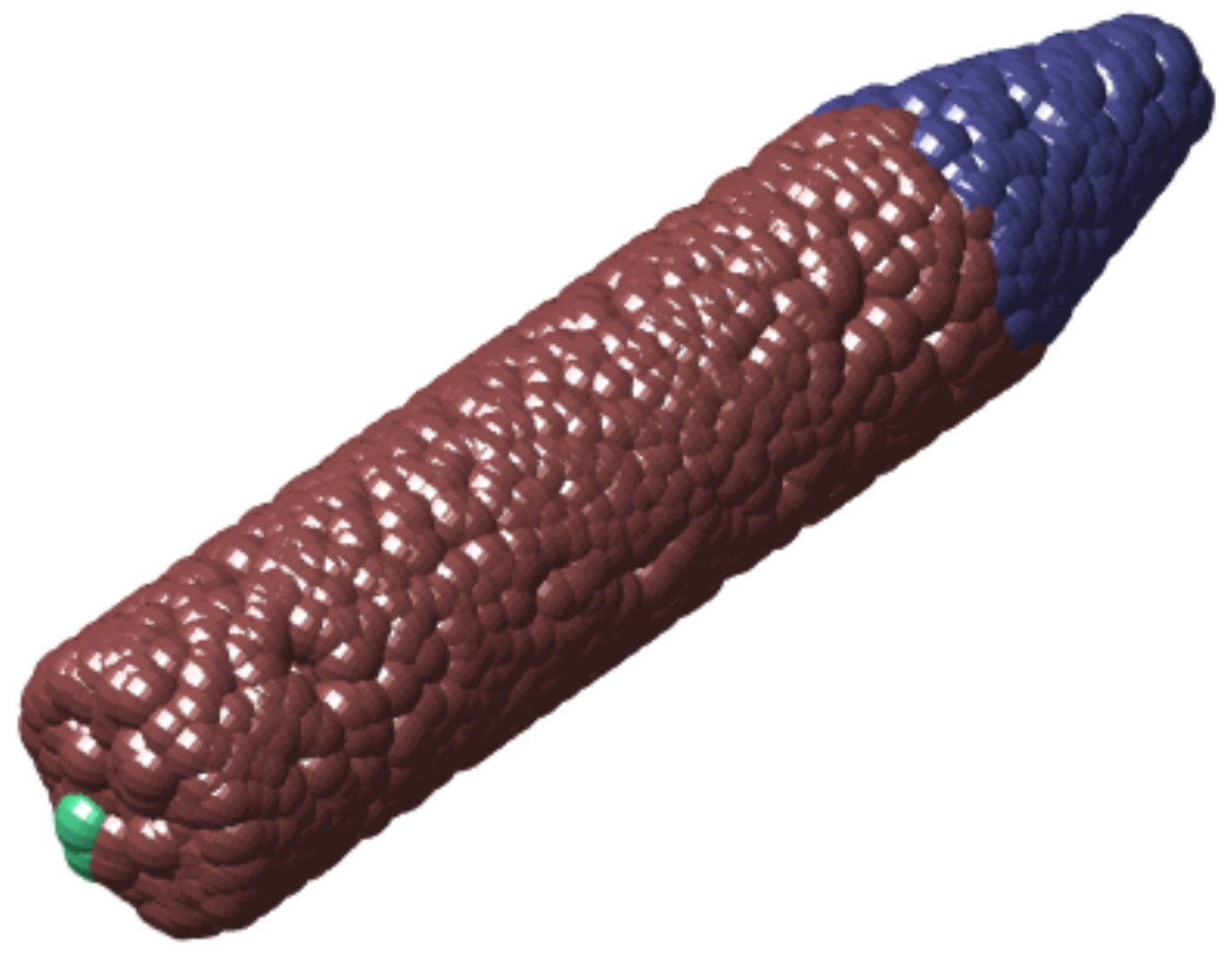}}
    \subfigure[rocket]{\includegraphics[scale=0.13]{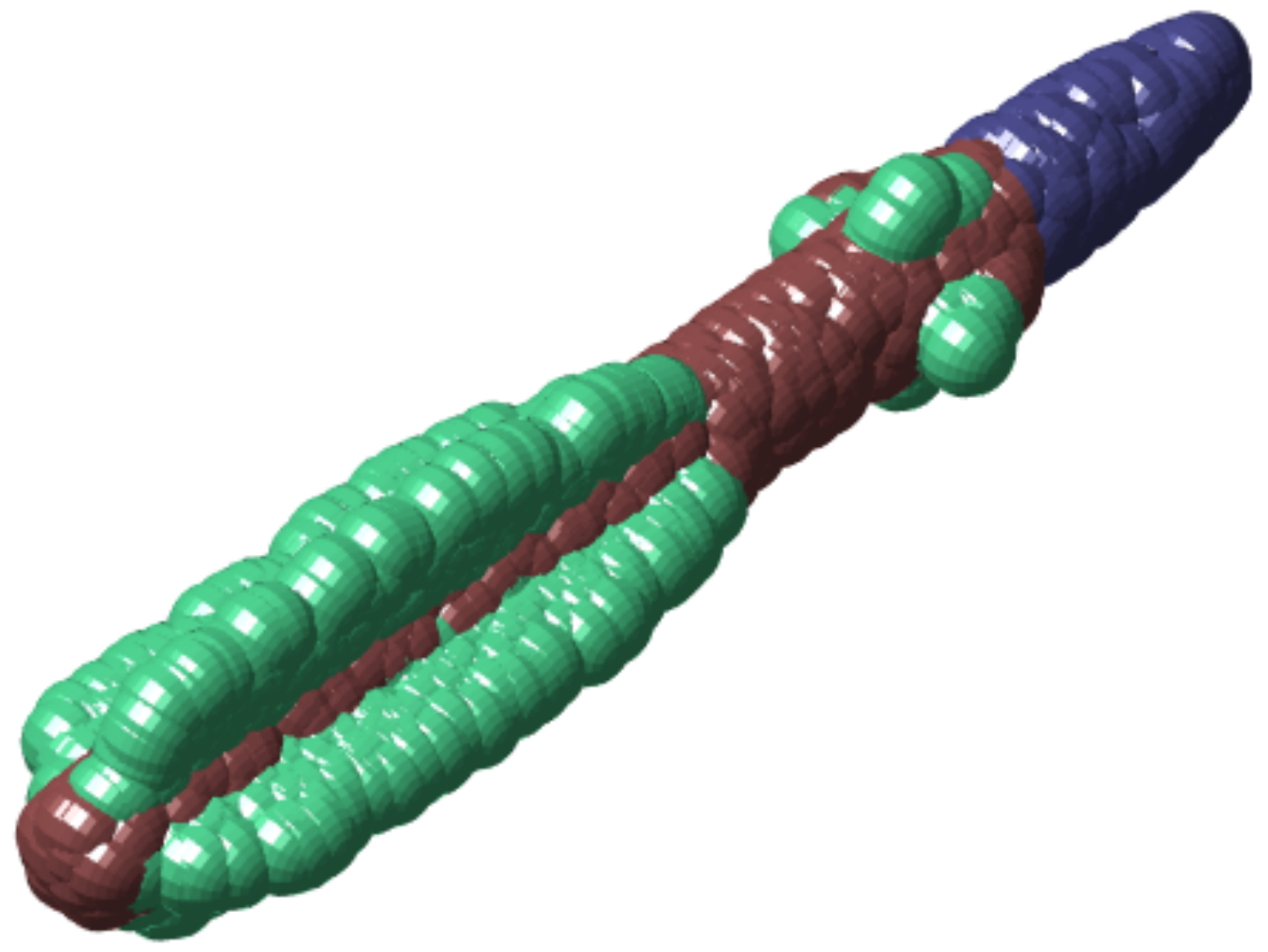}
    \includegraphics[scale=0.13]{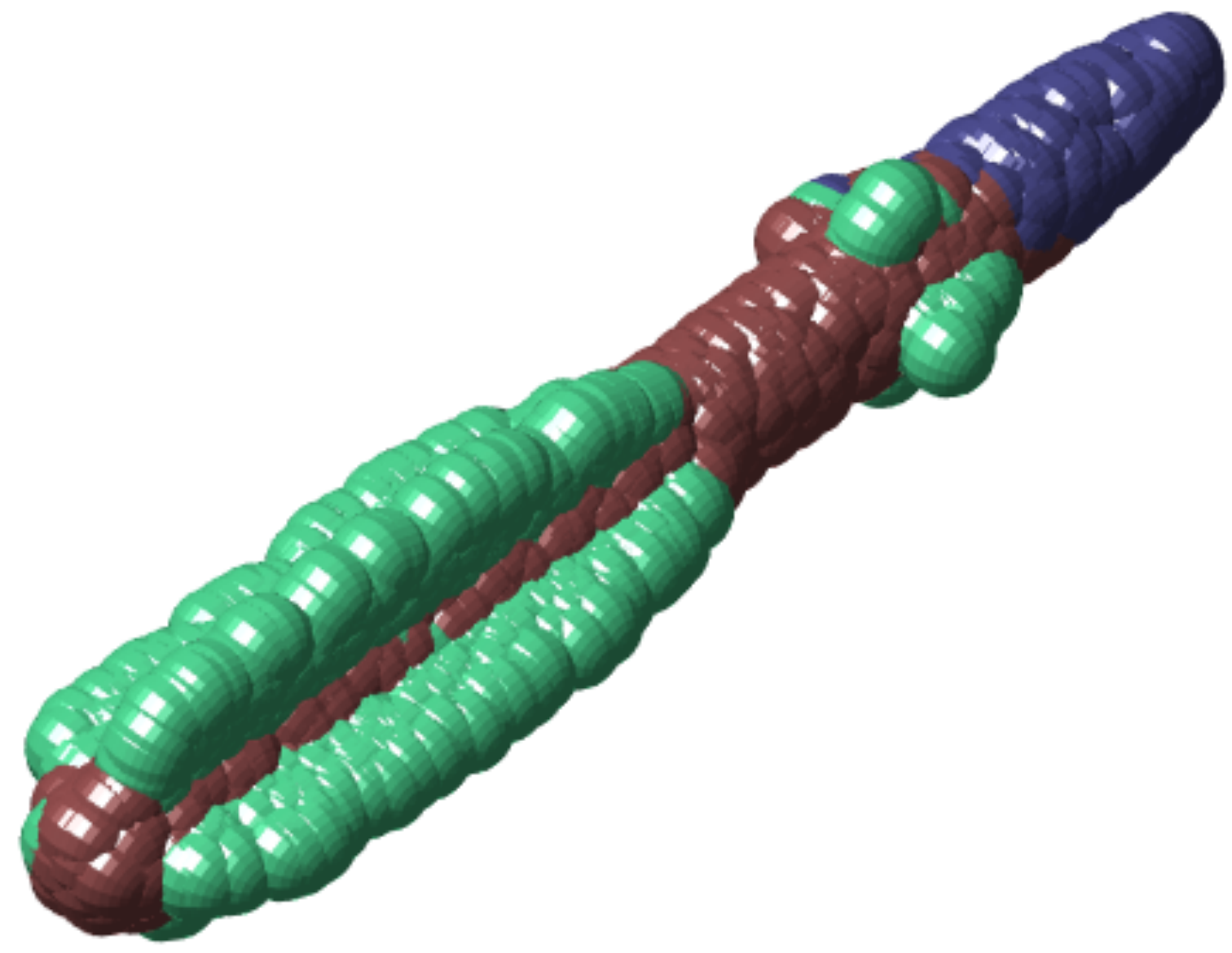}}
    \subfigure[rocket]{\includegraphics[scale=0.13]{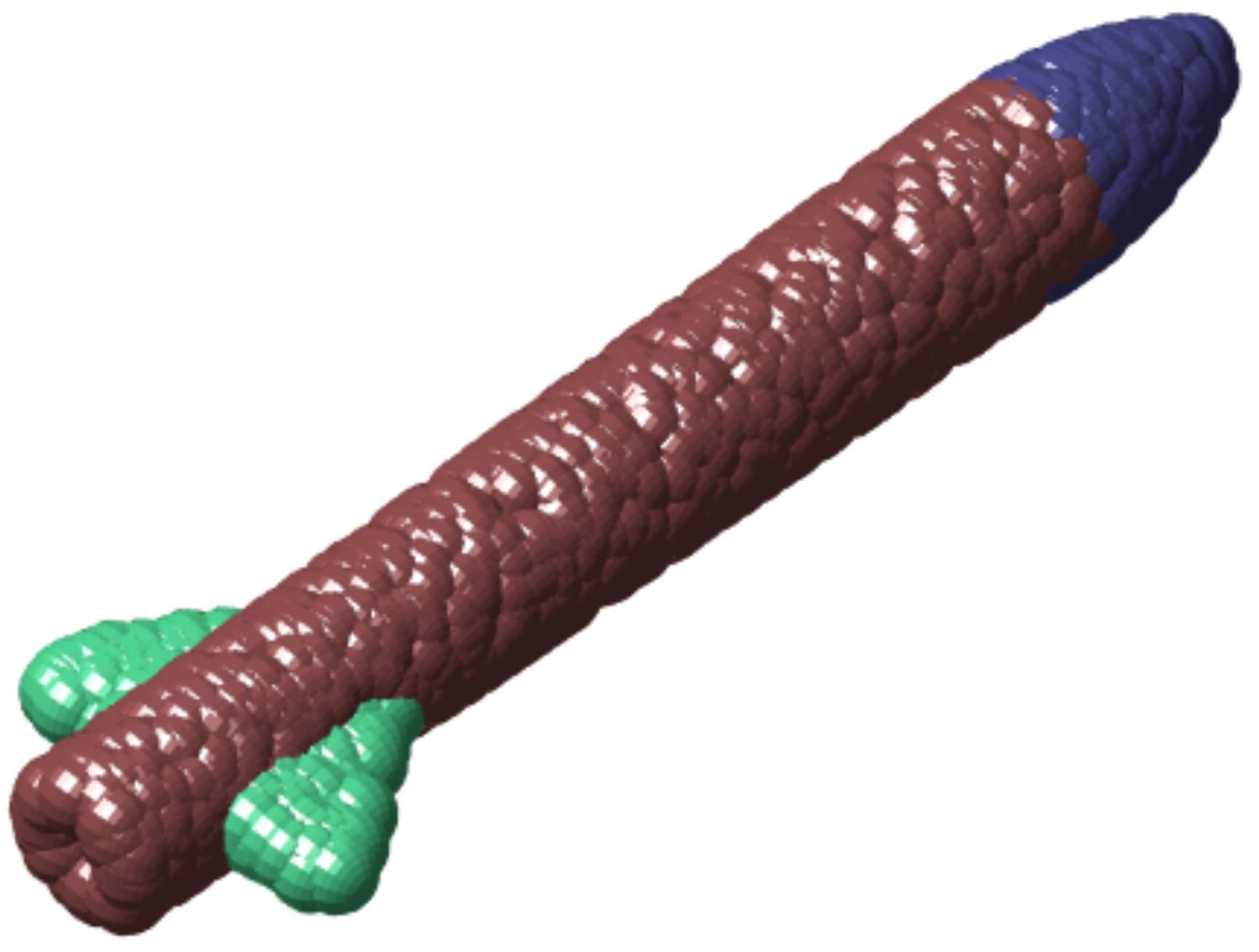}
    \includegraphics[scale=0.13]{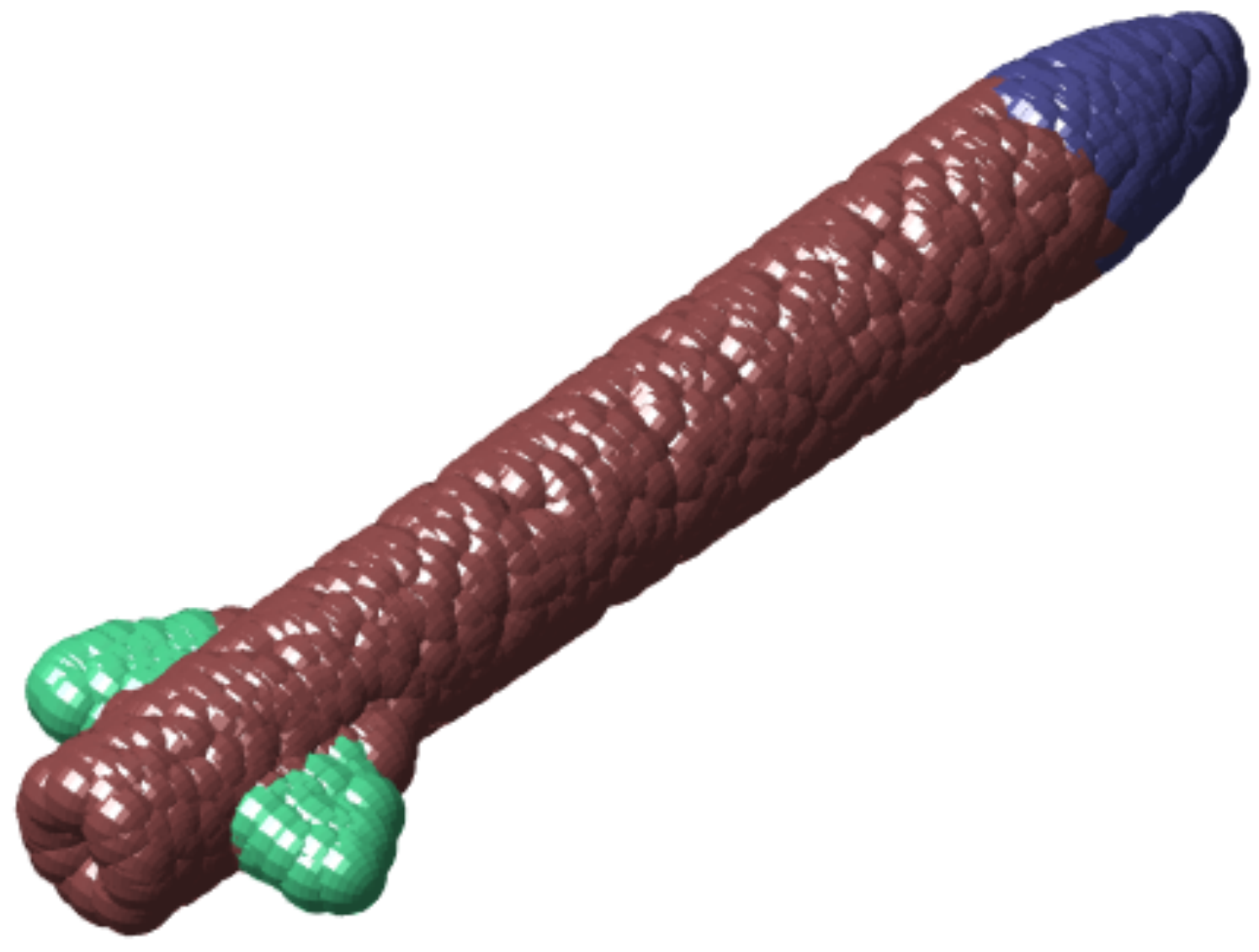}}
    \subfigure[table]{\includegraphics[scale=0.09]{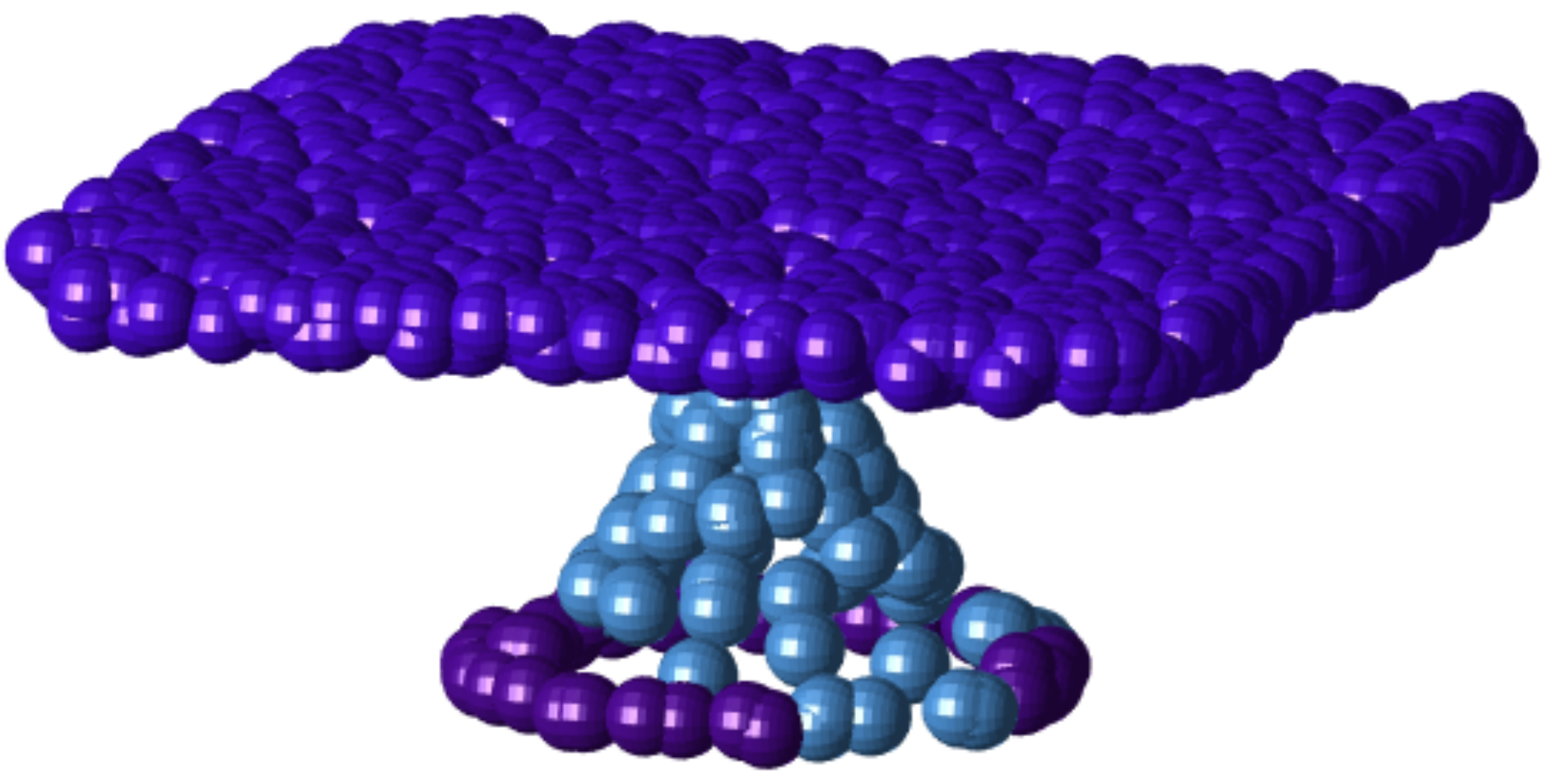}
    \includegraphics[scale=0.09]{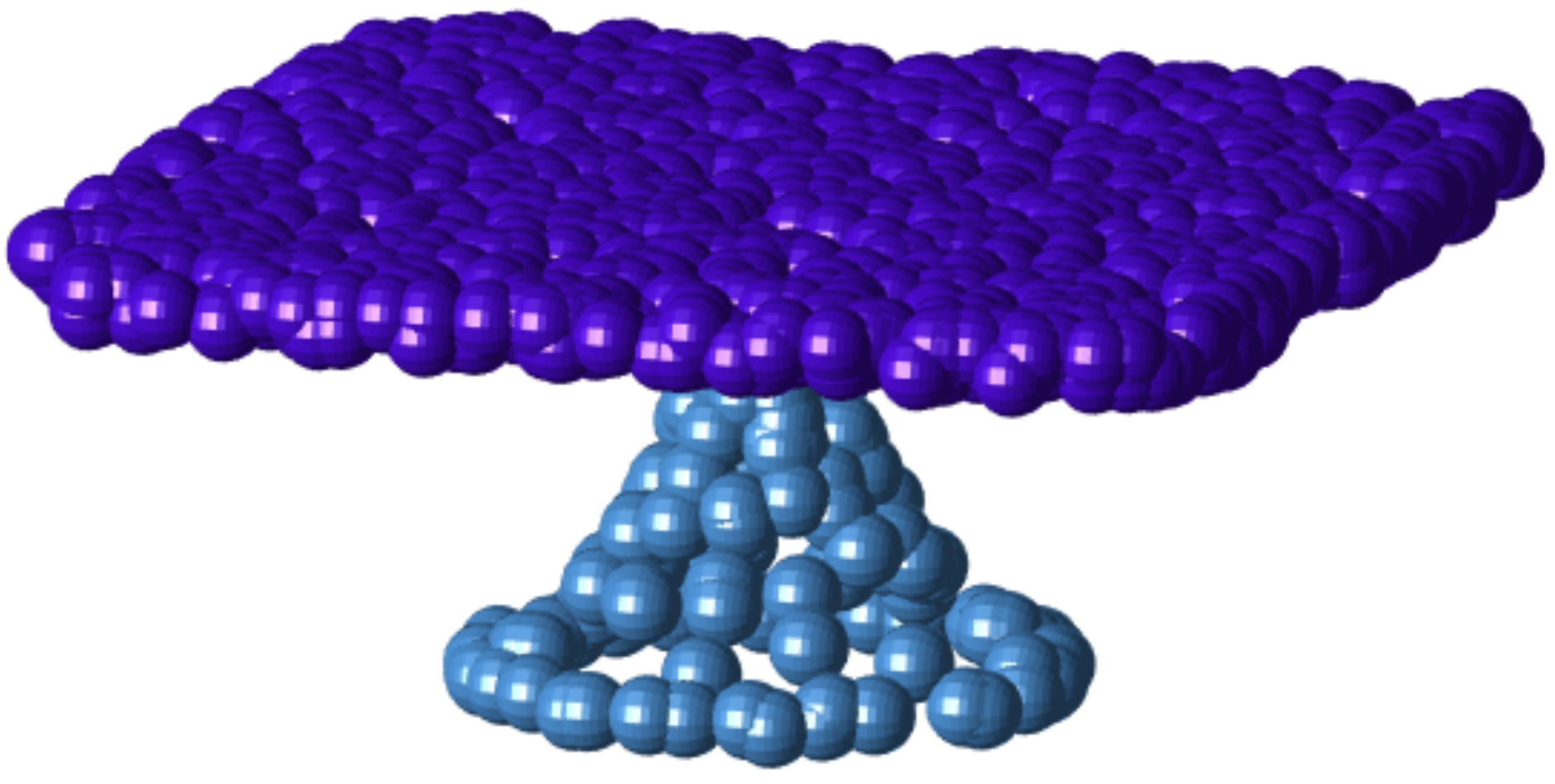}}
    \subfigure[table]{\includegraphics[scale=0.09]{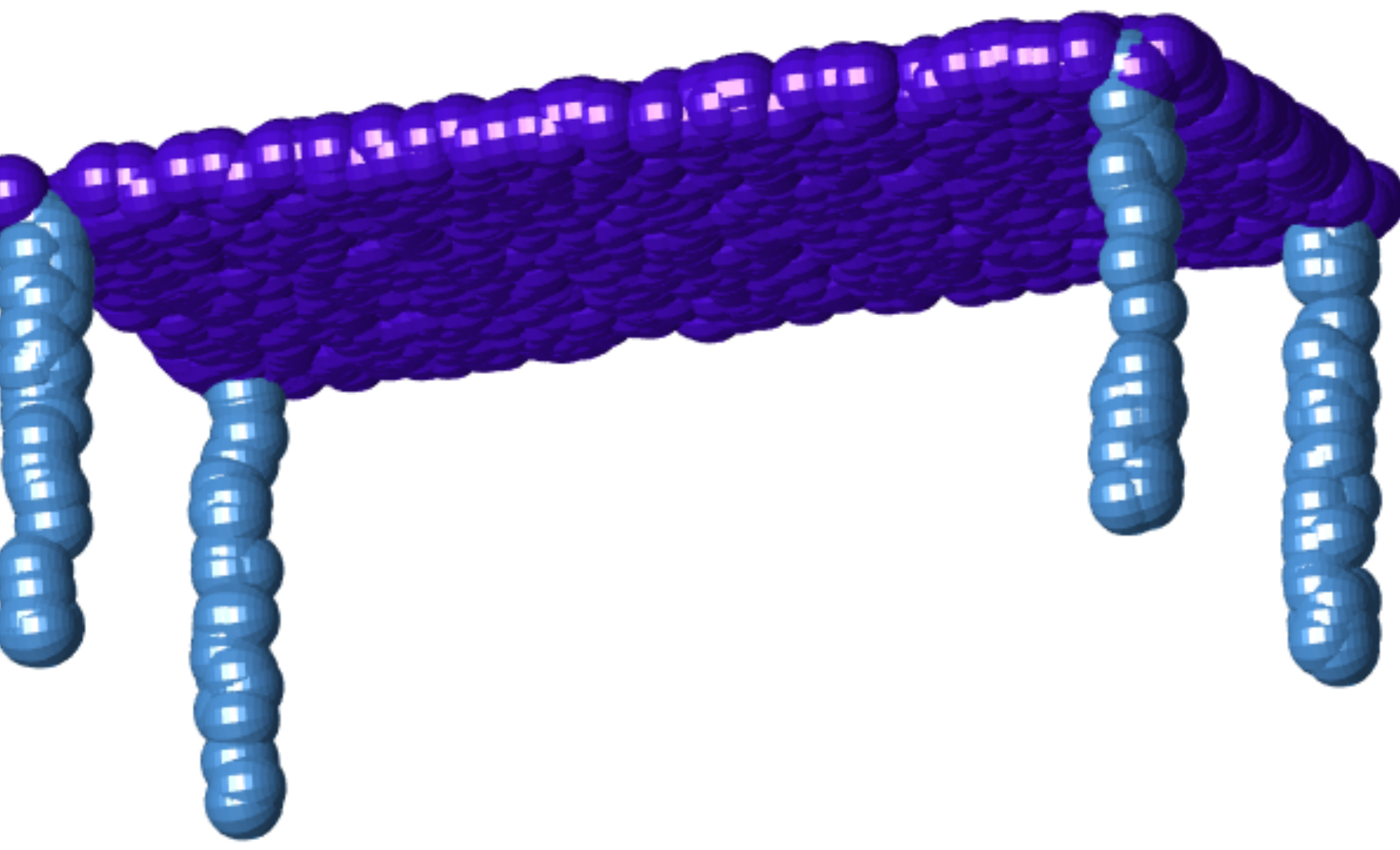}
    \includegraphics[scale=0.09]{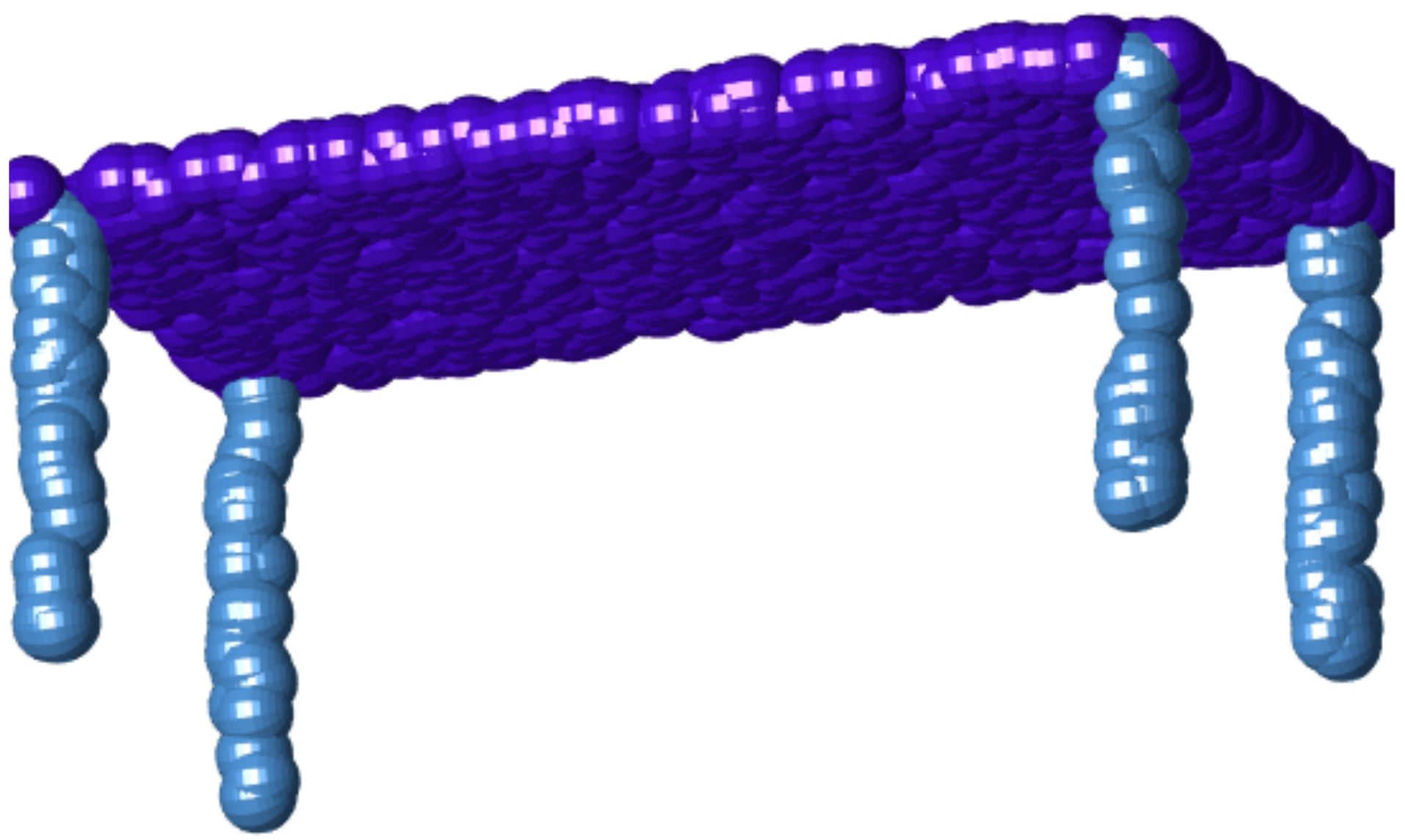}}
    \subfigure[table]{\includegraphics[scale=0.09]{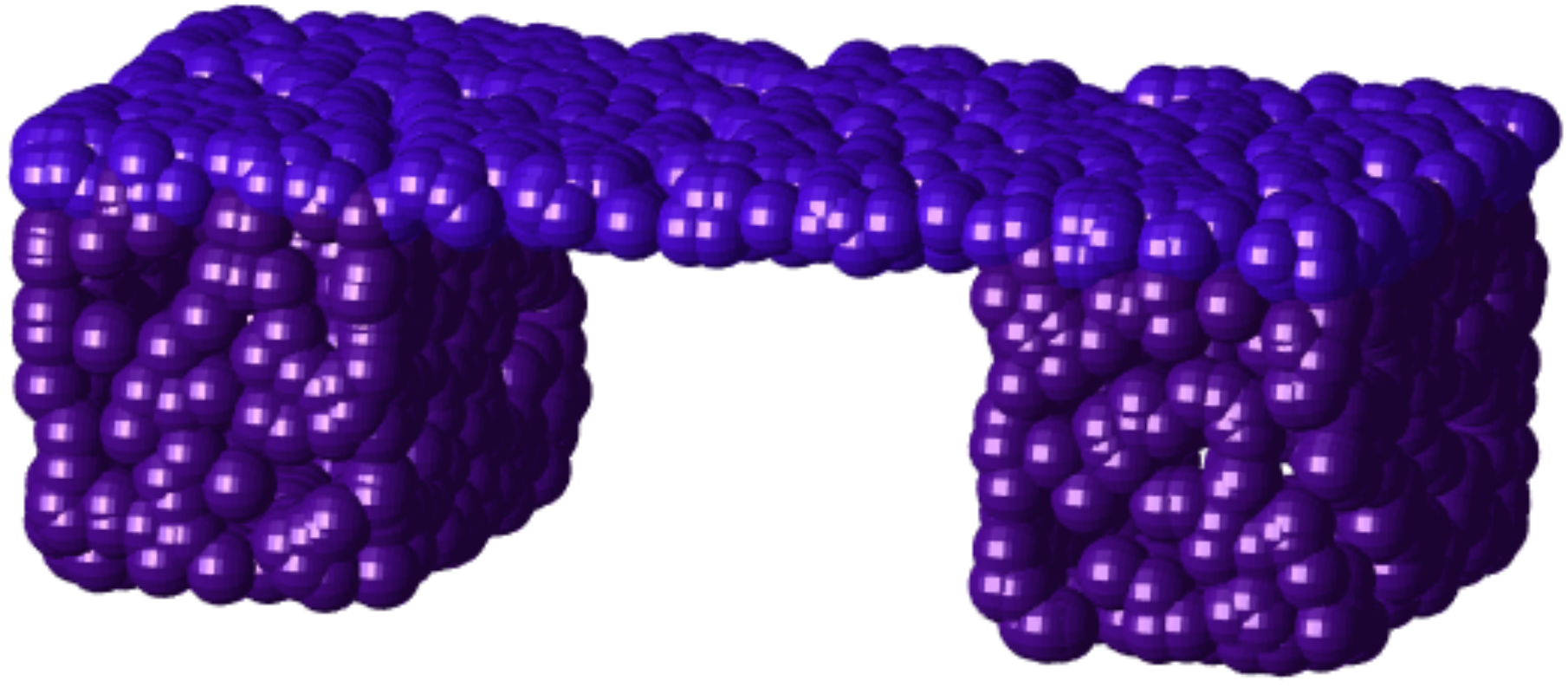}
    \includegraphics[scale=0.09]{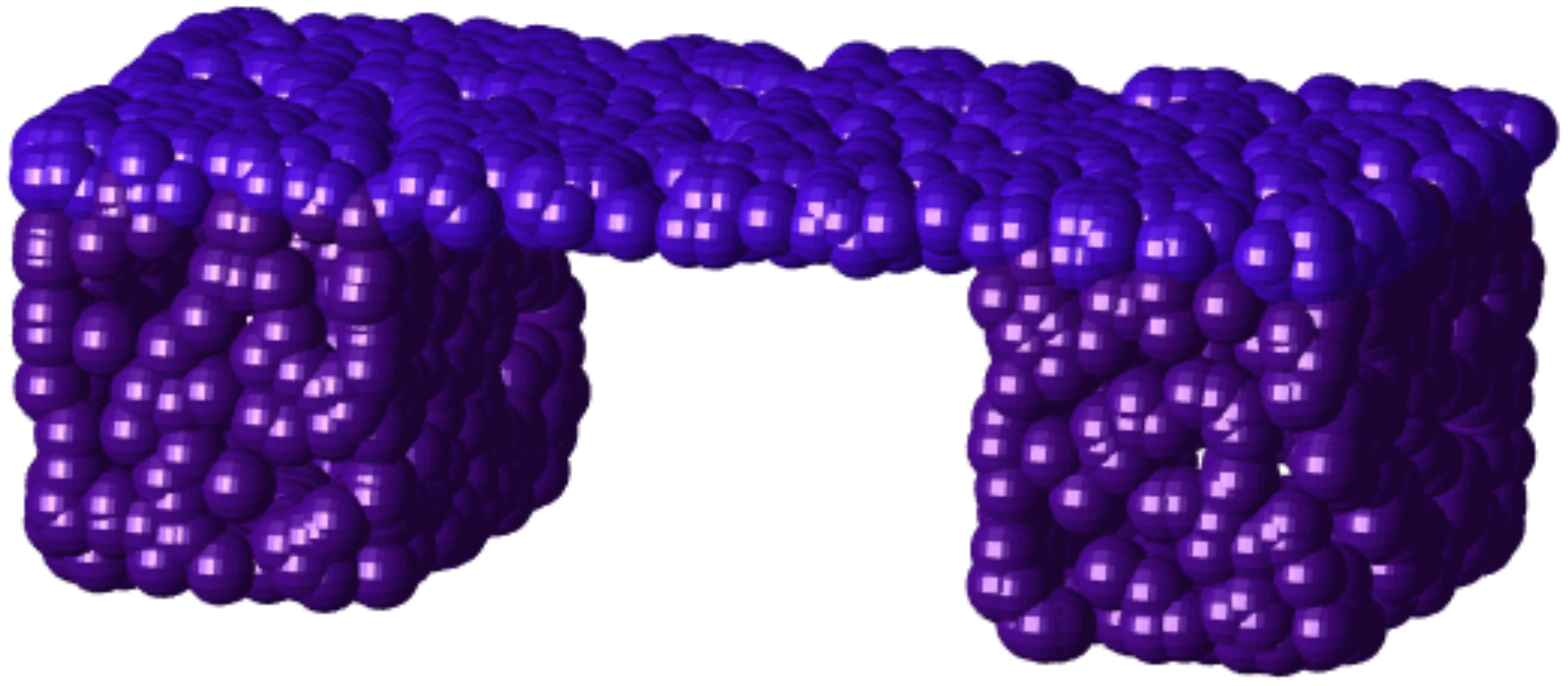}}
    \subfigure[table]{\includegraphics[scale=0.09]{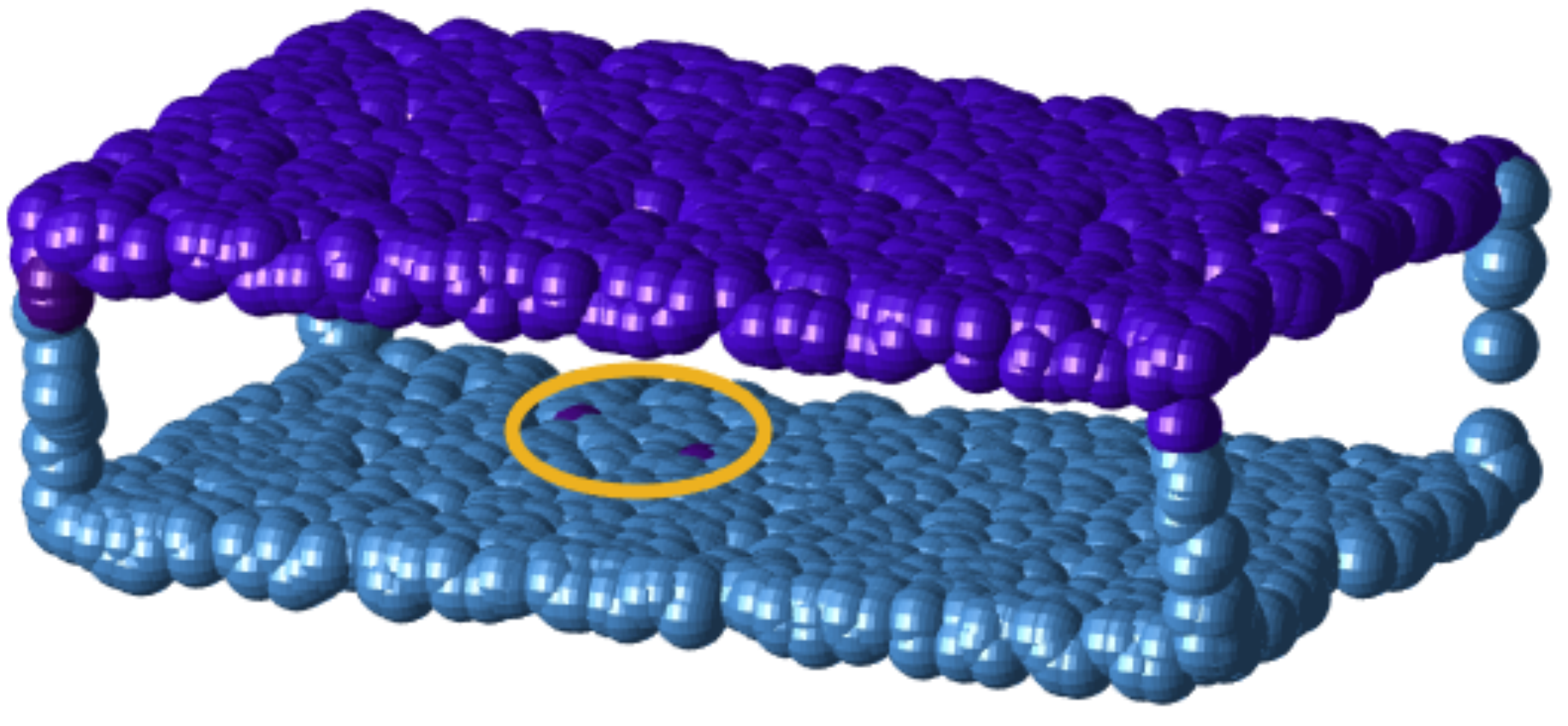}
    \includegraphics[scale=0.09]{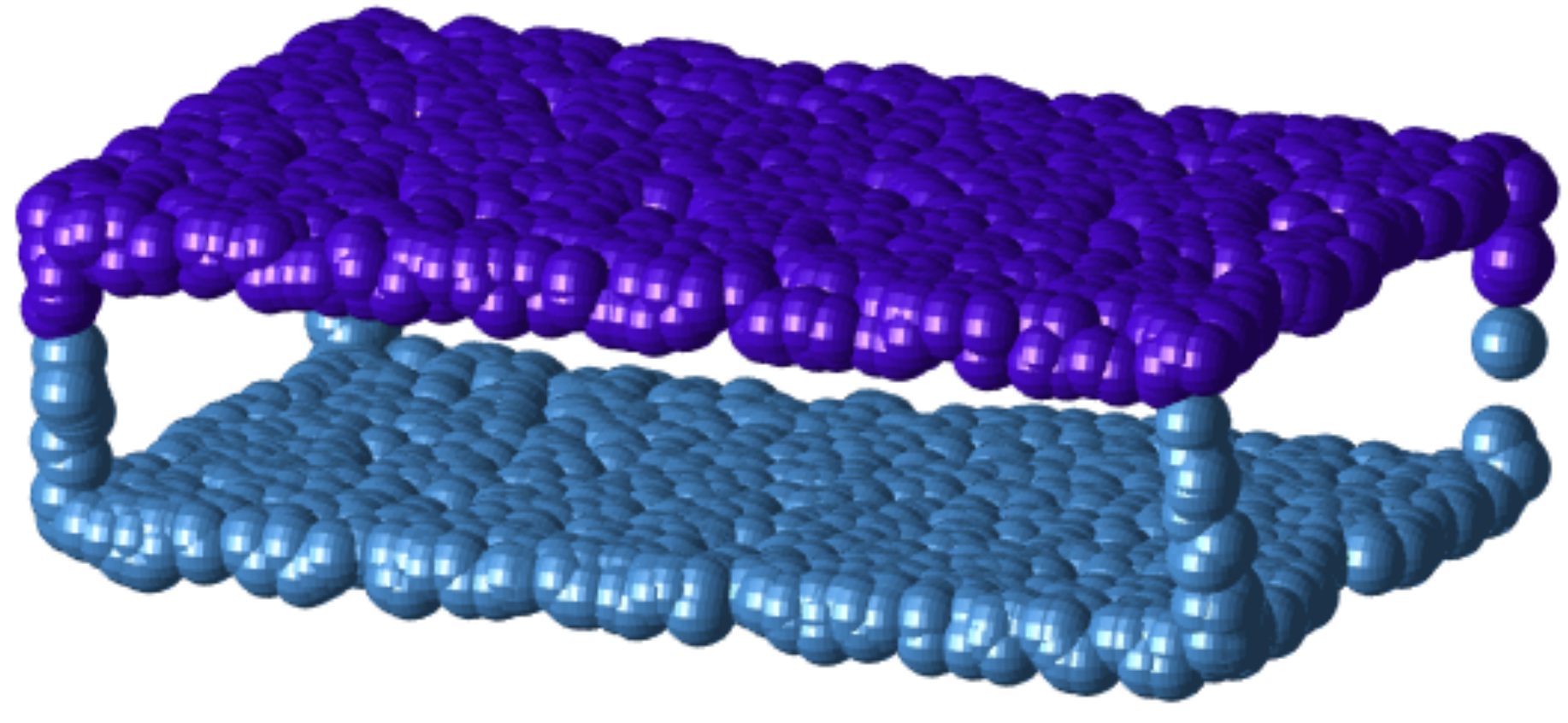}}
    \subfigure[skateboard]{\includegraphics[scale=0.10]{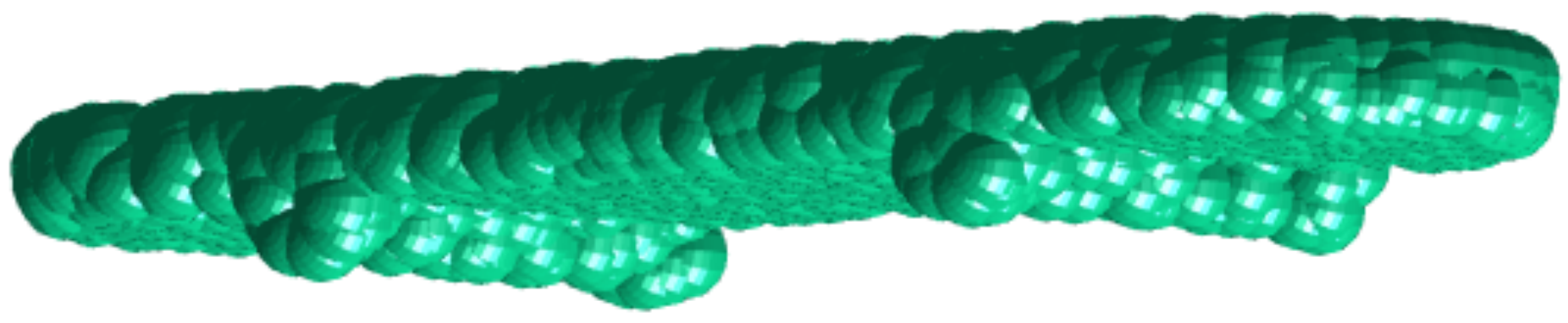}
    \includegraphics[scale=0.10]{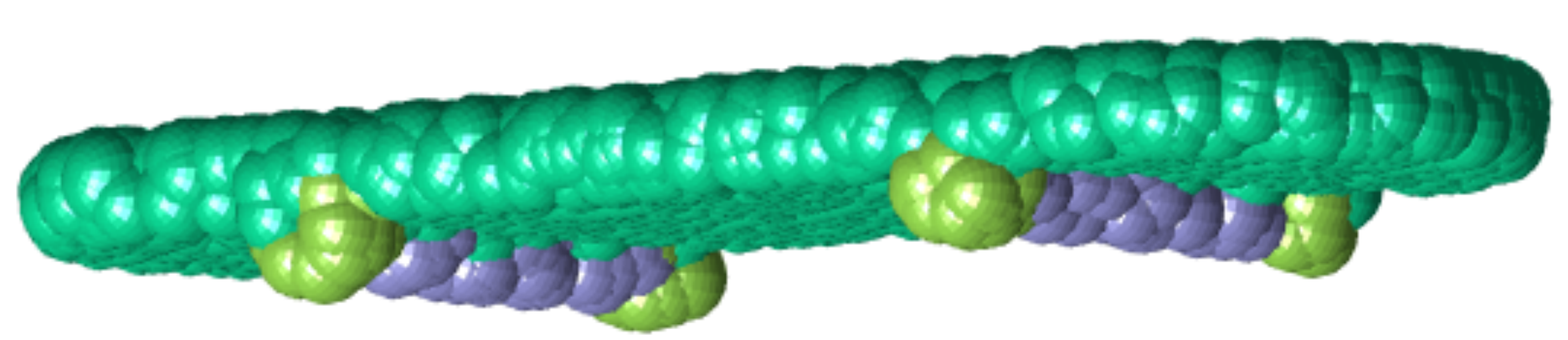}}
    \subfigure[skateboard]{\includegraphics[scale=0.11]{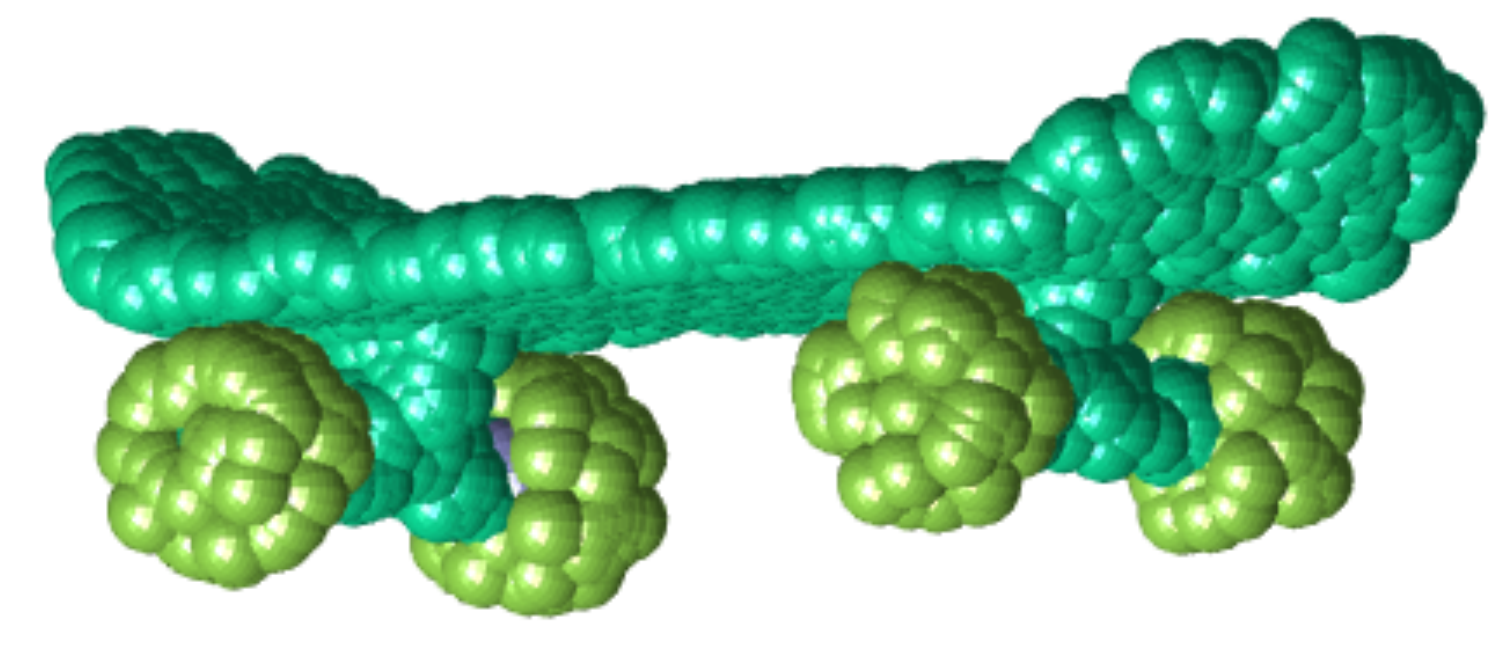}
    \includegraphics[scale=0.11]{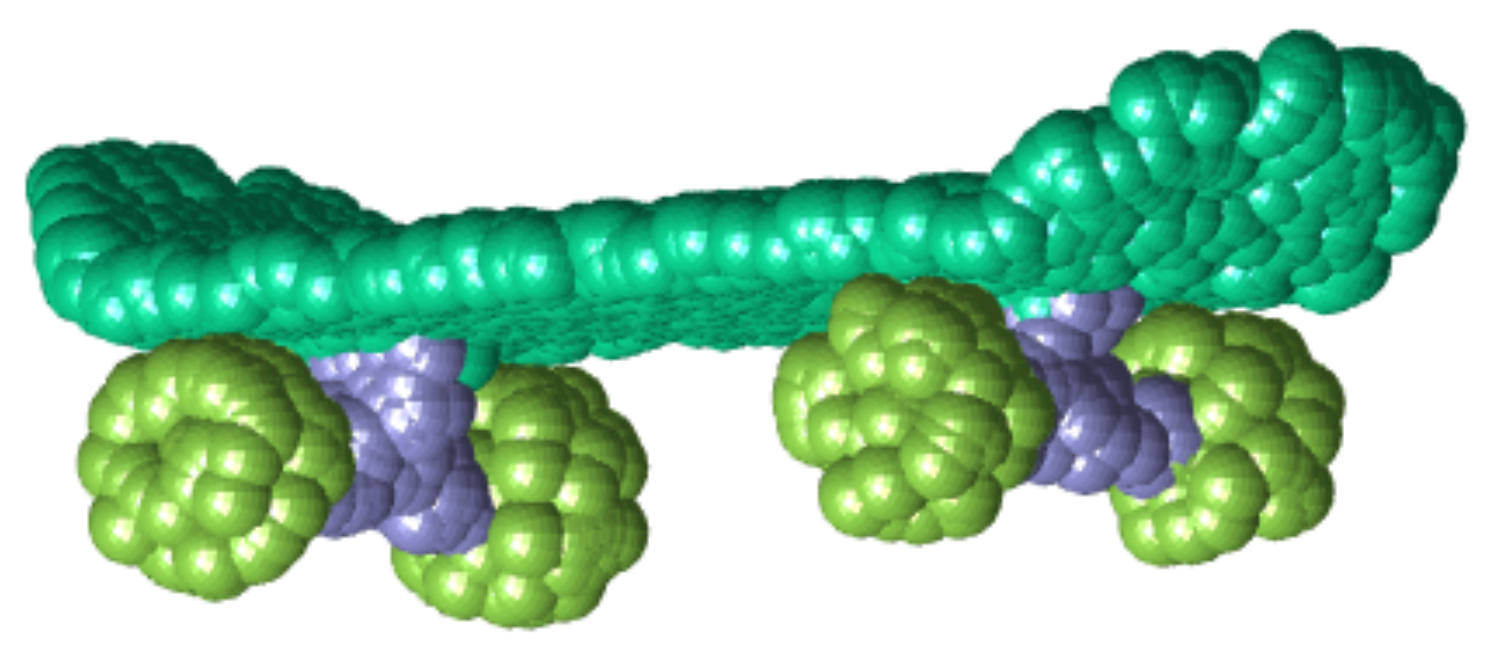}}
    \subfigure[skateboard]{\includegraphics[scale=0.11]{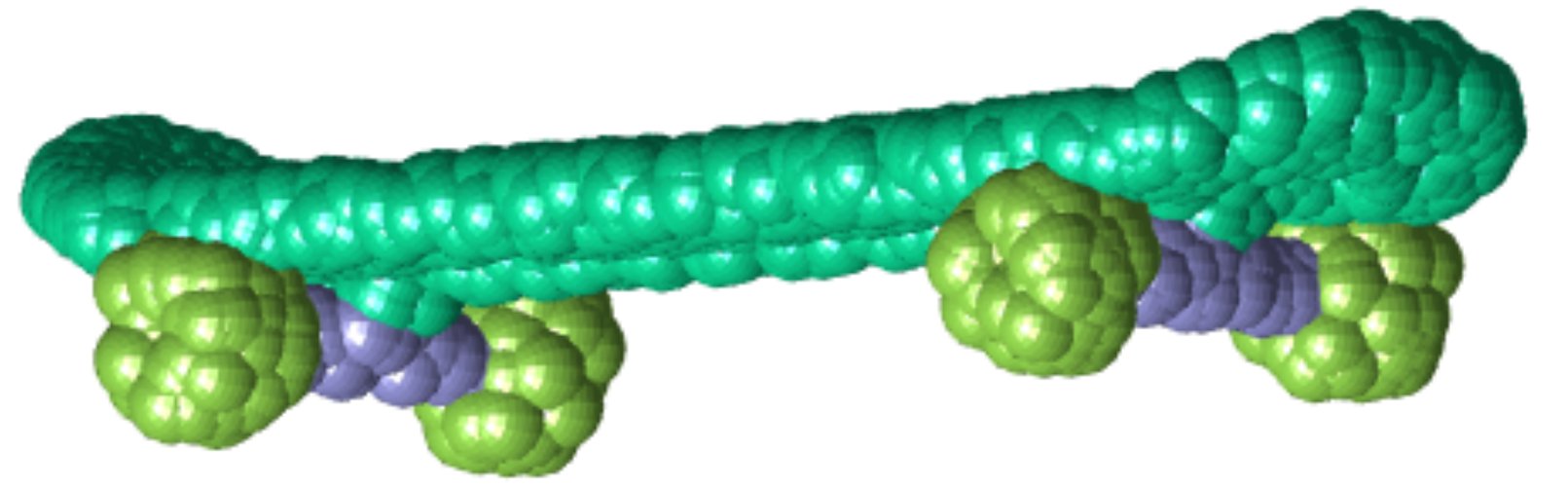}
    \includegraphics[scale=0.11]{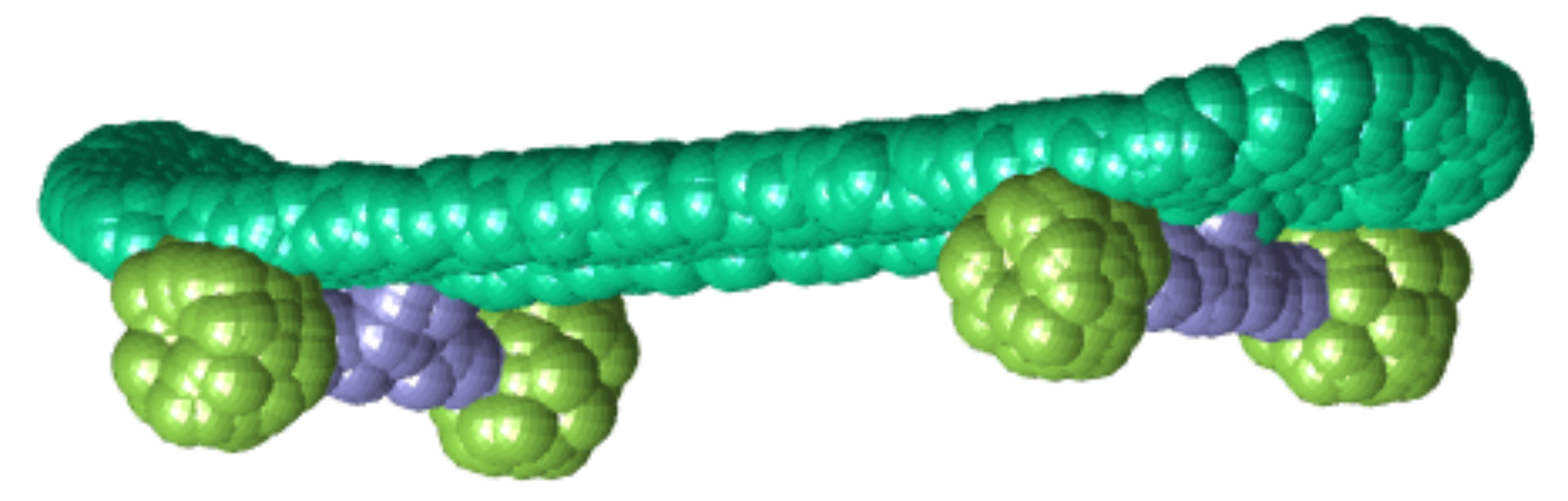}}
    \subfigure[skateboard]{\includegraphics[scale=0.11]{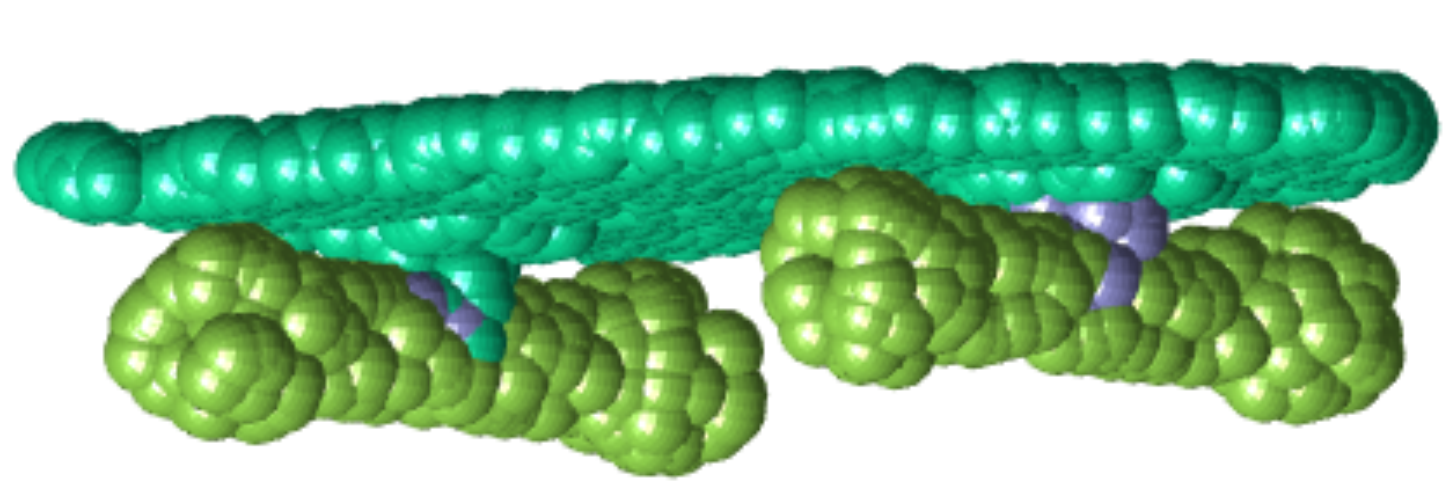}
    \includegraphics[scale=0.11]{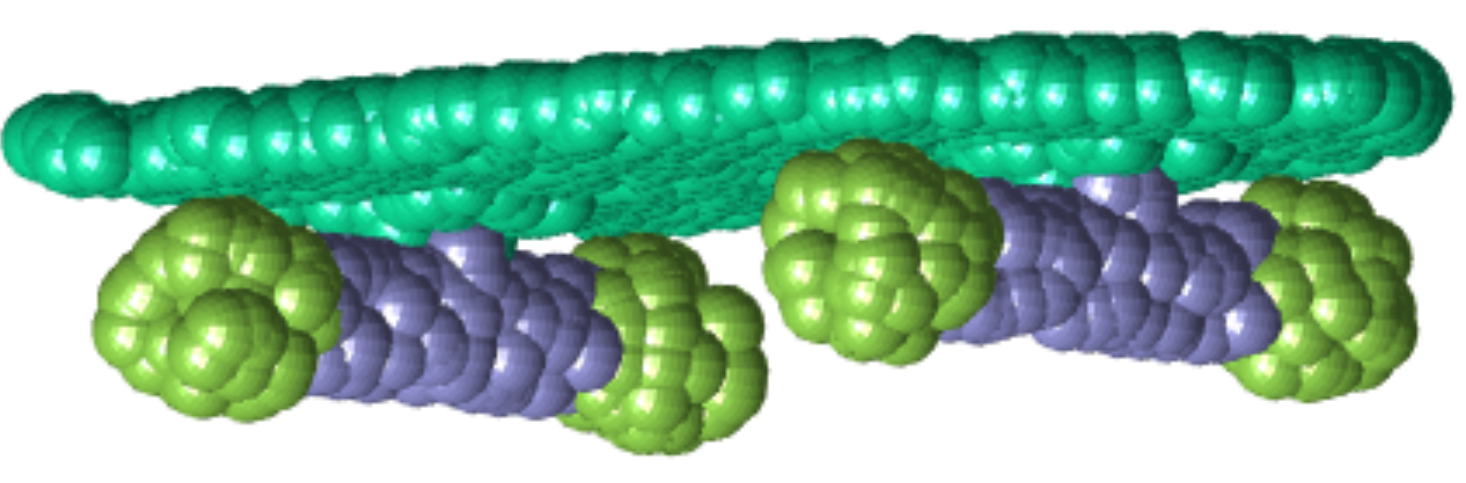}}
\caption{In each subfigure, we show ground truth (left) and our segmentation result (right). }
\label{fig:seg_plots}
\end{figure*}

\section{Conclusions}
Point-cloud helps with understanding 3D geometric shapes. In this work, we proposed a definition of correlation/ convolution on 3D point cloud and have shown that our framework is  ``augmentation-free"and is rotation invariant. Unlike the previous state-of-the-art methods, our proposed framework uses much leaner model by utilizing geometric structures in 3D point-clods. The core of our method is the proposed rotation invariant convolution on point-cloud induced from the topology of  sphere. We performed significantly better result over state-of-the-art algorithms for part segmentation task on shapenet dataset. In classification task, we have achieved similar result on rotated test data without explicit data augmentation. We have also tested on an unsupervised object detection task and detect the corpus callosum shape from an entire 3D brain scan. In future, we like to explore point set completion and object detection in large scale data, e.g. KITTI and various autonomous driving datasets. 

\bibliographystyle{ieeetran}
\bibliography{references_used}

\begin{thebibliography}{10}
\providecommand{\url}[1]{#1}
\csname url@samestyle\endcsname
\providecommand{\newblock}{\relax}
\providecommand{\bibinfo}[2]{#2}
\providecommand{\BIBentrySTDinterwordspacing}{\spaceskip=0pt\relax}
\providecommand{\BIBentryALTinterwordstretchfactor}{4}
\providecommand{\BIBentryALTinterwordspacing}{\spaceskip=\fontdimen2\font plus
\BIBentryALTinterwordstretchfactor\fontdimen3\font minus
  \fontdimen4\font\relax}
\providecommand{\BIBforeignlanguage}[2]{{%
\expandafter\ifx\csname l@#1\endcsname\relax
\typeout{** WARNING: IEEEtran.bst: No hyphenation pattern has been}%
\typeout{** loaded for the language `#1'. Using the pattern for}%
\typeout{** the default language instead.}%
\else
\language=\csname l@#1\endcsname
\fi
#2}}
\providecommand{\BIBdecl}{\relax}
\BIBdecl

\bibitem{Iandola2017SqueezeNetAA}
F.~N. Iandola, M.~W. Moskewicz, K.~Ashraf, S.~Han, W.~J. Dally, and K.~Keutzer,
  ``Squeezenet: Alexnet-level accuracy with 50x fewer parameters and <1mb model
  size,'' \emph{ArXiv}, vol. abs/1602.07360, 2017.

\bibitem{pointnet}
C.~R. Qi, H.~Su, K.~Mo, and L.~J. Guibas, ``Pointnet: Deep learning on point
  sets for 3d classification and segmentation,'' \emph{Proc. Computer Vision
  and Pattern Recognition (CVPR), IEEE}, vol.~1, no.~2, p.~4, 2017.

\bibitem{pointnet++}
C.~R. Qi, L.~Yi, H.~Su, and L.~J. Guibas, ``Pointnet++: Deep hierarchical
  feature learning on point sets in a metric space,'' in \emph{Advances in
  Neural Information Processing Systems}, 2017, pp. 5099--5108.

\bibitem{lin2018novel}
C.-H. Lin, Y.~Chung, B.-Y. Chou, H.-Y. Chen, and C.-Y. Tsai, ``A novel campus
  navigation app with augmented reality and deep learning,'' in \emph{2018 IEEE
  International Conference on Applied System Invention (ICASI)}.\hskip 1em plus
  0.5em minus 0.4em\relax IEEE, 2018, pp. 1075--1077.

\bibitem{rambach2016learning}
J.~R. Rambach, A.~Tewari, A.~Pagani, and D.~Stricker, ``Learning to fuse: A
  deep learning approach to visual-inertial camera pose estimation,'' in
  \emph{Mixed and Augmented Reality (ISMAR), 2016 IEEE International Symposium
  on}.\hskip 1em plus 0.5em minus 0.4em\relax IEEE, 2016, pp. 71--76.

\bibitem{tulsiani2018factoring}
S.~Tulsiani, S.~Gupta, D.~Fouhey, A.~A. Efros, and J.~Malik, ``Factoring shape,
  pose, and layout from the 2d image of a 3d scene,'' in \emph{Proceedings of
  the IEEE Conference on Computer Vision and Pattern Recognition}, 2018, pp.
  302--310.

\bibitem{vasu2018occlusion}
S.~Vasu, M.~M. MR, and A.~Rajagopalan, ``Occlusion-aware rolling shutter
  rectification of 3d scenes,'' in \emph{Proceedings of the IEEE Conference on
  Computer Vision and Pattern Recognition}, 2018, pp. 636--645.

\bibitem{dai2018scancomplete}
A.~Dai, D.~Ritchie, M.~Bokeloh, S.~Reed, J.~Sturm, and M.~Nie{\ss}ner,
  ``Scancomplete: Large-scale scene completion and semantic segmentation for 3d
  scans,'' in \emph{CVPR}, vol.~1, 2018, p.~2.

\bibitem{chen2018lidar}
Y.~Chen, J.~Wang, J.~Li, C.~Lu, Z.~Luo, H.~Xue, and C.~Wang, ``Lidar-video
  driving dataset: Learning driving policies effectively,'' in
  \emph{Proceedings of the IEEE Conference on Computer Vision and Pattern
  Recognition}, 2018, pp. 5870--5878.

\bibitem{shen2018stereo}
S.~Shen, ``Stereo vision-based semantic 3d object and ego-motion tracking for
  autonomous driving,'' \emph{arXiv preprint arXiv:1807.02062}, 2018.

\bibitem{wu2018squeezeseg}
B.~Wu, A.~Wan, X.~Yue, and K.~Keutzer, ``Squeezeseg: Convolutional neural nets
  with recurrent crf for real-time road-object segmentation from 3d lidar point
  cloud,'' in \emph{2018 IEEE International Conference on Robotics and
  Automation (ICRA)}.\hskip 1em plus 0.5em minus 0.4em\relax IEEE, 2018, pp.
  1887--1893.

\bibitem{deng2009imagenet}
J.~Deng, W.~Dong, R.~Socher, L.-J. Li, K.~Li, and L.~Fei-Fei, ``Imagenet: A
  large-scale hierarchical image database,'' in \emph{2009 IEEE conference on
  computer vision and pattern recognition}.\hskip 1em plus 0.5em minus
  0.4em\relax Ieee, 2009, pp. 248--255.

\bibitem{krizhevsky2012imagenet}
A.~Krizhevsky, I.~Sutskever, and G.~E. Hinton, ``Imagenet classification with
  deep convolutional neural networks,'' in \emph{Advances in neural information
  processing systems}, 2012, pp. 1097--1105.

\bibitem{wu20153d}
Z.~Wu, S.~Song, A.~Khosla, F.~Yu, L.~Zhang, X.~Tang, and J.~Xiao, ``3d
  shapenets: A deep representation for volumetric shapes,'' in
  \emph{Proceedings of the IEEE conference on computer vision and pattern
  recognition}, 2015, pp. 1912--1920.

\bibitem{riegler2017octnet}
G.~Riegler, A.~O. Ulusoy, and A.~Geiger, ``Octnet: Learning deep 3d
  representations at high resolutions,'' in \emph{Proceedings of the IEEE
  Conference on Computer Vision and Pattern Recognition}, vol.~3, 2017.

\bibitem{wang2017cnn}
P.-S. Wang, Y.~Liu, Y.-X. Guo, C.-Y. Sun, and X.~Tong, ``O-cnn: Octree-based
  convolutional neural networks for 3d shape analysis,'' \emph{ACM Transactions
  on Graphics (TOG)}, vol.~36, no.~4, p.~72, 2017.

\bibitem{su2015multi}
H.~Su, S.~Maji, E.~Kalogerakis, and E.~Learned-Miller, ``Multi-view
  convolutional neural networks for 3d shape recognition,'' in
  \emph{Proceedings of the IEEE international conference on computer vision},
  2015, pp. 945--953.

\bibitem{qi2016volumetric}
C.~R. Qi, H.~Su, M.~Nie{\ss}ner, A.~Dai, M.~Yan, and L.~J. Guibas, ``Volumetric
  and multi-view cnns for object classification on 3d data,'' in
  \emph{Proceedings of the IEEE conference on computer vision and pattern
  recognition}, 2016, pp. 5648--5656.

\bibitem{kalogerakis20173d}
E.~Kalogerakis, M.~Averkiou, S.~Maji, and S.~Chaudhuri, ``3d shape segmentation
  with projective convolutional networks,'' in \emph{Proc. CVPR}, vol.~1,
  no.~2, 2017, p.~8.

\bibitem{zhou2018learning}
L.~Zhou, S.~Zhu, Z.~Luo, T.~Shen, R.~Zhang, M.~Zhen, T.~Fang, and L.~Quan,
  ``Learning and matching multi-view descriptors for registration of point
  clouds,'' \emph{arXiv preprint arXiv:1807.05653}, 2018.

\bibitem{rethage2018fully}
D.~Rethage, J.~Wald, J.~Sturm, N.~Navab, and F.~Tombari, ``Fully-convolutional
  point networks for large-scale point clouds,'' \emph{arXiv preprint
  arXiv:1808.06840}, 2018.

\bibitem{MR3D}
M.~Gadelha, R.~Wang, and S.~Maji, ``Multiresolution tree networks for 3d point
  cloud processing,'' \emph{arXiv preprint arXiv:1807.03520}, 2018.

\bibitem{DGCNN}
Y.~Wang, Y.~Sun, Z.~Liu, S.~E. Sarma, M.~M. Bronstein, and J.~M. Solomon,
  ``Dynamic graph cnn for learning on point clouds,'' \emph{arXiv preprint
  arXiv:1801.07829}, 2018.

\bibitem{li2018pointcnn}
Y.~Li, R.~Bu, M.~Sun, W.~Wu, X.~Di, and B.~Chen, ``Pointcnn: Convolution on
  x-transformed points,'' in \emph{Advances in Neural Information Processing
  Systems}, 2018, pp. 820--830.

\bibitem{scholkopf1997kernel}
B.~Sch{\"o}lkopf, A.~Smola, and K.-R. M{\"u}ller, ``Kernel principal component
  analysis,'' in \emph{International conference on artificial neural
  networks}.\hskip 1em plus 0.5em minus 0.4em\relax Springer, 1997, pp.
  583--588.

\bibitem{maturana2015voxnet}
D.~Maturana and S.~Scherer, ``Voxnet: A 3d convolutional neural network for
  real-time object recognition,'' in \emph{2015 IEEE/RSJ International
  Conference on Intelligent Robots and Systems (IROS)}.\hskip 1em plus 0.5em
  minus 0.4em\relax IEEE, 2015, pp. 922--928.

\bibitem{tatarchenko2017octree}
M.~Tatarchenko, A.~Dosovitskiy, and T.~Brox, ``Octree generating networks:
  Efficient convolutional architectures for high-resolution 3d outputs,'' in
  \emph{Proceedings of the IEEE International Conference on Computer Vision},
  2017, pp. 2088--2096.

\bibitem{klokov2017escape}
R.~Klokov and V.~Lempitsky, ``Escape from cells: Deep kd-networks for the
  recognition of 3d point cloud models,'' in \emph{Proceedings of the IEEE
  International Conference on Computer Vision}, 2017, pp. 863--872.

\bibitem{su2018splatnet}
H.~Su, V.~Jampani, D.~Sun, S.~Maji, E.~Kalogerakis, M.-H. Yang, and J.~Kautz,
  ``Splatnet: Sparse lattice networks for point cloud processing,'' in
  \emph{Proceedings of the IEEE Conference on Computer Vision and Pattern
  Recognition}, 2018, pp. 2530--2539.

\bibitem{adams2010fast}
A.~Adams, J.~Baek, and M.~A. Davis, ``Fast high-dimensional filtering using the
  permutohedral lattice,'' in \emph{Computer Graphics Forum}, vol.~29,
  no.~2.\hskip 1em plus 0.5em minus 0.4em\relax Wiley Online Library, 2010, pp.
  753--762.

\bibitem{bilaterCNN}
V.~Jampani, M.~Kiefel, and P.~V. Gehler, ``Learning sparse high dimensional
  filters: Image filtering, dense crfs and bilateral neural networks,'' in
  \emph{Proceedings of the IEEE Conference on Computer Vision and Pattern
  Recognition}, 2016, pp. 4452--4461.

\bibitem{landrieu2018large}
L.~Landrieu and M.~Simonovsky, ``Large-scale point cloud semantic segmentation
  with superpoint graphs,'' in \emph{Proceedings of the IEEE Conference on
  Computer Vision and Pattern Recognition}, 2018, pp. 4558--4567.

\bibitem{cohen2017convolutional}
T.~Cohen, M.~Geiger, J.~K{\"o}hler, and M.~Welling, ``Convolutional networks
  for spherical signals,'' \emph{arXiv preprint arXiv:1709.04893}, 2017.

\bibitem{chakraborty2018cnn}
R.~Chakraborty, M.~Banerjee, and B.~C. Vemuri, ``A cnn for homogneous
  riemannian manifolds with applications to neuroimaging,'' \emph{arXiv
  preprint arXiv:1805.05487}, 2018.

\bibitem{helgason}
S.~Helgason, \emph{Differential geometry, Lie groups, and symmetric
  spaces}.\hskip 1em plus 0.5em minus 0.4em\relax Academic press, 1979,
  vol.~80.

\bibitem{ioffe2015batch}
S.~Ioffe and C.~Szegedy, ``Batch normalization: Accelerating deep network
  training by reducing internal covariate shift,'' \emph{arXiv preprint
  arXiv:1502.03167}, 2015.

\bibitem{kingma2018glow}
D.~P. Kingma and P.~Dhariwal, ``Glow: Generative flow with invertible 1x1
  convolutions,'' in \emph{Advances in Neural Information Processing Systems},
  2018, pp. 10\,215--10\,224.

\bibitem{hobson1931theory}
E.~W. Hobson, \emph{The theory of spherical and ellipsoidal harmonics}.\hskip
  1em plus 0.5em minus 0.4em\relax CUP Archive, 1931.

\bibitem{zhao2016tensor}
Q.~Zhao, G.~Zhou, S.~Xie, L.~Zhang, and A.~Cichocki, ``Tensor ring
  decomposition,'' \emph{arXiv preprint arXiv:1606.05535}, 2016.

\bibitem{chang2015shapenet}
A.~X. Chang, T.~Funkhouser, L.~Guibas, P.~Hanrahan, Q.~Huang, Z.~Li,
  S.~Savarese, M.~Savva, S.~Song, H.~Su \emph{et~al.}, ``Shapenet: An
  information-rich 3d model repository,'' \emph{arXiv preprint
  arXiv:1512.03012}, 2015.

\bibitem{deng2012mnist}
L.~Deng, ``The mnist database of handwritten digit images for machine learning
  research [best of the web],'' \emph{IEEE Signal Processing Magazine},
  vol.~29, no.~6, pp. 141--142, 2012.

\bibitem{fotenos2005normative}
A.~F. Fotenos, A.~Snyder, L.~Girton, J.~Morris, and R.~Buckner, ``Normative
  estimates of cross-sectional and longitudinal brain volume decline in aging
  and ad,'' \emph{Neurology}, vol.~64, no.~6, pp. 1032--1039, 2005.

\bibitem{avants2009advanced}
B.~B. Avants, N.~Tustison, and G.~Song, ``Advanced normalization tools
  (ants),'' \emph{Insight j}, vol.~2, pp. 1--35, 2009.

\bibitem{vemurimultiatlas}
F.~Wang, B.~C. Vemuri, A.~Rangarajan, and S.~J. Eisenschenk, ``Simultaneous
  nonrigid registration of multiple point sets and atlas construction,''
  \emph{IEEE transactions on pattern analysis and machine intelligence},
  vol.~30, no.~11, pp. 2011--2022, 2008.

\end{thebibliography}

\end{document}